\documentclass{article}

\usepackage{microtype}
\usepackage{graphicx}
\usepackage{subfigure}
\usepackage{booktabs} 

\usepackage{hyperref}

\usepackage[accepted]{icml2025}

\usepackage{amsmath}
\usepackage{amssymb}
\usepackage{mathtools}
\usepackage{amsthm}

\usepackage[capitalize,noabbrev]{cleveref}

\usepackage{graphicx}
\usepackage{hyperref}
\usepackage{amsmath}
\usepackage{amssymb}
\usepackage{mathtools}
\usepackage{amsthm}
\usepackage{mathbbol}
\usepackage{multirow}
\usepackage{graphicx}
\usepackage{subfigure}
\usepackage{booktabs} 
\usepackage{subfigure}
\usepackage{booktabs} 
\usepackage{dsfont}
\usepackage{mathtools}
\usepackage{amsfonts}
\usepackage{xcolor}
\usepackage{varwidth}
\usepackage{amsmath}
\usepackage{amsthm}
\usepackage{mathtools,amssymb}
\usepackage{lscape}
\usepackage{enumitem}
\usepackage{stmaryrd}
\usepackage{comment}
\usepackage{nicematrix}
\usepackage{algorithm}

\def\dist{\text{dist}}
\def\cls{\text{cls}}
\def\E{\mathbb{E}}
\def\Var{\mathbb V}
\def\P{\mathbb P}
\def\R{\mathbb R}
\def\Beta{\text{Beta}}

\def\SA{\text{SA}}
\def\ConfTr{\text{CTr}}
\def\DPSM{\text{DM}}

\def\UA{\text{CUT}}
\def\QR{\text{QR}}
\def\dist{\text{dist}}

\def\dom{\text{dom}}
\def\Beta{{\text{Beta}}}
\def\Sigmoid{{\text{Sigmoid}}}

\def\indicator{\mathbb{1}}

\def\APS{\text{APS}}
\def\HPS{\text{HPS}}
\def\RAPS{\text{RAPS}}
\def\reg{\text{reg}}

\def\calX{\mathcal X}
\def\calY{\mathcal Y}

\def\calU{\mathcal U}
\def\calC{\mathcal C}

\def\calD{\mathcal D}

\def\calP{\mathcal P}
\def\calQ{\mathcal Q}
\def\calB{\mathcal B}

\def\calF{\mathcal F}
\def\calL{\mathcal L}

\def\E{\mathbb E}
\def\P{\mathbb P}
\def\R{\mathbb R}

\def\min{\text{min}}
\def\sup{\text{sup}}
\def\max{\text{max}}
\def\tr{\text{tr}}
\def\cal{\text{cal}}
\def\test{\text{test}}

\theoremstyle{plain}
\newtheorem{theorem}{Theorem}[section]
\newtheorem{proposition}[theorem]{Proposition}
\newtheorem{lemma}[theorem]{Lemma}
\newtheorem{corollary}[theorem]{Corollary}
\theoremstyle{definition}
\newtheorem{definition}[theorem]{Definition}
\newtheorem{assumption}[theorem]{Assumption}
\theoremstyle{remark}
\newtheorem{remark}[theorem]{Remark}

\usepackage[textsize=tiny]{todonotes}

\icmltitlerunning{ Direct Prediction Set Minimization via Bilevel Conformal Classifier Training }

\begin{document}

\twocolumn[
\icmltitle{Direct Prediction Set Minimization via Bilevel Conformal Classifier Training}

\icmlsetsymbol{equal}{*}

\begin{icmlauthorlist}
\icmlauthor{Yuanjie Shi}{equal,yyy}
\icmlauthor{Hooman Shahrokhi}{equal,yyy}
\icmlauthor{Xuesong Jia}{yyy}
\icmlauthor{Xiongzhi Chen}{sch}
\icmlauthor{Janardhan Rao Doppa}{yyy}
\icmlauthor{Yan Yan}{yyy}
\end{icmlauthorlist}

\icmlaffiliation{yyy}{School of Electrical Engineering and Computer Science, Washington State University, Pullman, Washington, USA}
\icmlaffiliation{sch}{Department of Mathematics and Statistics,  Washington State University, Pullman, Washington, USA}

\icmlcorrespondingauthor{Yan Yan}{yan.yan1@wsu.edu}
\icmlcorrespondingauthor{Janardhan Rao Doppa}{jana.doppa@wsu.edu}

\icmlkeywords{Machine Learning, ICML}

\vskip 0.3in
]



\printAffiliationsAndNotice{\icmlEqualContribution} 

\begin{abstract}
Conformal prediction (CP) is a promising uncertainty quantification framework which works as a wrapper around a black-box classifier to construct prediction sets (i.e., subset of candidate classes) with provable guarantees. 
However, standard calibration methods for CP tend to produce large prediction sets which makes them less useful in practice. 
This paper considers the problem of integrating conformal principles into the training process of deep classifiers to directly minimize the size of prediction sets. 
We formulate conformal training as a bilevel optimization problem and propose the {\em Direct Prediction Set Minimization (DPSM)} algorithm to solve it. 
The key insight behind DPSM is to minimize a measure of the prediction set size (upper level) that is conditioned on the learned quantile of conformity scores (lower level). 
We analyze that DPSM has a learning bound of $O(1/\sqrt{n})$ (with $n$ training samples),
while prior conformal training methods based on stochastic approximation for the quantile has a bound of $\Omega(1/s)$ (with batch size $s$ and typically $s \ll \sqrt{n}$).
Experiments on various benchmark datasets and deep models show that DPSM significantly outperforms the best prior conformal training baseline with $20.46\%\downarrow$ in the prediction set size and validates our theory. 
\end{abstract}

\section{Introduction}
\label{section:introduction}

Deep neural networks have achieved high prediction accuracy and enabled numerous applications in diverse domains including computer vision \cite{he2016deep,he2017mask} and natural language processing \cite{vaswani2017attention,kumar2016ask}. 
However, to safely deploy deep machine learning (ML) classifiers in high-stakes applications (e.g., medical diagnosis \cite{begoli2019need,yang2021uncertainty}) and to build efficient human-ML collaborative systems (e.g., a classifier produces a small set of decisions and a human selects from them \cite{cresswellconformal,straitouriicml2023improving,babbarijcai2022utility}), we need reliable uncertainty quantification (UQ) \cite{abdar2021review}. UQ
could take the form of a prediction set (a subset of classes) for classification tasks. For example, in medical diagnosis \cite{begoli2019need,yang2021uncertainty}, such prediction sets allow a doctor to rule out harmful diagnoses such as stomach cancer, even if the most likely diagnosis is a stomach ulcer.

Conformal prediction (CP) is a general framework that provides finite-sample guarantees of {\em valid} prediction sets to include the correct output, which are agnostic to the ML model and data distribution \cite{vovk2005algorithmic,angelopoulos2021gentle}. 
As a result, CP is increasingly used for UQ in real-world problems \cite{cresswellconformal,lu2022fair,lu2022improving}.
The general principle behind CP is to convert the output of an ML model on a given input example (e.g., softmax scores of a deep classifier) into a prediction set \cite{romano2020classification} that contains the true label with a user-specified probability (e.g., $90\%$ probability) referred to as {\it coverage} via a calibration step. 
One of the main limitations of standard CP is that it tends to produce large prediction sets, which may not be useful in practice \cite{vovk2016criteria,babbarijcai2022utility,straitouriicml2023improving}. 
Recent prior work focused on how to improve the predictive efficiency of CP \cite{fisch2021efficient,fisch2021few,angelopoulos2021uncertainty,huang2023conformal,ding2024class,shi2024conformal}. However, their focus is on the calibration step only by considering the underlying ML model as a black-box.

Recent work has considered integrating CP into classifier training to trade-off between prediction accuracy and {\it conformal alignment} (e.g., minimizing the size of prediction sets) in an end-to-end fashion \cite{stutz2021learning,einbinder2022training,yan2024provably}. 
For instance, ConfTr \cite{stutz2021learning} incorporates a stochastic approximation (SA) of the average prediction set size into the training objective function based on empirical quantiles over mini-batches. 
To achieve accurate conditional coverage, conformal uncertainty-aware training (CUT) \cite{einbinder2022training} employs a similar SA approach to acquire the empirical quantiles of non-conformity scores on mini-batches and penalizes their deviation from the uniformity. 
However, a primary challenge for these SA-based conformal training methods is the large error between the empirical batch-level quantiles and the true quantiles, 
resulting in a learning bound of up to $O(1/\sqrt{s})$ with batch size $s$ \cite{einbinder2022training} .
This bound is significantly larger than the standard learning bound of $O(1/\sqrt{n})$ with $n$ training samples in empirical risk minimization (ERM) \cite{mohri2018foundations,shalev2014understanding}.
Consequently, it notably degenerates the overall trade-off between the prediction accuracy and conformal alignment of the trained classifier. 
This motivates the main question of this paper: 
{\it can we develop a conformal classifier training method to calibrate the conformal uncertainty with a standard $O(1/\sqrt{n})$ learning bound?}

This paper gives an affirmative answer by developing a novel conformal classifier training algorithm referred to as {\em \textbf{D}irect \textbf{P}rediction \textbf{S}et \textbf{M}inimization (DPSM)}. 
The key idea behind DPSM is to formulate conformal training as a bilevel optimization problem \cite{ghadimi2018approximation,zhang2024introduction} that estimates the empirical quantile of non-conformity scores by learning a quantile regression (QR) in the lower-level subproblem.
Conditioned on this learned quantile, the upper-level subproblem measures and minimizes the prediction set size of CP.
This bilevel formulation implicitly defines a conformal alignment objective and 
we analyze that minimizing the resulting function leads to a learning bound of $O(1/\sqrt{n})$.
In contrast, we show that the learning bound of the prior SA-based conformal training methods is lower bounded by $\Omega(1/s)$ with batch size $s$ (typically $s \ll \sqrt{n}$).
Moreover, to optimize the DPSM objective, we develop a simple and practical stochastic first-order algorithm.
Our experiments on diverse real-world datasets with different conformity scores and deep models demonstrate that DPSM achieves significant reduction ($\downarrow 20.46\%$) in prediction set sizes compared to the best baseline.

\noindent {\bf Contributions.} The key contributions of this paper include:
\begin{itemize}
\item We develop a new conformal training algorithm DPSM as a bilevel formulation, which directly minimizes the prediction set size (upper level) conditioned on the learned quantile of the conformity scores (lower level).

\item We analyze the learning bound of DPSM and find that it is bounded above by $O(1/\sqrt{n})$, which significantly improves over the existing conformal training methods whose lower bound is at least $\Omega(1/s)$.

\item 
We develop a simple stochastic first-order algorithm to optimize the DPSM training objective.

\item Experiments on multiple benchmark datasets to demonstrate significant improvements of DPSM over the best prior conformal training baseline ($20.46\% \downarrow$ prediction set size) and to validate our theoretical results. 
The DPSM code is available at \url{https://github.com/YuanjieSh/DPSM_code}.
\end{itemize}

\section{Related Work}
\label{section:related_work}

\noindent {\bf Conformal Prediction} has been extensively studied \citep{angelopoulos2021gentle} recently. 
Key contributions include laying the foundational principles in CP \cite{vovk2005algorithmic,vovk1999machine,shafer2008tutorial}.  
There is some focus on cross-validation methods \citep{vovk2015cross} and the jackknife+ approach \citep{barber2021predictive}, 
but split CP is the most common approach\citep{vovk2005algorithmic,romano2020classification,oliveira2024split}, where calibration of a given pre-trained model is performed on a held-out dataset.  
CP is applied in many tasks including classification \cite{romano2020classification,angelopoulos2021uncertainty,cauchois2021knowing}, 
regression \cite{romano2019conformalized,gibbs2023conformal,gibbs2021adaptive,sesia2021conformal},
computer vision \cite{angelopoulos2024conformal,angelopoulos2022image,bates2021distribution}, 
adversarial attacks \cite{liu2024pitfalls,ghosh2023probabilistically},  
online learning \cite{angelopoulos2023conformal,bhatnagar2023improved},
and generative tasks \cite{GPS}.
Valid coverage and predictive efficiency are two critical and often competing evaluation measures in CP \citep{angelopoulos2021uncertainty}. 
Recent CP advances include the development of new conformity scores \citep{angelopoulos2021uncertainty,huang2023conformal}
and calibration procedures \cite{fisch2021few,fisch2021efficient,guan2023localized,NCP,ding2024class,kiyani2024length,shi2024conformal,zhu2025predicate}
to improve predictive efficiency with valid coverage.
These methods follow the general CP workflow by relying on pre-trained models. Consequently, their predictive efficiency critically depends on the underlying ML model. However, since the ML models are not specifically trained for conformal alignment, CP may produce large prediction sets \citep{bellotti2021optimized}.

\noindent {\bf Conformal Training} aims at integrating CP into the classifier training process to promote conformal alignment. 
A typical approach is to introduce a regularization function that integrates conformity property into the training process\cite{bellotti2021optimized,stutz2021learning,einbinder2022training,yan2024provably}.
For example,
ConfTr \cite{stutz2021learning} integrates CP into the training of deep classifiers through a differentiable approximation for the average prediction set size, 
allowing for smaller and valid prediction sets using the CP-aware classifier.
Conformalized uncertainty-aware training (CUT) \citep{einbinder2022training} combines the conformity scores with a uncertainty-aware loss function for accurate conditional coverage. 
It penalizes the deviation of the conformity scores from the uniformity during the training process. 
However, both ConfTr and CUT employ a stochastic quantile of non-conformity scores from a mini-batch of $s$ samples,
resulting in a large $O(1/\sqrt{s})$ error with respect to the true quantile depending on the batch size $s$.

{\bf Bilevel Optimization}
has been extensively studied due to its critical role in many ML use-cases \cite{liu2021investigating}, 
including 
reinforcement learning \cite{cheng2022adversarially,zhou2024bi,shenprincipled},
federated learning \cite{huang2023achieving,yang2023simfbo},
and
continual learning \cite{borsos2020coresets,zhou2022probabilistic,hao2024bilevel,borsos2024data}.
Bilevel optimization methods can be roughly categorized into three groups:
(i)
implicit gradient methods \cite{chen2021closing,ji2021bilevel,dagreou2022framework,ghadimi2018approximation,pedregosa2016hyperparameter,domke2012generic},
(ii)
iterative differentiation methods \cite{bolte2022automatic,bolte2021nonsmooth},
and
(iii)
penalty-based methods \cite{lu2024first,chen2024penaltybased,chen2024finding,shenicml2023on,shenmapr2025on}.
Compared to penalty-based methods, implicit gradient methods typically require more restrictive assumptions, such as twice differentiable and strongly-convex lower-level function \cite{shenmapr2025on}.
On the other hand, iterative differentiation methods typically require an iterative subroutine for the lower-level problem with more computational cost \cite{shenmapr2025on}.
Instead, penalty-based methods build a single level problem and only use first-order information \cite{shenicml2023on}.
To summarize, bilevel optimization algorithms either require restrictive assumptions or have complex update steps.
It is non-trivial to find a practical algorithm that is particularly tailored for our proposed conformal training problem.

Our goal is to develop a bilevel conformal training algorithm with
(i) a standard $O(1/\sqrt{n})$ learning bound for $n$ training samples to improve over $O(1/\sqrt{s})$ for batch size $s$ from prior conformal training methods, and 
(ii) a practical stochastic first-order bilevel optimization algorithm.

\section{ Background and Motivation }
\label{section:notation_and_bg}

\noindent {\bf Notations.}
Suppose $X \in \calX$ is an input from $\mathcal{X}$, and $Y \in \calY = \{ 1,2,\cdots, K\}$ is the ground-truth label, where $K$ is the number of candidate classes. 
Let $(X, Y)$ be a data sample drawn from an underlying distribution $\calP$ defined on $\calX \times \calY$. 
Let $f(X): \calX \rightarrow \Delta_+^K$ denote the confidence prediction by soft classifier $f \in \calF$ (e.g., Softmax scores),
where $\Delta_+^K$ is the (K-1)-dimensional probability simplex,
$f(X)_y$ is the confidence score of class $y$ and 
$\calF$ denotes a hypothesis class.
Define $\calD_\tr = \{ (X_i, Y_i) \}_{i=1}^n$ as the training set of $n$ samples for training classifier $f$.
We denote $\calD_\cal$ and $\calD_\test$ as two separate calibration and testing datasets, respectively.
We assume that there are $m$ calibration samples, i.e., $|\calD_\cal| = m$.
Define $\indicator[\cdot]$ as an indicator function.
Denote $\calB = \{ (X_i, Y_i) | i \in I_s\}$ as a randomly sampled batch of training data of size $s$, such that the batch index set $I_s \subset \{1, 2, \cdots, n \}$ with $| I_s | =s$.
Given a function $h(x, y)$, we denote $\nabla_x h(x,y)$ as the directional derivative of $h$ in $x$ and $\widehat \nabla_x h(x, y)$ as its stochastic estimation.

\noindent
{\bf Conformal Prediction.} 
As mentioned above, CP calibrates predictions from a given model based on a {\it non-conformity} scoring function.
Let $S: \calX \times \calY \rightarrow \R$ denote a non-conformity scoring function to measure the difference between a new data and existing ones \cite{vovk2005algorithmic}. 
For simplicity of notation, we denote the non-conformity score of the $i$-th example as 
$S_{i}$ = $S_f(X_i, Y_i)$ for $(X_i, Y_i) \in \calD_\tr$ 
and of the $j$-th smallest value in $\{S_{i}\}^n_{i=1}$ as $S_{(j)}$.
Many non-conformity scoring functions are proposed in prior work
\cite{huang2023conformal,angelopoulos2021uncertainty,shafer2008tutorial}.
In this paper, we consider Homogeneous Prediction Sets (HPS) \cite{sadinle2019least}, Adaptive Prediction Sets (APS) \cite{romano2020classification}, and Regularized Adaptive Prediction Sets (RAPS) \cite{angelopoulos2021uncertainty}
(details are in Appendix \ref{appendix:section:additional_details}).
We assume that there is no tie in non-conformity scores \cite{romano2020classification}.

Given the mis-coverage parameter $\alpha$ and an underlying model $f$, on a testing data $(X_\test, Y_\test)$, CP aims to achieve a coverage of the true label with a prediction set $\calC_f: \calX \rightarrow 2^{|\calY|}$ with probability at least $1-\alpha$, i.e., 
\begin{align}
\label{eq:cp_coverage}
\P_{(X_\test, Y_\test)\backsim \calP} \{ Y_{\test} \in \calC_f(X_{\test}) \} 
\geq 
1- \alpha
.
\end{align}
CP typically constructs $\calC_f(X_{\test})$ via the empirical quantile $\widehat Q_f(\alpha)$ on the calibration set $\calD_\cal$:
\begin{align}
\label{eq:cp_construct_prediction_set}
&
\calC_f(X_\test)
=
\{ y \in \calY : S_f(X_{\test}, y) 
\leq 
\widehat Q_f(\alpha) \}
,
\end{align}
where $\widehat Q_f(\alpha)$ is computed as the $\lceil 
(1-\alpha)(1+m) \rceil$th smallest value in $\{ S_{f,i} \}_{i=1}^m$ and is an estimation to the population quantile $Q_f(\alpha) = \min\{ \tau: \P_{X, Y}[ S_f(X, Y) \leq \tau ] 
\geq 
1 - \alpha
\}
$.
Throughout this paper, unless otherwise specified, we omit the target mis-coverage $\alpha$ for notation simplicity, e.g., $Q_f$.

\noindent
{\bf Quantile-based Conformal Training.} The goal is to align the classification model $f$ with favorable conformal principles \cite{bellotti2021optimized}. 
A general framework \cite{stutz2021learning,einbinder2022training} is to establish the trade-off between prediction accuracy and conformal alignment on the training set $\calD_\tr$ with an objective structured as follows:
\begin{align}\label{eq:conformal_training_general_form}
\min_f ~ \calL_\cls(f) + \lambda \cdot \calL_c(f)
,
\end{align}
where $\calL_\cls(f)$ is the classification loss (e.g., cross entropy), 
$\calL_c(f)$ is the conformal alignment loss based on the quantile $Q_f$,
and $\lambda$ is the regularization hyper-parameter. 
Following \cite{stutz2021learning}, this paper focuses on a concrete definition of $\calL_c(f)$ as the expected differentiable size of prediction sets:
\begin{align}
\label{eq:conformal_loss_population}
&
\calL_c(f) 
=
\ell(f, Q_f),
\text{ such that }
\\
&
\ell(f, q)
\triangleq
\E_X 
\underbrace{ 
\Bigg[ 
\sum_{y \in \calY} \tilde \indicator \big[ S_f(X, y) \leq q \big] 
\Bigg]
}_{ {\text{differentiable }} \E_X[|\calC_f(X)|] }
,
\nonumber
\end{align}
where $\tilde \indicator[\cdot]$ is a smoothed estimator for the indicator function $\indicator[\cdot]$ and defined by a Sigmoid function \cite{stutz2021learning}, 
i.e.,
$\tilde \indicator[ S \leq q ] = 1/(1+\exp(-(q-S)/\tau_\Sigmoid))$ with a tunable temperature parameter $\tau_\Sigmoid$.
Here, we denote $\ell(f,q)$ as a general conformal loss function that takes any $f$ and $q$ as inputs,
and 
$\calL_c(f)$ requires the true quantile $Q_f$ as the input to $\ell$.

However, it is challenging to accurately compute $\calL_c(f, Q_f)$ when training $f$,
as {\em $Q_f$ is an implicit function of $f$}.
On one hand, updating $f$ iteratively also changes $Q_f$ accordingly.
On the other hand, after getting an updated $f$, re-computing an empirical quantile $\widehat Q_f$ on training data $\calD_\tr$ requires sorting all $n$ non-conformity scores of training data, with a computational complexity of $O(n \log(n))$.

To address the above-mentioned challenge, prior conformal training methods \cite{stutz2021learning,einbinder2022training} employ a stochastic approximation (SA) method to estimate $\calL_c(f)$ in batches.
Specifically, at each iteration during training, they use the empirical batch-level quantile from a mini-batch $\calB$ sampled from $\calD_\tr$, denoted by $\widehat q_f$, as the input to $\ell$, i.e., $\ell(f, \widehat q_f)$.
Since $\widehat q_f$ depends on the randomly sampled batches, we regard $\widehat q_f$ as a random variable drawn from its underlying distribution $\calQ_f$, 
i.e., $\widehat q_f \sim \calQ_f$
(we defer our analysis for $\calQ_f$ to Proposition \ref{proposition:quantile_distribution} stated below).
Then the in-effect SA-based conformal loss $\widehat \calL_c^\SA(f)$ is given by 
\begin{align}
\label{eq:conformal_training_conftr_SA}
\widehat \calL_c^\SA(f) 
= &
\E_{\widehat q_f \sim \calQ_f} [ \widehat \ell(f, \widehat q_f) ]
,
\text{ such that }
\\
\widehat \ell(f, q)
\triangleq &
\frac{1}{n} \sum_{i=1}^n \Bigg[ \sum_{y \in \calY } \tilde \indicator \big [ S_f(X_i, y) \leq q \big ] \Bigg]
,
\nonumber
\end{align}
\noindent
where $\widehat \ell(f,q)$ is an empirical version of $\ell(f,q)$ on $\calD_\tr$.

A key limitation of the SA method is the large error gap between the true and in-effect conformal alignment loss functions, i.e., $| \widehat \calL_c^\SA(f) - \calL_c(f) |$.
Prior work \cite{einbinder2022training} showed that the error of SA-based conformal alignment loss is upper bounded by $O(1/\sqrt{s})$ with batch size $s$.
We further analyze the learning error for the SA-based conformal alignment loss in this paper, including its lower bound.
We start with the following assumptions.
\begin{assumption}
\label{assumption:bi_lipschitz}
(Bi-Lipschitz continuity of score $S_{(j)}$) 
$S_{(j)}$ is bi-Lipschitz continuous for normalized order $j/n$ such that $ L_1 |\frac{j_1}{n} - \frac{j_2}{n} |
\leq 
| S_{(j_1)} - S_{(j_2)}| \leq L_2 |\frac{j_1}{n} - \frac{j_2}{n} |$ for $L_1, L_2 > 0$.
\end{assumption}

\begin{assumption}
\label{assumption:straongly_concave}
($\mu$-strongly concavity of $\widehat \ell(f, q)$) 
$\widehat \ell(f, q)$ is $\mu$-strongly concave locally around $\E[\widehat q_f]$ for fixed $f$.
\end{assumption}

\begin{remark}
We conduct experiments to empirically verify the Assumptions \ref{assumption:bi_lipschitz} and \ref{assumption:straongly_concave} (See Fig \ref{fig:sodt_set_size_comparisons_main}).
These assumptions allow us to analyze how the SA-based conformal training method approaches to the true conformal alignment loss.
\end{remark}

The following proposition reveals the specification of the distribution $\calQ_f$ aforementioned in $\widehat \calL_c^\SA$ (\ref{eq:conformal_training_conftr_SA}) that considers $\widehat q_f$ as a random variable drawn from the training dataset $\calD_\tr$.

\begin{proposition}
\label{proposition:quantile_distribution}
(Distribution of mini-batch quantiles)
Define an event $Z(j)$ as ``the $j$-th smallest score $S_{(j)}$ from $\{S_i\}_{i=1}^n$ is randomly selected as $\widehat q^s(\alpha)$ on a mini-batch $\calB$''. 
Then the probability of $Z(j)$ is:
\begin{align*}
\P(Z(j)) = \frac{\tbinom{j}{\lceil(1-\alpha)(s+1)\rceil } \tbinom{n-j-1}{s-\lceil(1-\alpha)(s+1)\rceil-1}} {\tbinom{n}{s}}.
\end{align*}
Furthermore, we have the following asymptotic result
\begin{align*}
\lim_{n \to \infty} \P(Z(j)) = 
\frac{e}{n}\P_\Beta\Big(
\frac{j}{n}; &
\lceil(1-\alpha)(s+1)\rceil+1; 
\\
&
s- \lceil(1-\alpha)(s+1)\rceil \Big)
,
\end{align*}
where $\P_\Beta(x;a;b)$ is the PDF of Beta distribution with two shape parameters $a, b$ and $e$ is the Euler's number.
\end{proposition}

With the above proposition that captures the random sampling process of the quantile $\widehat q_f$ in each mini-batch, we are ready to present the following main result analyzing the learning bound of the SA-based conformal training method.

\begin{theorem}
\label{theorem:learning_bounds_SA}
(Learning bounds of SA method)
Suppose that Assumption \ref{assumption:bi_lipschitz} and \ref{assumption:straongly_concave} hold.
Assume $(1-\alpha)(s+1)$ is not an integer.
If $\lceil (1-\alpha)(s+1) \rceil - (1-\alpha) (s+1) \geq \Omega(1/s)$ and $s \leq \sqrt{n}$,
then
the following inequality holds with probability at least $1 - \delta$:
\begin{align*}
\Omega (1/s) 
\leq 
| \widehat \calL_c^\SA(f) - \calL_c(f)  | \leq \tilde O (1/\sqrt{s}).
\end{align*}
\end{theorem}

\begin{remark}
Throughout this paper, we use $\tilde O$ to suppress the $\log$ dependency.
The above result shows the upper bound $\tilde O(1/\sqrt{s})$ and lower bound $\Omega(1/s)$ of learning error of the SA-based conformal training method. 
It is worth noting that even the lower bound $\Omega(1/s)$ can be significantly larger than the standard result $O(1/\sqrt{n})$ as in empirical risk minimization \cite{mohri2018foundations} under a common setting $s \leq \sqrt{n}$ for training deep models \cite{masters2018revisiting}.
\end{remark}

\section{ Direct Prediction Set Minimization Method}
\label{section:method}

In this section, we first propose {\em \textbf{D}irect \textbf{P}rediction \textbf{S}et \textbf{M}inimization (DPSM)} algorithm by formulating it as a bilevel optimization problem.
Next, we show that the learning bound of DPSM is at most $O(1/\sqrt{n})$ (Theorem \ref{theorem:learning_bound_DPSM}), which improves over the SA-based conformal training methods ($\Omega(1/s)$, Theorem \ref{theorem:learning_bounds_SA}).
Finally, we develop a simple and practical optimization algorithm to solve it.

\subsection{Bilevel Problem Formulation and Learning Bound }
\label{subsection:method}

The key idea to directly minimize the prediction sets is inspired by quantile regression (QR), a well-studied method to learn the quantile of a set of random variables.
A common approach to training QR is to minimize an average pinball loss on a set of data \cite{narayanexpected,gibbs2023conformal,koenker2005quantile,koenker1978regression}.
For the $(1-\alpha)$-quantile, the pinball loss is defined as follows.
\begin{align}
\label{eq:pinball_loss}
\rho_\alpha (q, S) 
= 
\begin{cases}
( 1 - \alpha ) ( S - q ),   & \text{if } S \geq q , \\
\alpha ( q - S ),           & \text{otherwise,}
\end{cases}
\end{align}
where $S$ and $q$ represent a real value and quantile prediction, respectively.
Minimizing the average pinball loss on a set of data $\{S_i\}_{i=1}^n$ gives the $(1-\alpha)$-quantile of $n$ scores $\{S_i\}_{i=1}^n$:
\begin{align}
q^*
\in
\arg\min_q \frac{1}{n} \sum_{i=1}^n \rho_\alpha ( q, S_i )
.
\end{align}

\noindent {\bf Conformal Training via Bilevel Optimization.}
Instead of using a stochastic quantile on each batch of data, we propose to incorporate QR into the conformal training objective as a constraint.
Conditioned on solving this QR-based constraint and acquiring the accurate quantile, we directly minimize the average size of prediction sets.
Specifically, we formulate the following {\it bilevel optimization} problem for our DPSM method for conformal training:
\begin{align}\label{eq:conformal_training_bilevel}
\min_{f, q} ~
&~
\widehat \calL_\cls(f) + \lambda \cdot \widehat \calL_c^\DPSM(f, q)
\\
s.t. ~
&~
q
\in 
\calU(f)
\triangleq
\arg\min_{q' } ~ \widehat \calL^\QR(f, q')
,
\nonumber
\end{align}
\noindent
where $\widehat \calL^\QR(f, q) = \frac{1}{n} \sum_{i =1}^n \rho_\alpha(q, S_f(X_i, Y_i))$ in the lower level is the average pinball loss for the non-conformity scores on $n$ training samples, and
the conformal alignment loss $\widehat \calL_c^\DPSM(f, q)$ is given by 
\begin{align}\label{eq:conformal_loss_DPSM}
\widehat \calL_c^\DPSM(f, q)
=
\frac{1}{n} \sum_{i =1}^n \Bigg[ \sum_{y\in\calY} \tilde \indicator[ S_f(X_i, y) \leq q ] \Bigg]
,
\end{align}
which follows the definition of $\widehat \calL_c^\SA(f)$ in (\ref{eq:conformal_training_conftr_SA}) except that it allows any variable $q$ as the input quantile, instead of a stochastic quantile $\widehat q_f \sim \calQ_f$.
To analyze DPSM, 
following the bilevel optimization literature \cite{chen2024finding,shenicml2023on},
we define the {\it implicit conformal loss}:
\begin{align*}
\bar \calL_c^\DPSM(f) \triangleq \min_{q \in \calU(f)} \widehat \calL_c^\DPSM(f, q)
.
\end{align*}
The key innovation in the proposed bilevel conformal training problem (\ref{eq:conformal_training_bilevel}) compared with the SA-based method is to {\it explicitly parameterize the quantile} with a single score variable $q$.
This approach has been used in the CP literature \cite{gibbs2023conformal,gibbsjmlr2024conformal,kiyani2024conformal,deshpande2024online,podkopaev2024adaptive}, and we employ it for conformal training in our bilevel formulation for the first time.
Conditioned on $q$, regardless of how accurate $q$ is during training relative to the true $Q_f$, the conformal alignment loss $\widehat \calL_c^\DPSM(f, q)$ participates in training.

The major benefit of the quantile parameterization idea is that it {\it decouples the quantile $q$ from the stochastic batches during training $f$}, rather than the immediate dependency as in the SA-based conformal training, i.e., $\widehat q_f \sim \calQ_f$.
Specifically, QR in the lower-level subproblem, due to its well-known property, enables the iterative updates for $q$ until attaining the ($1-\alpha$)-quantile of non-conformity scores as training $f$.
If an optimization algorithm ensures the simultaneous convergence of both variables in Problem (\ref{eq:conformal_training_bilevel}) to their optimal solution sets (discussed in Section \ref{subsection:optimization}),
i.e., 
\begin{align*}
q \rightarrow \calU(f),
\quad
f \rightarrow \arg\min_f \widehat \calL_\cls(f) + \lambda \cdot \bar \calL_c^\DPSM(f)
,
\end{align*}
then we can achieve an improved learning bound for the conformal alignment loss $|\bar \calL_c^\DPSM(f) - \calL_c(f)|$ that is independent from batch size $s$.
This is the reason why we named our method as {\it direct prediction set minimization (DPSM)}.

{\bf Improved Learning Bound of DPSM.}
As in Theorem \ref{theorem:learning_bounds_SA}, we show that the learning error of SA-based method $| \widehat \calL_c^\SA(f) - \calL_c(f) |$ is upper bounded by $O(1/\sqrt{s})$ and lower bounded by $\Omega(1/s)$, respectively.
We then show that the learning error of DPSM $| \bar \calL_c^\DPSM(f) - \calL_c(f) |$ can be bounded above by $O(1/\sqrt{n})$ in the following result.
\begin{theorem}
\label{theorem:learning_bound_DPSM}
(Learning bound of DPSM)
Suppose Assumption \ref{assumption:bi_lipschitz} holds.
For any classifier model $f \in\calF$, with probability at least $1 - \delta$, we have:
\begin{align*}
| \bar \calL_c^\DPSM ( f ) - \calL_c(f) | 
\leq 
\tilde O(1/\sqrt{n}).
\end{align*}
\end{theorem}
\begin{remark}
The above theorem directly answers the central question that we asked in the introduction, 
i.e., bilevel conformal training calibrates the conformal uncertainty with a standard $O(1/\sqrt{n})$ learning bound as in ERM \cite{mohri2018foundations,shalev2014understanding}.
In comparison with the learning error bounds of the SA-based conformal training method derived in Theorem \ref{theorem:learning_bounds_SA},
the DPSM bound improves by dropping the dependency on the batch size $s$.
\end{remark}

The improved learning bound of DPSM over SA-based methods is mainly due to  smaller error in estimating the true quantile, i.e., $| q - Q_f |$. 
Indeed, results shown in Fig \ref{fig:results_main} (c) and (d) empirically demonstrates this hypothesis. 
The practical significance of this result is that the smaller estimation error of DPSM results in improved conformal alignment of the classifier leading to smaller prediction set sizes.

\subsection{ Stochastic Optimization for DPSM }
\label{subsection:optimization}

{\bf Challenges for DPSM.}
There is a large literature on methods for bilevel optimization \cite{liu2021investigating,beck2023survey}.
However, 
existing methods typically require restrictive assumptions to guarantee convergence.
Some examples of such assumptions include requiring Lipscthiz Hessians \cite{chen2024finding,chen2021closing,dagreou2022framework,yang2021provably}; 
twice continuously differentiable \cite{kwoniclr2024on};
requiring solving a subproblem \cite{liu2020generic};
strongly convex lower-level subproblem \cite{gongicml2024nearly,kwon2023fully,ji2021bilevel,gong2024an,chen2024optimal,hongsiamtwo,khandurineurips2021a}; 
convexity in both upper and lower levels \cite{sabachsiam2017first}; 
unique lower-level solution \cite{liuneurips2022bome};
and 
Lipschitz smoothness \cite{shenicml2023on} (less restrictive).

In the DPSM problem (\ref{eq:conformal_training_bilevel}), the standard classification loss $\widehat \calL_\cls(f)$ and the conformal alignment loss $\widehat \calL_c^\DPSM(f,q)$ are not necessarily convex in $f$ and $q$.
Meanwhile, the QR loss $\widehat \calL^\QR(f, q)$ in the lower level is not continuously differentiable, not strongly convex, and typically has non-unique solution set even for the case without any tie in conformity scores. 
To the best of our knowledge, there is no prior stochastic optimization algorithm that solves DPSM problem with the assumptions of bilevel optimization literature satisfied and convergence guaranteed.
Consequently, it is {\it non-trivial and an open challenge to develop a simple stochastic gradient optimization algorithm to solve a problem that shares similar conditions with DPSM}.
We leave this challenge for future work.

\begin{algorithm}[t]
\caption{Direct Prediction Set Minimization (DPSM)}
\label{alg:DPSM}
\begin{algorithmic}[1]

\STATE \textbf{Input}: Training dataset $\calD_\tr$, regularization parameter $\lambda$, learning-rate $\eta, \gamma >0$, mis-coverage level $\alpha$

\STATE Randomly initialize the deep neural network $f_0$ and quantile regression model parameter $q_0 \in \R$

\STATE Randomly split $\calD_\tr$ into two disjoint subsets $\calD_1$ and $\calD_2$ with the same size such that $| \calD_1 | = | \calD_2 | = \frac{n}{2}$
\label{alg:line:split}

    \FOR{$t \leftarrow 0 : T-1$}

        \STATE Randomly sample two batches $\calB_t^1 \subset \calD_1, \calB_t^2 \subset \calD_2$
        
        \STATE Compute 
        $\widehat \nabla_{f,t}
        \leftarrow 
        \widehat \nabla_f\widehat \calL_\cls(f_{t-1})$ on batch $\calB_t^1$
        \label{alg:line:ce_gradient_f}

        \STATE Compute 
        $\widehat \nabla_{q,t}^\QR
        \leftarrow 
        \widehat \nabla_q \widehat \calL^\QR(f_{t-1}, q_{t-1})$ on batch $\calB_t^1$
        \label{alg:line:qr_loss_gradient_q}
        
        \STATE Compute $\widehat \nabla_{f,t}^\DPSM 
        \leftarrow 
        \widehat \nabla_f \widehat \calL_c^\DPSM(f_{t-1}, q_{t-1})$ on batch $\calB_t^2$
        \label{alg:line:conformal_loss_gradient_f}

        \STATE $f_{t+1} 
        \leftarrow 
        f_t - \eta \big( \widehat \nabla_{f,t} + \lambda \widehat \nabla_{f,t}^\DPSM \big)$
        \label{alg:line:update_f}
          
        \STATE $q_{t+1} 
        \leftarrow 
        q_t - \gamma \widehat \nabla_{q,t}^\QR$ 
        \label{alg:line:update_q}
        
    \ENDFOR
    \label{alg:line:reinitialize}
\STATE \textbf{Output}: the trained classification model $f_T$
\end{algorithmic}
\end{algorithm}

{\bf Penalty-based Reformulation.} 
Given that the lower-level function in DPSM is not strongly convex, implicit gradient methods,e.g., \cite{ghadimi2018approximation}, cannot handle it.
Penalty-based methods is another large family in the bilevel optimization literature.
A recent paper \cite{chen2024penaltybased} studied a penalty-based method when the lower-level problem satisfies H\"olderian error bound, and meanwhile both upper and lower level objectives are non-smooth, which is one of the most relevant setting to the DPSM problem at hand.
Hence, we reformulate DPSM as a penalty-based problem and analyze the global solution of the penalty-based DPSM problem with a penalty parameter $\sigma$.
\begin{align}
\label{eq:penalty_baesd_DPSM}
\min_{f,q}
&
\widehat \calL_\cls(f) + \lambda \cdot \widehat \calL_c^\DPSM(f,q) 
\\
&
+ \sigma ( \widehat \calL^\QR(f,q) - \widehat \calL^\QR(f,q^*) )
,
\text{ where }
q^* \in \calU(f).
\nonumber
\end{align}

\begin{definition}
\label{definition:HEB}
(H\"olderian error bound)
A function $h(x)$, where its domain $\dom(h)$ is a closed convex set, satisfies H\"olderian error bound if there exists $\nu \geq 1$ and $c > 0$ s.t.
\begin{align*}
\dist(x, X^*)^\nu 
\leq 
c ( h(x) - \min_{x' \in X^*} h(x') ),
\forall x \in \dom(h),
\end{align*}
where $X^* = \arg\min_{x \in \dom(h)} h(x)$ is the optimal solution set for minimizing $h(x)$ and 
$\dist(x, X^*) = \min_{x'} \| x - x'\|$ denotes the Euclidean distance between $x$ and $X^*$.
\end{definition}

\begin{remark}
H\"olderian error bound (HEB) is a well-studied condition in optimization \cite{pang1997error,bolte2017error}.
It captures the local sharpness of the objective function that helps accelerate the optimization convergence \cite{roulet2017sharpness,yangjmlr2018rsg}.
The following result shows that $\widehat \calL_c^\DPSM(f, q)$ satisfies HEB w.r.t. $q$.
\end{remark}

\begin{lemma}
\label{lemma:HEB_QR_loss}
(HEB for QR loss)
Suppose there is no tie in conformity scores $\{S_i\}_{i=1}^n$.
Then for a fixed $f$, $\widehat \calL^\QR(f, q)$ satisfies HEB w.r.t. $q$ for the exponent $\nu=1$ and $c > 0$.
\end{lemma}
\begin{corollary}
\label{corollary:global_local_solution}
(Global solution of penalized problem)
Suppose $\widehat \calL_\cls(f)$ and $\widehat \calL_c^\DPSM(f,q)$ is $L_\cls$- and $L_\DPSM$-Lipschitz continuous, respectively.
For any given $\epsilon > 0$,
let 
$l = ( L_\cls + L_\DPSM)$
and
$\sigma = c l$ in (\ref{eq:penalty_baesd_DPSM}).
Then the $\epsilon$-optimal solution of the penalized problem is an $(\epsilon, \epsilon/l)$-optimal solution of the original problem.
\end{corollary}
This is an immediate result from Theorem 2.7 in \cite{chen2024penaltybased} 
and 
connects the solution of the penalized DPSM problem (\ref{eq:penalty_baesd_DPSM}) to that of the original DPSM problem (\ref{eq:conformal_training_bilevel}).

{\bf Simple Stochastic Gradient-based Algorithm.}
We further present a simple stochastic gradient algorithm to optimize Problem (\ref{eq:conformal_training_bilevel}).
We summarize this approach in Algorithm \ref{alg:DPSM}.
It mainly follows the common schemes in standard stochastic gradient methods to iteratively update the two variables $f$ and $q$ simultaneously.
First, it randomly splits the training set $\calD_\tr$ into two disjoint subsets $\calD_1$ and $\calD_2$ with the same size (Line \ref{alg:line:split}).
Next, given the fixed model $f_{t-1}$ and quantile variable $q_{t-1}$, we compute three stochastic gradients over mini-batches in each training iteration:
(i) $\widehat \nabla_f \widehat \calL_\cls(f_{t-1})$ on mini-batches of $\calD_1$ (Line \ref{alg:line:ce_gradient_f}),
(ii) $\widehat \nabla_f \widehat \calL_c^\DPSM(f_{t-1}, q_{t-1})$ on mini-batches of $\calD_2$ (Line \ref{alg:line:conformal_loss_gradient_f}), and
(iii) $\widehat \nabla_q \widehat \calL^\QR(f_{t-1}, q_{t-1})$ on mini-batches of $\calD_1$ (Line \ref{alg:line:qr_loss_gradient_q}).
Next, we update the model $f_t$ (Line \ref{alg:line:update_f}) and quantile $q_t$ (Line \ref{alg:line:update_q}) accordingly with their learning rates $\eta_t$ and $\gamma_t$.

\begin{table*}[!t]
\centering
\caption{
\textbf{HPS $\shortrightarrow$ Training, HPS and APS $\shortrightarrow$ Calibration/Testing} (details in Appendix \ref{subsec:train_test_strategies}): 
The average prediction set size (APSS) on three datasets with two deep models trained with HPS and calibrated/tested with HPS and APS when $\alpha = 0.1$. 
$\downarrow$ indicates the percentage improvement in predictive efficiency compared to the best existing method, whereas $\uparrow$ denotes a percentage decrease in predictive efficiency. 
All results are averaged over 10 different runs, with the mean and standard deviation reported.
DPSM significantly outperforms almost all the best baselines with $20.46\%$ 
reduction in prediction set size on average across all three datasets and two scores. 
}
\label{tab:cvg_set_hps_train_main}
\resizebox{\textwidth}{!}{
\begin{NiceTabular}{@{}c|cccc|cccc @{}}
\toprule
\multirow{2}{*}{Model} & \multicolumn{4}{c|}{HPS} & \multicolumn{4}{c}{APS} \\ 
\cmidrule(lr){2-5} \cmidrule(lr){6-9}
& CE & CUT & ConfTr & DPSM & CE & CUT & ConfTr & DPSM  \\ 
\midrule
\Block{1-*}{CalTech-101}
\\
\midrule
DenseNet 
& 3.50 $\pm$ 0.10  & 1.62 $\pm$ 0.030 & 4.10 $\pm$ 0.19  & \textbf{0.90 $\pm$ 0.003} ($\downarrow$ 44.44\%) 
& 8.44  $\pm$ 0.15 & 3.87 $\pm$ 0.11  & 8.64 $\pm$ 0.21  & \textbf{1.58 $\pm$ 0.022} ($\downarrow$ 59.17\%) 
\\ 
ResNet   
& 1.57 $\pm$ 0.018 & 1.64 $\pm$ 0.049 & 1.52 $\pm$ 0.040 & \textbf{0.91 $\pm$ 0.005} ($\downarrow$ 44.51\%) 
& 4.50 $\pm$ 0.059 & 4.59 $\pm$ 0.072 & 3.61 $\pm$ 0.08 & \textbf{1.74 $\pm$ 0.031} ($\downarrow$ 51.80\%) 
\\ 
\midrule
\Block{1-*}{CIFAR-100}
\\
\midrule
DenseNet 
& 2.59 $\pm$ 0.053 & 2.27 $\pm$ 0.09  & 2.28 $\pm$ 0.07  & \textbf{2.17 $\pm$ 0.086}  ($\downarrow$ 4.82\%) 
& 3.38 $\pm$ 0.12 & \textbf{2.41 $\pm$ 0.11} & 3.08 $\pm$ 0.11   & 2.64 $\pm$ 0.086 ($\uparrow$ 8.71\%)
\\ 
ResNet   
& 3.39 $\pm$ 0.10  & 3.01 $\pm$ 0.11  & 3.77 $\pm$ 0.14  & \textbf{2.94 $\pm$ 0.08}  ($\downarrow$ 2.32\%) 
& 3.98 $\pm$0.13 & 3.81 $\pm$ 0.08 & 4.90 $\pm$ 0.18  & \textbf{3.53 $\pm$ 0.11} ($\downarrow$ 7.35\%)
\\ 
\midrule
\Block{1-*}{iNaturalist}
\\
\midrule
DenseNet 
& 94.58 $\pm$ 3.45 & 77.13 $\pm$ 3.72 & 79.93 $\pm$ 3.70 & \textbf{61.22 $\pm$ 2.49} ($\downarrow$ 20.63\%) 
& 101.97 $\pm$ 3.16 & 88.93 $\pm$ 3.06 & 90.79 $\pm$ 3.17 & \textbf{75.98 $\pm$ 2.99} ($\downarrow$ 14.56\%)
\\ 
ResNet   
& 99.48 $\pm$ 8.95 & 73.09 $\pm$ 2.00 & 76.73 $\pm$ 3.87 & \textbf{70.04 $\pm$ 1.99} ($\downarrow$ 4.17\%) 
& 95.81 $\pm$ 3.80 & \textbf{79.00 $\pm$ 2.21} & 88.70 $\pm$ 3.88   & 79.43 $\pm$ 2.39 ($\uparrow$  0.54\%) 
\\ 
\bottomrule
\end{NiceTabular}
}
\end{table*}

It is worth noting that the proposed DPSM optimization algorithm is particularly tailored for our bilevel conformal training problem (\ref{eq:conformal_training_bilevel}). 
{\bf 1)} We employ $\calD_1$ to compute $\widehat \nabla_f \calL_\cls(f)$ and $\widehat \nabla_q \widehat \calL^\QR(f, q)$, while evaluate $\widehat \nabla_f \widehat \calL_c^\DPSM(f, q)$ on $\calD_2$.
This has been used in conformal training \cite{einbinder2022training} and helps to prevent over-fitting.
{\bf 2)} We rely on simple first-order information to update, rather than hyper-gradient or penalization from the existing bilevel optimization methods.
This design keeps the iterative update simple especially for practitioners. Our empirical results in Figure \ref{fig:results_main} demonstrate stable convergence of both upper and lower level problems of DPSM by employing standard practices for learning rates $\eta$ and $\gamma$.

\section{Experiments and Results}
\label{section:experiment}
This section describes our experimental evaluation of the proposed DPSM algorithm and baselines on real datasets.
\subsection{Experimental Setup}
\label{subsection:experiment_setup}

\begin{figure*}[!t]
    \centering
    \begin{minipage}[t]{0.32\linewidth}
    \centering
    \textbf{(a)} Caltech-101
    \includegraphics[width = \linewidth]{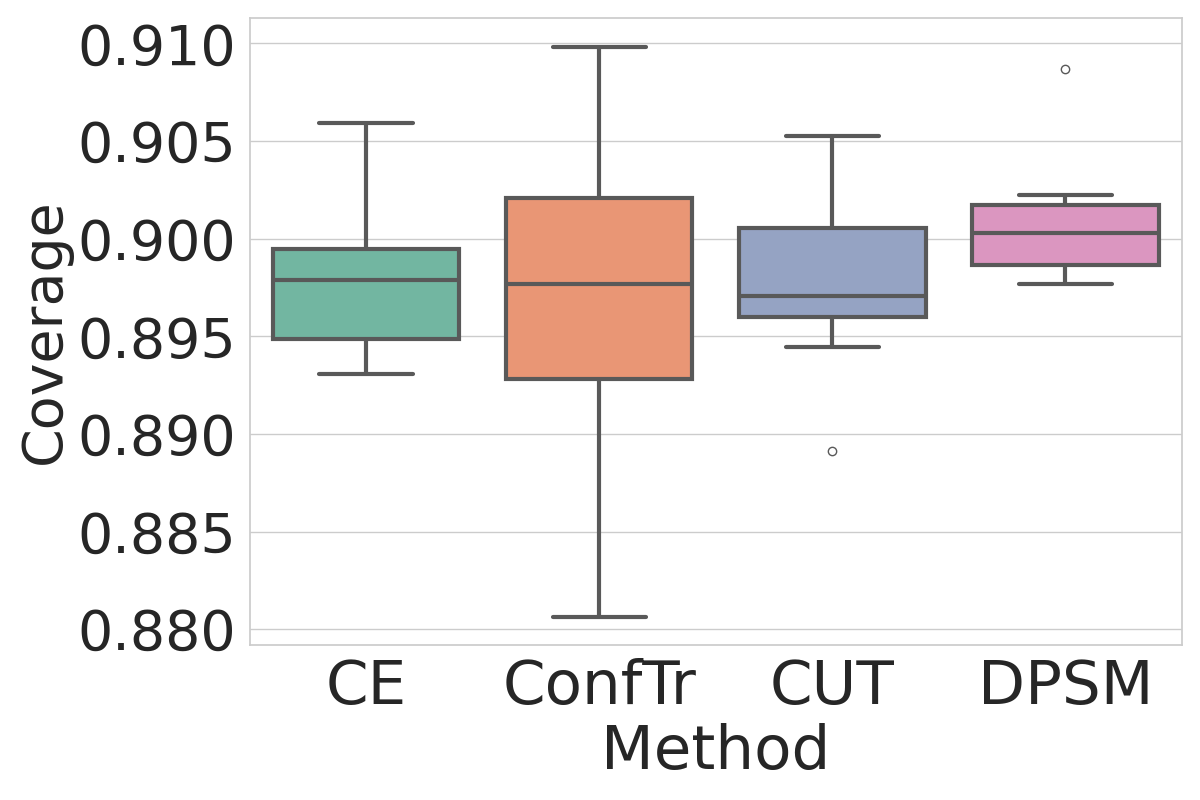}
    \\
    \includegraphics[width = \linewidth]{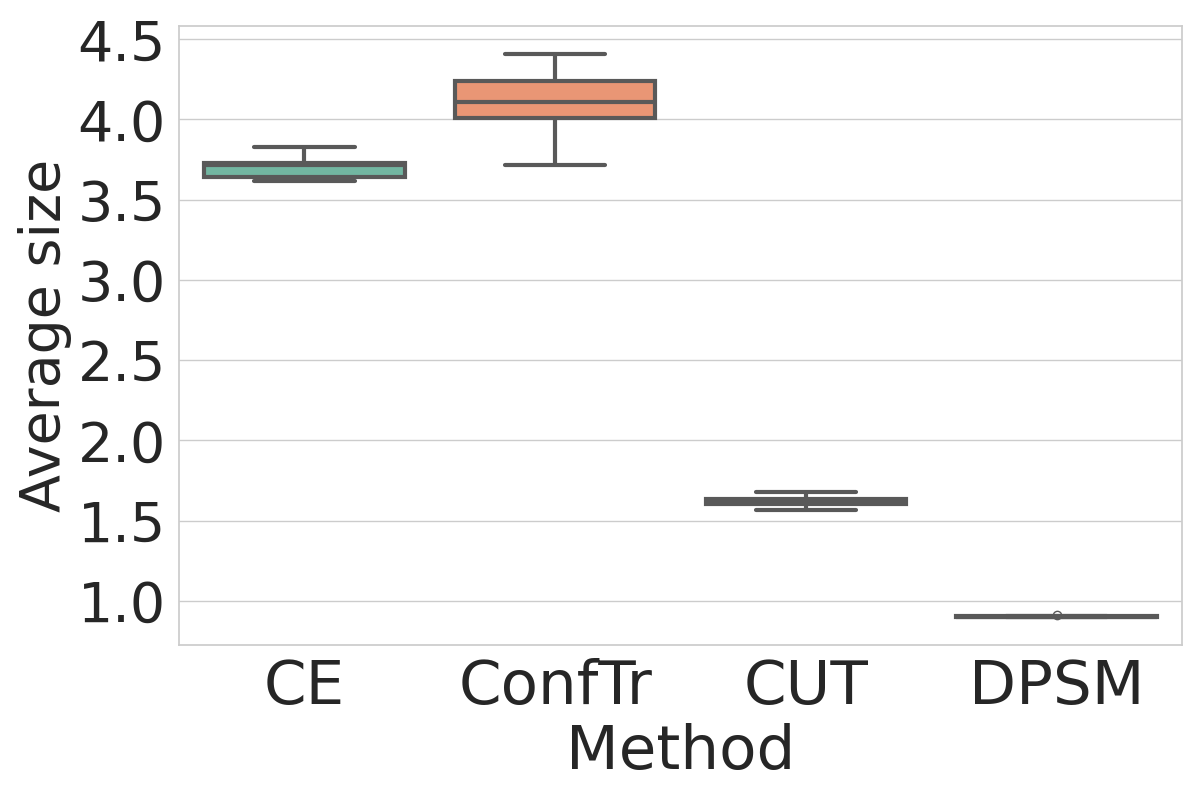}
    \end{minipage} 
    \begin{minipage}[t]{0.32\linewidth}
    \centering
    \textbf{(b)} CIFAR-100
    \includegraphics[width=\linewidth]{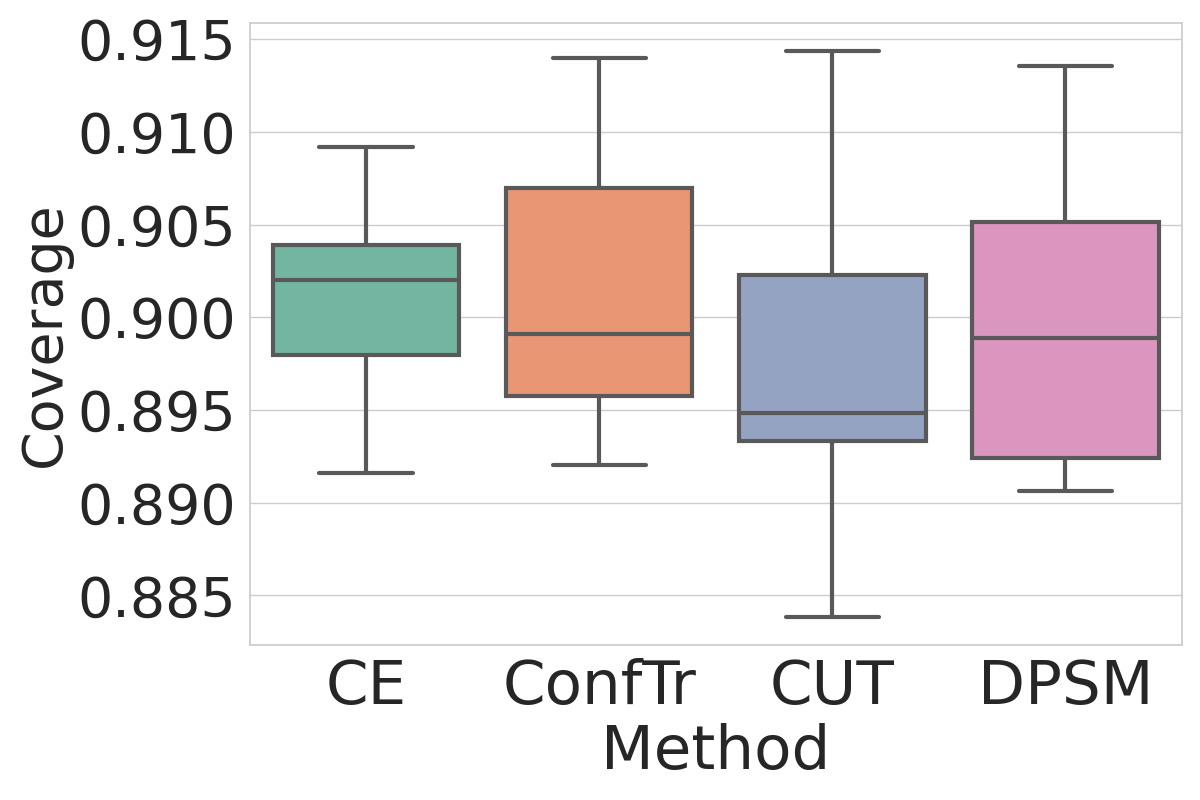}
    \\
    \includegraphics[width=\linewidth]{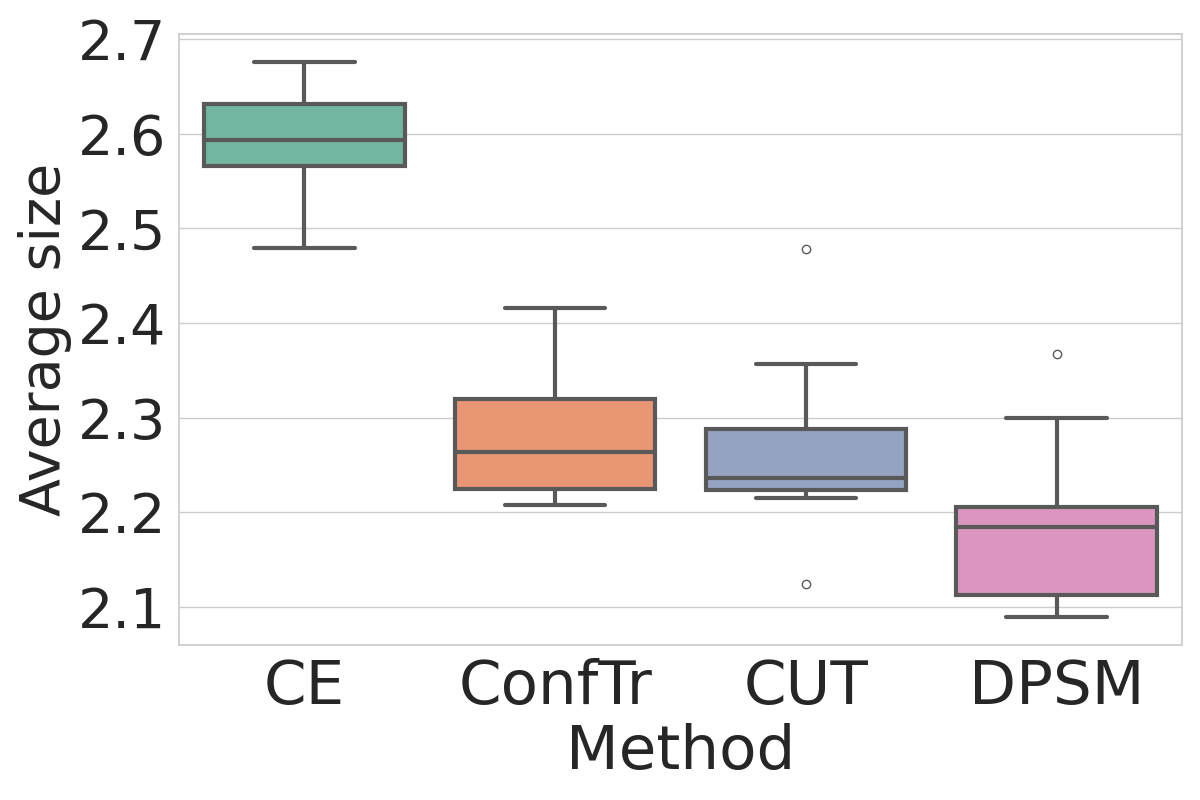}
    \end{minipage} 
    \begin{minipage}[t]{0.32\linewidth}
    \centering
    \textbf{(c)} iNaturalist
    \includegraphics[width=\linewidth]{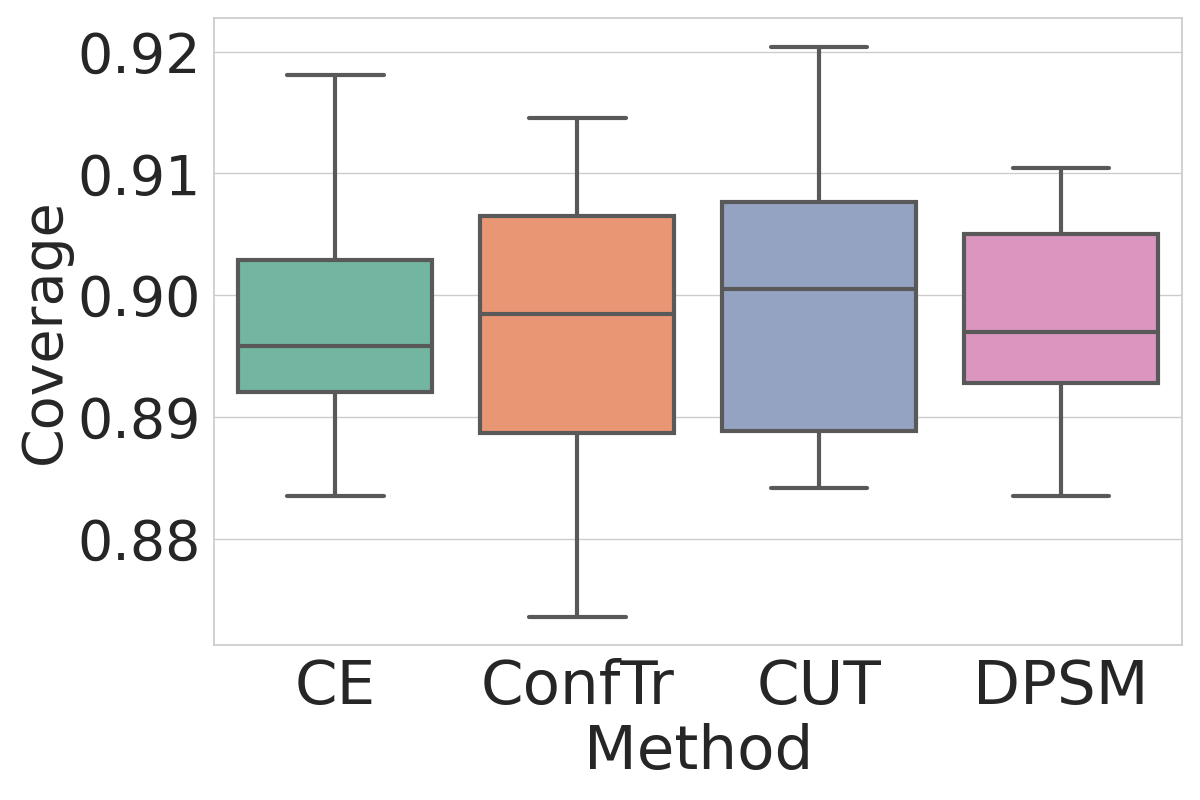}
    \\
    \includegraphics[width=\linewidth]{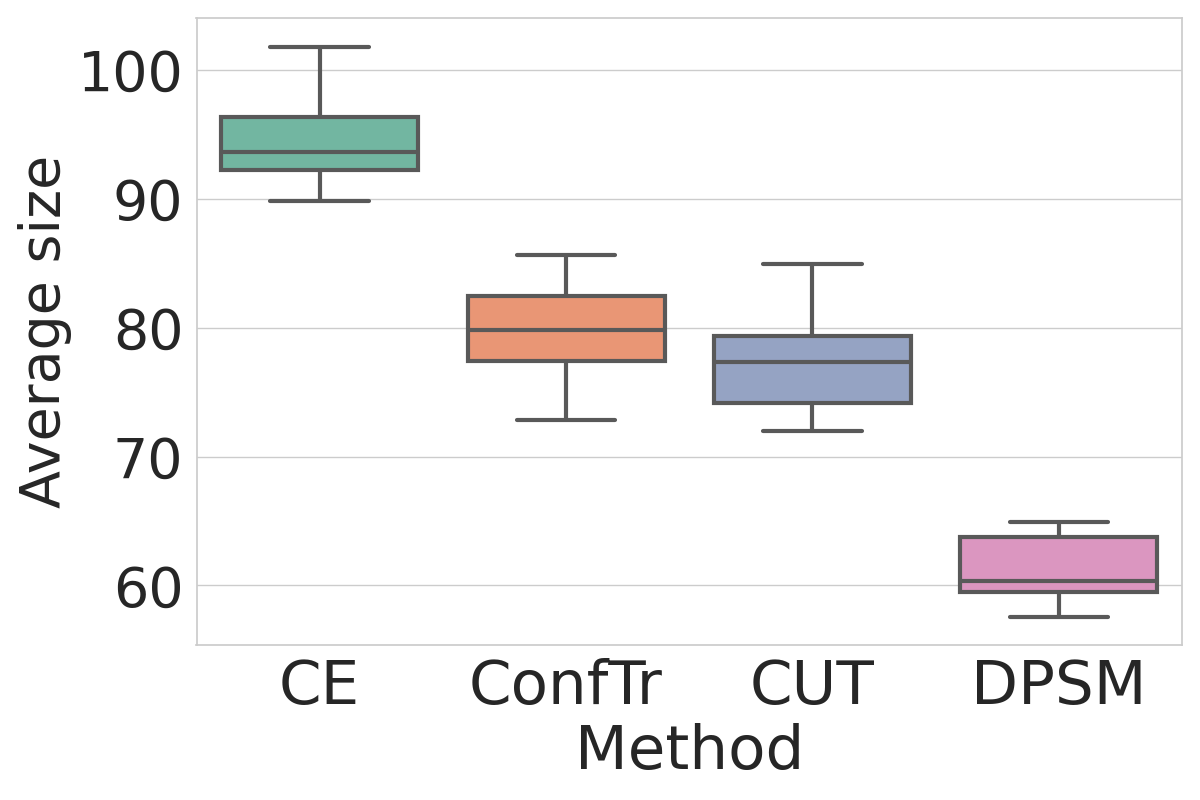}
    \end{minipage} 
    \vspace{-0.15in}
    \caption{
    \textbf{Box plots of coverage (Top row) and APSS (Bottom row)} of all methods using DenseNet and HPS score.
    DPSM achieves significantly smaller prediction set size while maintaining the valid coverage.
    }
    \label{fig:results_overall}
    \vspace{-2.0ex}
\end{figure*}

\noindent {\bf Datasets.}
We utilize the benchmark datasets CIFAR-100 \cite{krizhevsky2009learning}, Caltech-101 \cite{FeiFei2004LearningGV}, 
and iNaturalist \cite{van2018inaturalist}, where all details are summarized in Table \ref{tab:Data_stat} of Appendix \ref{subsec:train_test_strategies}.

\noindent {\bf Deep Models.}
We train two widely used neural network architectures with HPS scoring function: ResNet \cite{he2016deep} and DenseNet \cite{huang2017densely}. The training hyperparameters are provided in Table \ref{tab:finetune_hyper_params} of Appendix.

\noindent {\bf Conformity scoring functions.}
We consider three non-conformity scoring functions:
HPS \cite{vovk2005algorithmic, lei2013distribution} APS \cite{romano2020classification}, and RAPS \cite{angelopoulos2021uncertainty}. 
The detailed review of these scoring functions is in Appendix \ref{appendix:section:additional_details}.

\noindent {\bf Baseline methods.}
We select three different training methods as baselines:
(i) CE, training with standard cross-entropy loss only; 
(ii) CUT \cite{einbinder2022training}, mitigating the overconfidence of the deep neural network by penalizing the gap between the CDF of the non-conformity scores and the uniform distribution; 
and (iii) ConfTr \cite{stutz2021learning}, aiming to decrease prediction set size by inducing conformal loss as the average differentiable prediction set size. 
The details of CUT and ConfTr are in Appendix \ref{appendix:section:additional_details}.

\noindent {\bf Evaluation metrics.}
Our first evaluation metric is the marginal coverage (Marg-Cov), defined as $\text{MargCov} = \frac{1}{|\calD_\test|}\sum_{i \in \calD_\test} \indicator[Y_{i} \in C_f({X}_{i})]$.
Our second evaluation metric is the prediction set size (Avg-Set-Size), defined as $\text{Avg-Set-Size} = \frac{1}{|\calD_\test|}\sum_{i \in \calD_\test} |C_f({X}_{i})|$.
To make the conformal alignment loss smooth, DPSM and ConfTr use the Sigmoid function to approximate the average prediction set size (APSS) during training, 
so we use the third evaluation metric as the Soft Set Size (Avg-Soft-Size), defined as $\text{Avg-Softset-Size} = \frac{1}{|\calD_\test|}\sum_{i \in \calD_\test} \sum_{y \in \calY} \tilde \indicator \big [ S_f(X, y) \leq \widehat q_f \big ]$, recall that
$\tilde \indicator[ S \leq q ] = 1/(1+\exp(-(q-S)/\tau_\Sigmoid))$ with a hyperparameter $\tau_\Sigmoid$.

\subsection{Results and Discussion}

\begin{figure}[!ht]
    \centering
    \begin{minipage}[t]{0.48\linewidth}
    \centering
    \textbf{(a)} Upper loss
    \end{minipage} 
    \begin{minipage}[t]{0.48\linewidth}
    \centering
    \textbf{(b)} Lower loss
    \end{minipage} 
    \hfill
    \begin{minipage}[t]{0.48\linewidth}  
    \centering 
        \includegraphics[width=\linewidth]{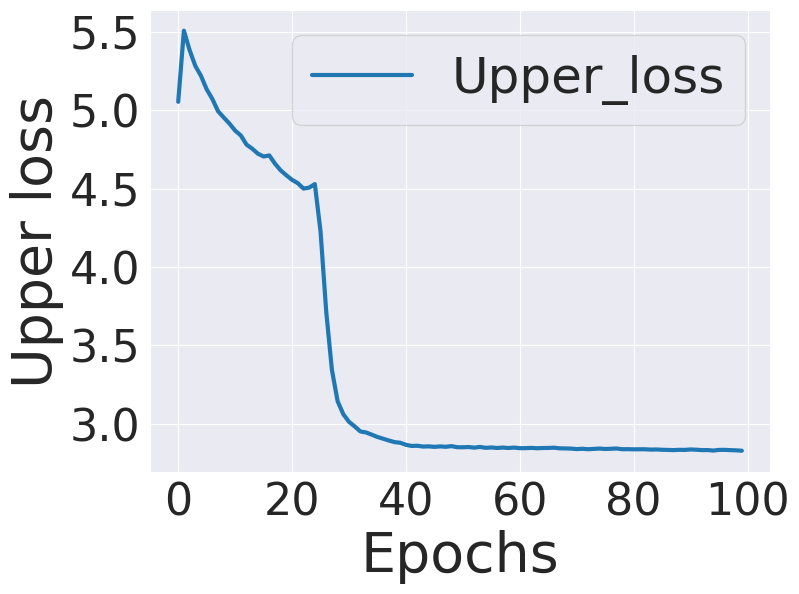}
    \end{minipage}
    \begin{minipage}[t]{0.48\linewidth}
        \centering 
        \includegraphics[width=\linewidth]{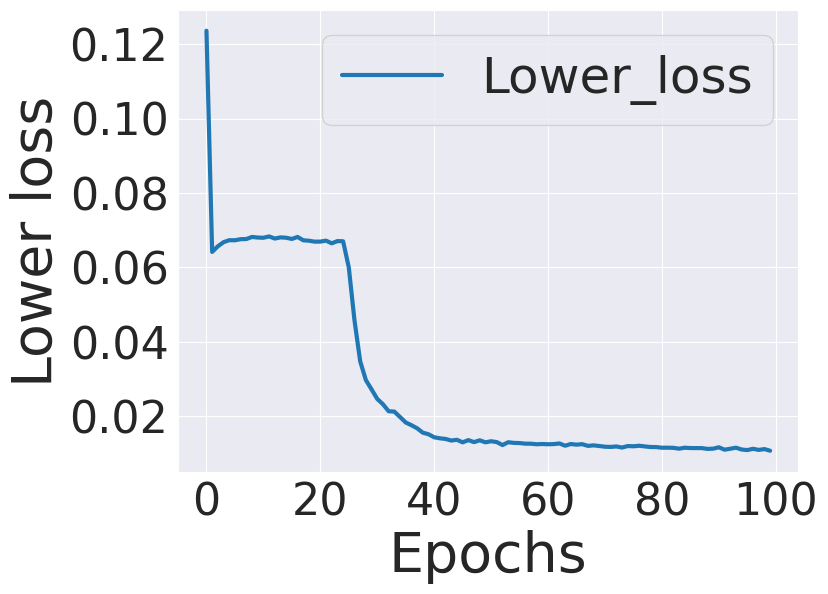}
    \end{minipage}
        \hfill
    \begin{minipage}[t]{0.48\linewidth}
    \centering
    \textbf{(c)} Conformal loss optimization gap
    \end{minipage} 
    \begin{minipage}[t]{0.48\linewidth}
    \centering
    \textbf{(d)} QR loss optimization gap
    \end{minipage} 
    \hfill
    \begin{minipage}[t]{0.48\linewidth}
     \centering   
     \includegraphics[width = \linewidth]{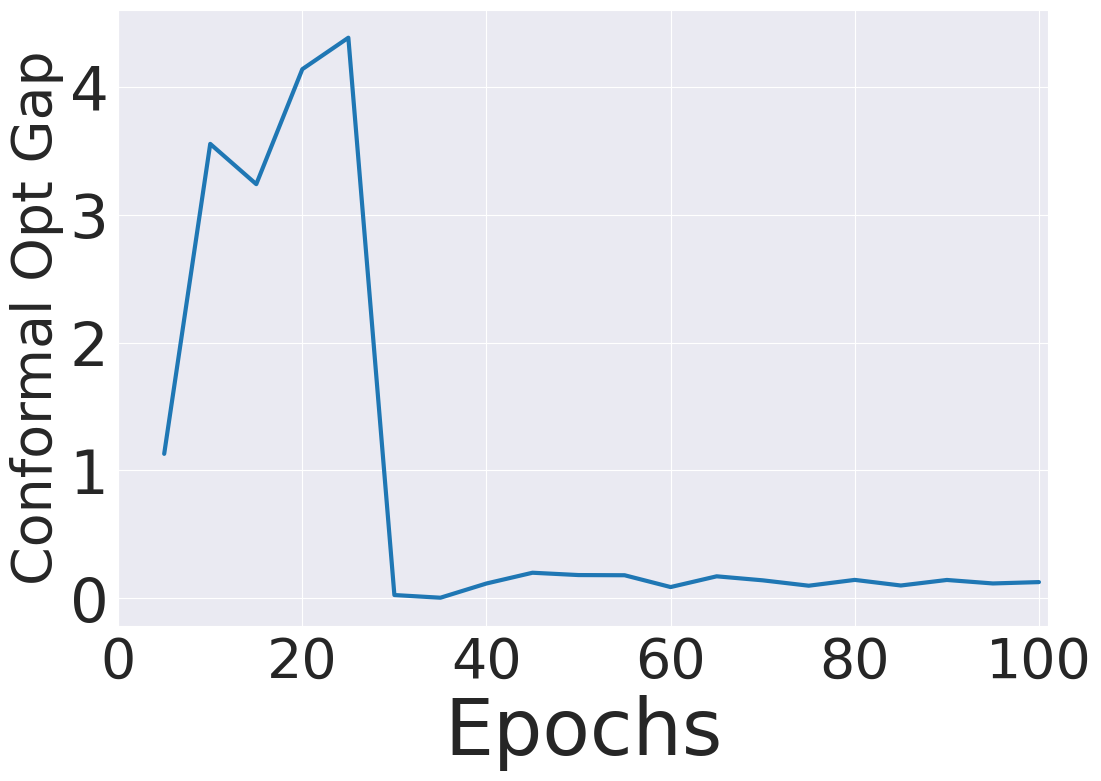}
     \end{minipage}
    \begin{minipage}[t]{0.48\linewidth}
    \centering
    \includegraphics[width=\linewidth]{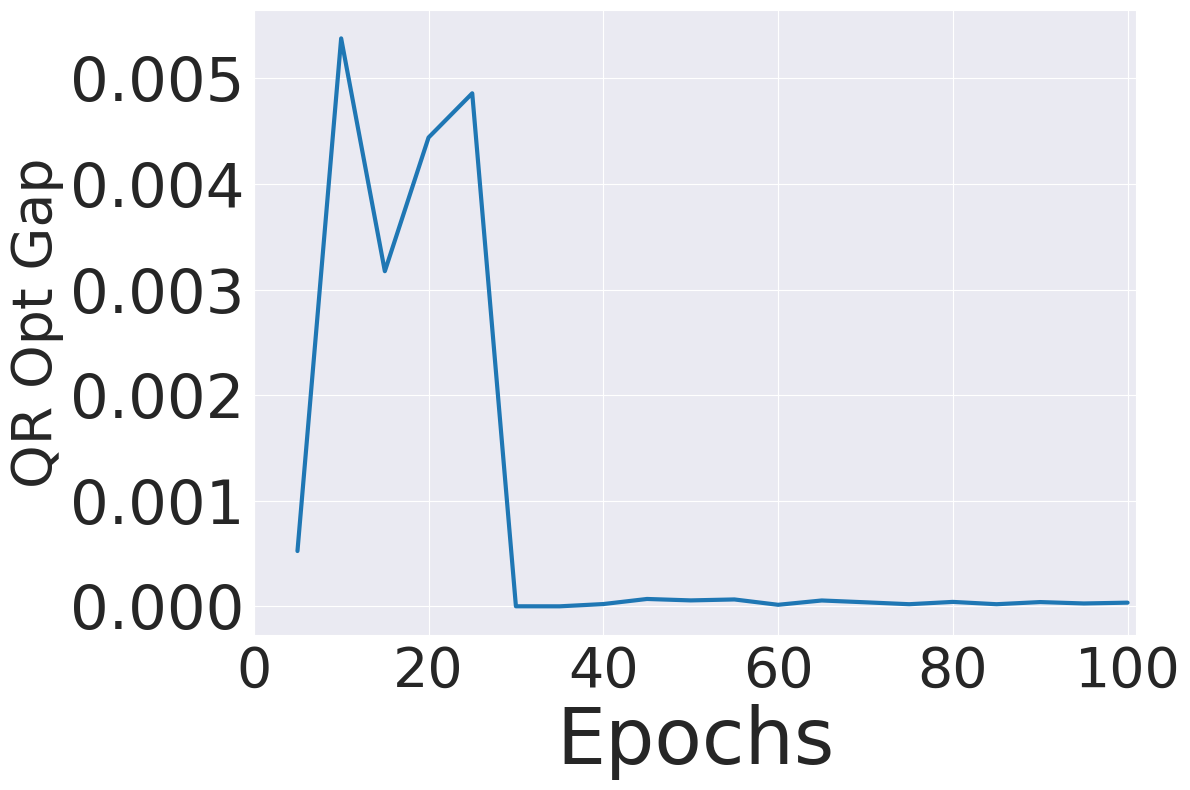}
    \end{minipage}
    \vspace{-0.15in}
    \caption{
    \textbf{Justification experiments for the convergence of DPSM }on CIFAR-100 using DenseNet and HPS score. 
   \textbf{(a)} Upper level loss (i.e., a combination of classification loss and conformal alignment loss);
    \textbf{(b)} Lower level loss (i.e., QR loss);
    \textbf{(c)} Optimization gap of conformal loss, defined as the difference between conformal losses using learned batch-level quantiles and dataset-level quantiles on the training set;
    \textbf{(d)} Optimization gap of the lower-level QR loss, defined similarly as the loss difference between learned batch-level quantiles and dataset-level quantiles.
    }
    \label{fig:results_convergence}
    \vspace{-4.0ex}
\end{figure}

\begin{figure*}[!t]
    \centering
    \begin{minipage}[t]{0.33\linewidth}
    \centering
    \textbf{(a)} Estimation error of quantiles
    \includegraphics[width = \linewidth]{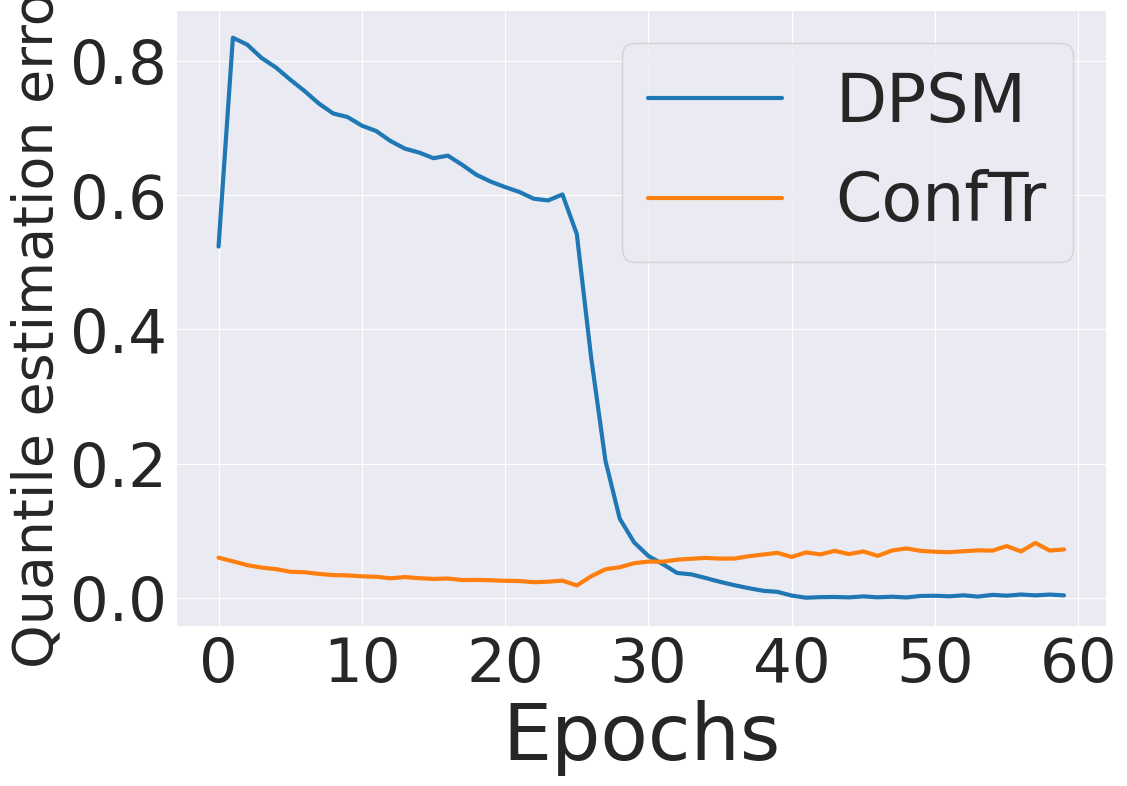}
    \end{minipage} 
    \begin{minipage}[t]{0.33\linewidth}
    \centering
    \textbf{(b)} Average soft set size
    \includegraphics[width=\linewidth]{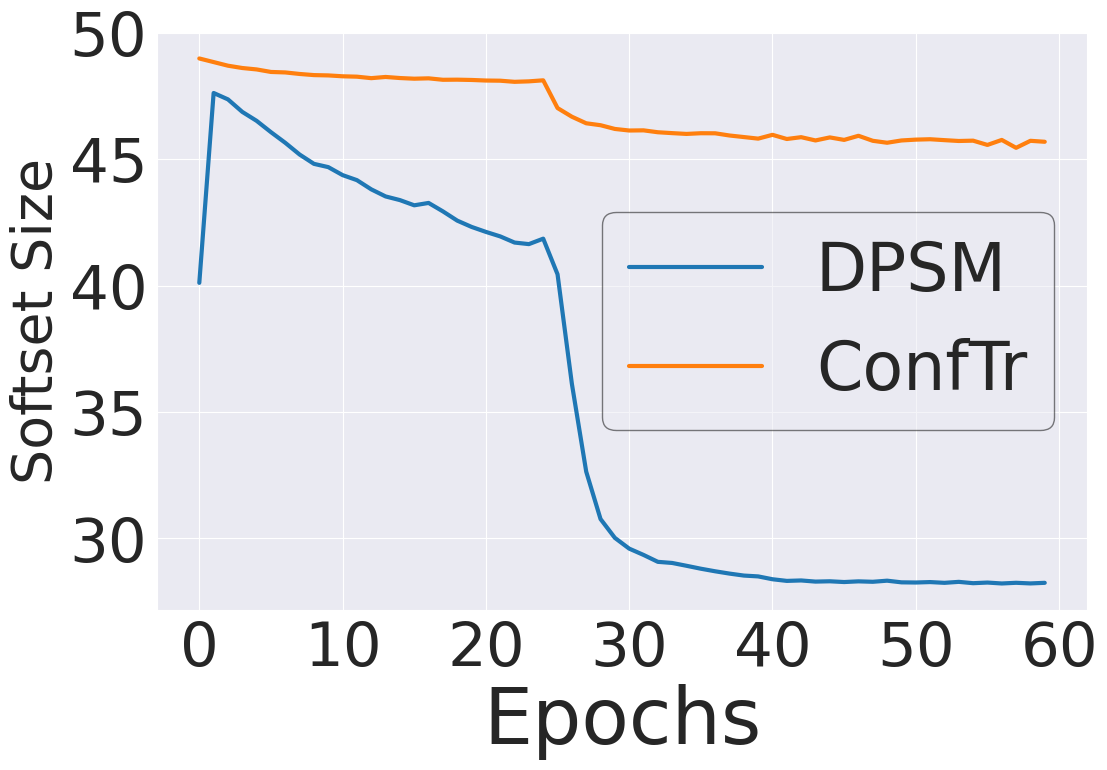}
    \end{minipage} 
    \begin{minipage}[t]{0.33\linewidth}
    \centering
    \textbf{(c)} Learning bound approximation
    \includegraphics[width=\linewidth]{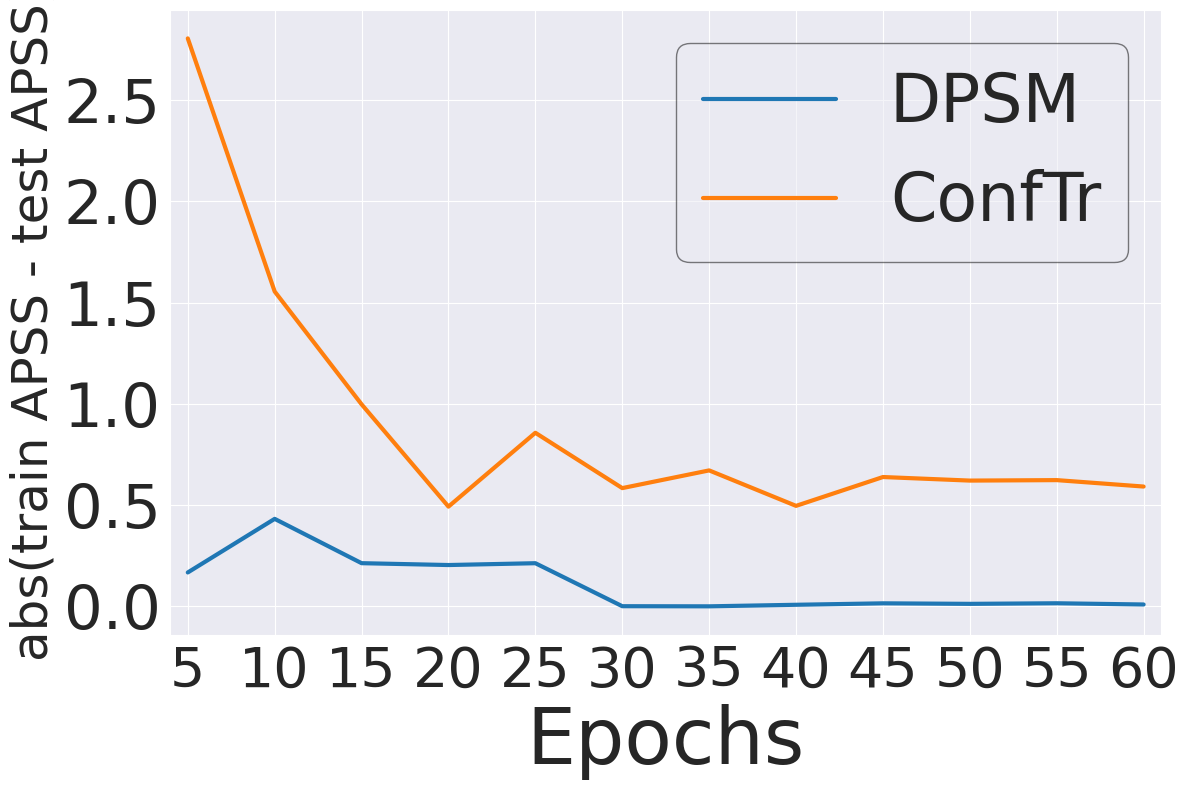}
    \end{minipage} 
    \vspace{-0.1in}
    \caption{
    \textbf{Justification experiments for the learning bound of DPSM }on CIFAR-100 using DenseNet and HPS score. 
    \textbf{(a)} Estimation error between the $\widehat Q^n_f$ (i.e., the dataset-level quantiles on training data) and $\widehat q_f$ (i.e., either the batch-level quantiles evaluated in ConfTr or the learned ones in DPSM);
    \textbf{(b)} Average soft set size (approximated using Sigmoid function) of DPSM and ConfTr;
    \textbf{(c)} Approximated learning error comparison between DPSM and ConfTr, measured by their gaps between the training and testing APSS.
    }
    \label{fig:results_bound}
    \vspace{-2.0ex}
\end{figure*}


We discuss the results comparing DPSM with baseline methods, fine-grained analysis for DPSM to demonstrate its effectiveness and validate our theoretical results.

\noindent \textbf{DPSM generates smaller prediction sets.}
Table \ref{tab:cvg_set_hps_train_main} presents the prediction set sizes for different methods using HPS score for training, HPS and APS scores for calibration and testing phases at $\alpha = 0.1$, as detailed in Appendix \ref{subsec:train_test_strategies}. On average, DPSM reduces the prediction set size by $20.46\%$ across the three datasets relative to the best baseline.
Specifically, DPSM provides the best predictive efficiency compared to all other existing methods when calibrated with HPS. 
It also outperforms the best baselines when calibrated with APS in all cases except $\uparrow$ 8.71\% increase on CIFAR-100 with DenseNet and $\uparrow$ 0.54\% increase on iNaturalist with ResNet in terms of the prediction set size.The coverage results for overall comparison and additional results for calibrating with RAPS scoring function are in Appendix \ref{appendix:subsec:additional_exps_marginal}.
We also visualize the coverage rate and APSS with confidence intervals of all methods using DenseNet and HPS score in Figure \ref{fig:results_overall}.
It clearly confirms that DPSM achieves significantly smaller prediction set size while maintaining the valid coverage.

\noindent \textbf{DPSM converges to stable error for bilevel optimization.} To further explore how DPSM effectively generates smaller prediction sets, we analyze the convergence of DPSM by plotting the loss of the upper level function (i.e., a combination of classification loss and conformal alignment loss) and the lower  level function (i.e., QR loss) over training epochs.
Figure \ref{fig:results_convergence} (a) and (b) show the upper-level loss and lower-level loss over 100 epochs, respectively. 
We also report the results of $40$-epoch training regime in Appendix \ref{appendix:subsec:additional_exps_marginal} for reference.
As shown, the upper-level loss of DPSM exhibits an initial increase during the first $5$ epochs, reaching the peak, then steadily decreases before stabilizing around epoch $35$.
In contrast, the lower-level loss decreases sharply within the first $5$ epochs, followed by a gradual reduction until convergence near the end of training. 
These results empirically demonstrate that DPSM effectively converges in terms of both upper and lower level training errors, validating its bilevel optimization approach.
To investigate how the learned quantiles influence the optimization error, we compute both conformal and QR losses using the learned quantiles and the optimal (dataset-level) quantiles. 
The corresponding optimization errors—defined as the loss differences between learned and optimal quantiles—are visualized in Figure \ref{fig:results_convergence} (c) and (d). 
Both errors converge to nearly $0$, indicating that the learned quantiles effectively approximate the optimal quantiles over training epochs.

\noindent \textbf{DPSM estimates empirical quantiles with small error.} 
To compare the precision of empirical quantiles estimation in ConfTr and DPSM,
we plot the estimation error between $\widehat Q^n_f$ (quantiles evaluated on the whole training dataset) and  $\widehat q_f$ (quantiles evaluated in ConfTr or learned in DPSM on mini-batches). 
Figure \ref{fig:results_bound} (a) plots this estimation error over training epochs. 
For the first $25$ epochs, the estimation errors for DPSM are significantly larger compared to ConfTr. 
However, as the training progresses, the estimation errors for DPSM decrease rapidly, converging close to $0$ after epoch $35$. 
This result verifies the theoretical result for smaller estimation error in learning bound analysis from Theorem \ref{theorem:learning_bound_DPSM}.
The rapid reduction in estimation error also reflects the convergence of the lower loss (i.e., QR loss), highlighting the effectiveness of DPSM in accurately estimating quantiles.

\noindent \textbf{Learning bound of DPSM is much tighter than ConfTr.}
To approximately compare the learning bounds of DPSM and ConfTr,
we compare the conformal alignment losses of ConfTr and DPSM during training in terms of the average soft set size, as shown in Figure \ref{fig:results_bound} (b).
The soft set size of DPSM is consistently smaller than that of ConfTr during training. 
Combining the empirical results of the smaller estimation error of quantiles from Figure \ref{fig:results_bound} (a), we can conclude that the learning bound of DPSM is much tighter than the learning bound of ConfTr, providing the empirical verification of Theorem \ref{theorem:learning_bound_DPSM} and \ref{theorem:learning_bounds_SA}.
Although the learning bound cannot be empirically computed, we approximate it using a common strategy in ML literature \cite{yuan2019stagewise,yang2021exact}, which estimates the generalization error by the absolute gap between training and test errors. 
For CP, we use APSS evaluated on the train and test sets to approximate the learning errors. 
Specifically, for DPSM, at each iteration, we:
(i) compute APSS on the training set using the learned quantiles as thresholds. 
It includes optimization error since the learned quantiles are not optimal (true dataset-level quantiles);
(ii) compute the APSS on the testing set using the dataset-level quantiles as thresholds.
The gap between these two APSS values is employed as an approximation of the learning bound.
We apply the same strategy to the SA-based ConfTr, 
where the training APSS is computed using the quantiles evaluated on mini-batches from the training data, 
and the test APSS is computed using the dataset-level quantiles from the test data.
This comparison is shown in Figure~\ref{fig:results_bound} (c), which demonstrates that the approximated learning error is improved by DPSM.

\noindent \textbf{Assumption \ref{assumption:bi_lipschitz} is empirically valid.} 
Figure \ref{fig:sodt_set_size_comparisons_main} (a) illustrates the conformity scores against their corresponding normalized order. 
The x-axis represents the normalized order, while the y-axis represents conformity scores. 
It is clear that the curve does not remain near the x-axis or y-axis, indicating that the gradient of conformity scores with respect to normalized index is both upper and lower bounded. 
This observation empirically validates Assumption \ref{assumption:bi_lipschitz}.

\noindent \textbf{Assumption \ref{assumption:straongly_concave} is empirically valid.}
Figure \ref{fig:sodt_set_size_comparisons_main} (b) visualizes the soft set size of ConfTr, with input as coverage rate $\in [0.02, 0.98]$ with range $0.02$.
When coverage rate approaches $0.9$, the curves of all methods exhibit a concave shape (zoom-in version also shown), providing empirical verification for Assumption \ref{assumption:straongly_concave}.

\begin{figure}[!ht]
    \centering
    \begin{minipage}[t]{0.48\linewidth}
    \centering
    \textbf{(a)} Bi-Lipschitz continuity of conformity score 
    \end{minipage} 
    \begin{minipage}[t]{0.48\linewidth}
    \centering
    \textbf{(b)} Strongly concavity of conformal loss
    \end{minipage} 
    \hfill
    \begin{minipage}[t]{0.48\linewidth}  
    \centering
        \includegraphics[width=\linewidth]{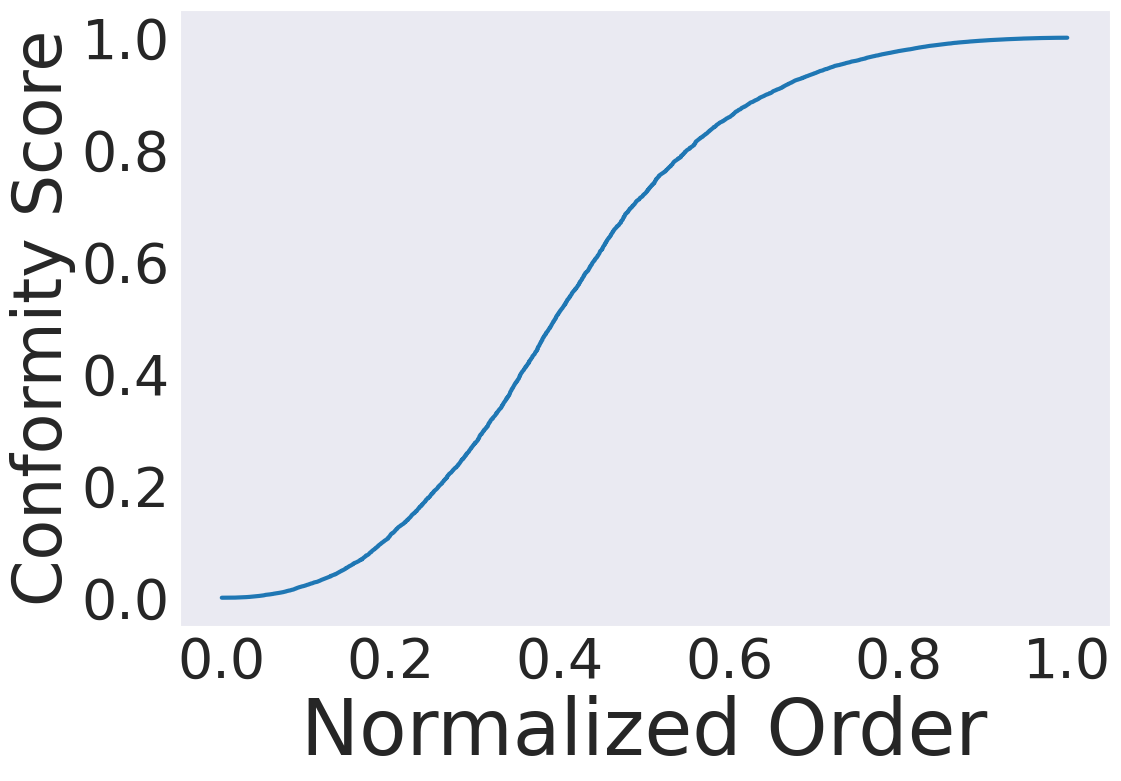}
    \end{minipage}
    \begin{minipage}[t]{0.49\linewidth}  
    \centering
    \includegraphics[width=\linewidth]{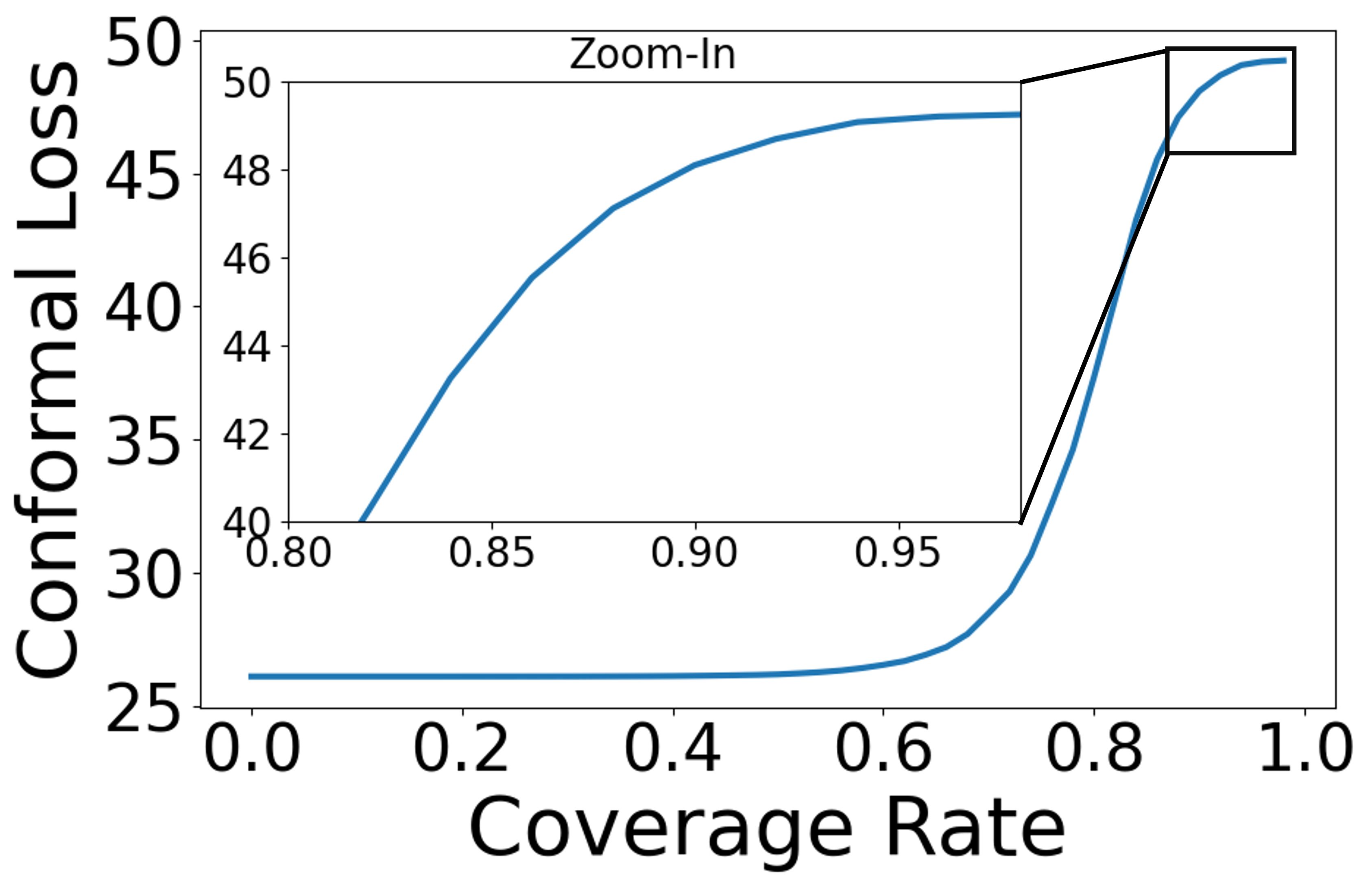}
    \end{minipage}
    \caption{\textbf{Assumption verification} on CIFAR-100 using DenseNet and HPS score on the calibration dataset.
    \textbf{(a)} HPS scores over corresponding normalized index produced by ConfTr. 
    The x-axis is the normalized order, the y-axis is the corresponding conformity score; 
    \textbf{(b)} The soft set size measure of ConfTr. The input coverage rate is from $[0.02, 0.98]$ with $0.02$ range with its zoom-in version where coverage rate is close to the target coverage $0.9$;
    When coverage rate is close to $0.9$, the curve for the soft set size exhibits a concave shape.
    }
\label{fig:sodt_set_size_comparisons_main}
\end{figure}

\section{Conclusion}

This paper developed Direct Prediction Set Minimization (DPSM), a novel conformal training algorithm formulated as a bilevel optimization problem. 
By leveraging quantile regression in the lower-level subproblem, DPSM precisely learns the empirical quantile of non-conformity scores, leading to  smaller prediction sets.
Theoretical analysis showed that DPSM attains a tight learning bound of $O(1/\sqrt{n})$ with $n$ training samples. 
Empirical evaluation on real-world datasets confirmed that DPSM significantly outperforms existing conformal training methods, validating both its theoretical advantages and practical effectiveness. 

\section*{Acknowledgements}

The authors gratefully acknowledge
the in part support by USDA-NIFA funded AgAID Institute
award 2021-67021-35344, and the NSF grants CNS-2312125, IIS-2443828. The views expressed are those of
the authors and do not reflect the official policy or position
of the USDA-NIFA and NSF.

\section*{Impact Statement}
This paper focuses on integrating conformal principles into the training procedure of deep neural network classifiers to produce small prediction sets. The insights in this paper could improve the safe deployment of deep classifiers in may high-risk applications. There is not any negative consequences to be specifically highlighted here.

\nocite{langley00}

\bibliography{ref}
\bibliographystyle{icml2025}

\newpage
\appendix
\onecolumn
\section{Additional background details}
\label{appendix:section:additional_details}

\noindent \textbf{Non-conformity scoring functions.} The homogeneous prediction sets (HPS) \cite{sadinle2019least} scoring function is
defined as follows:
\begin{align}
\label{eq:HPS}
S_f^{\HPS}(X, Y) = 1 - f_{\theta}(X)_Y.
\end{align}

\cite{romano2020classification} has proposed another conformity scoring function, Adaptive Prediction Sets (APS). 
APS scoring function is based on sorted probabilities.
For a given input $X$, we sort the softmax probabilities for all classes \{$1,\cdots,K$\} such that $1 \ge f_{\theta}(x)_{( 1 )} \geq \cdots \ge f_{\theta}(x)_{( K )}\ge 0$, and compute the cumulative confidence as follows:

\begin{align}
\label{eq:APS}
S_f^{\APS}(X, Y) = \sum^{r_f(X, Y)-1}_{l=1} f_{\theta}(X)_{( l )} + U \cdot f_{\theta}(X)_{( r_f(X, Y) )},
\end{align}
where $U \in [0,1]$ is a uniform random variable to break ties.

To reduce the probability of including unnecessary labels (i.e., labels with high ranks) and thus improve the predictive efficiency, \cite{angelopoulos2021uncertainty} has proposed Regularized Adaptive Prediction Sets (RAPS) scoring function. 
The RAPS score is computed as follows:
\begin{align}
\label{eq:RAPS}
S_f^{\RAPS}(X, Y) = 
\sum^{r_f(X, Y)-1}_{l=1} f_{\theta}(X)_{( l )} + U \cdot f_{\theta}(X)_{( r_f(X, Y) )} + \lambda_{\RAPS} (r_f(X, Y) - k_{\reg})^+,
\end{align}
where $\lambda_{\RAPS}$ and $ k_{\reg}$ are two hyper-parameters.

\noindent \textbf{Objective function for conformal training methods.} ConfTr \cite{stutz2021learning} estimates a soft measure of the prediction set size, defined as follows:
\begin{align}
\label{eq:ps_size_approximate_conftr}
\widehat \calL_{c, \ConfTr}^\SA(f) 
= &
\E_{\widehat q_f \sim \calQ_f} \Bigg [ \frac{1}{n} \sum_{i=1}^n \sum_{y \in \calY } \tilde \indicator \big [ S_f(X_i, y) \leq \widehat q_f \big ] \Bigg ]
,
\end{align}
where $\tilde \indicator[\cdot]$ is a smoothed estimator for the indicator function $\indicator[\cdot]$ and defined by a Sigmoid function \cite{stutz2021learning}, 
i.e.,
$\tilde \indicator[ S \leq q ] = 1/(1+\exp(-(q-S)/\tau_\Sigmoid))$ with a tunable temperature parameter $\tau_\Sigmoid$.

CUT \cite{einbinder2022training} measures the maximum deviation from the uniformity of conformity scores, defined as follows:
\begin{align}
\label{eq:ps_size_approximate_cut}
\widehat \calL_{c, \UA}^\SA(f) 
= &
\E_{\widehat q_f \sim \calQ_f} \Big [ \sup_{\alpha \in [0,1]} \big |  (1-\alpha) - \widehat q_f(\alpha) \big | \Big ]
,
\end{align}
where $\widehat q_f(\alpha)$ is the empirical batch-level quantile in $\calB$ of input $\alpha \in [0,1]$.

\section{ Proof for Section \ref{section:notation_and_bg} }
\label{section:appendix:proof_bg_section}

In this section, we prove Proposition \ref{proposition:quantile_distribution} and Theorem \ref{theorem:learning_bounds_SA} from Section \ref{section:notation_and_bg}.

\subsection{ Proof for Proposition \ref{proposition:quantile_distribution} }
\label{section:proof:proposition:quantile_distribution}

\begin{proposition}
\label{proposition:appendix:quantile_distribution}
(Proposition \ref{proposition:quantile_distribution} restated, distribution of mini-batch quantiles)
Define an event $Z(j)$ as 
``the $j$-th smallest score $S_{(j)}$ from $\{S_i\}_{i=1}^n$ is selected as $\widehat q_f(\alpha)$ on a mini-batch $\calB$''. 
Then the probability of $Z(j)$ is:
\begin{align*}
\P(Z(j)) = \frac{\tbinom{j}{\lceil(1-\alpha)(s+1)\rceil } \tbinom{n-j-1}{s-\lceil(1-\alpha)(s+1)\rceil-1}} {\tbinom{n}{s}}.
\end{align*}
Furthermore, we have the following asymptotic result
\begin{align*}
\lim_{n \to \infty} \P(Z(j)) = \frac{e}{n}\P_\Beta\Big(
\frac{j}{n}; 
\lceil(1-\alpha)(s+1)\rceil+1; 
s- \lceil(1-\alpha)(s+1)\rceil \Big)
,
\end{align*}
where $\P_\Beta(x;a;b)$ is the PDF of Beta distribution with two shape parameters $a, b$.
\end{proposition}

\begin{proof}
(Proof of Proposition \ref{proposition:quantile_distribution})
This proof follows the proof for Proposition 1 in \cite{kawaguchi2020ordered}.

Recall the Stirling's approximation that is used to prove this proposition:
\begin{align*}
n! \backsim \sqrt{2 \pi n} (n/e)^n
.
\end{align*}

To simplify the representation, we denote $a = \lceil(1-\alpha)(s+1)\rceil$ in this proof. 
Then, the $\P(Z(j))$ could be rewritten as:
\begin{align}
\label{eq:proof_proposiionn_1}
\P(Z(j))
&
= \frac{\tbinom{j}{ a} \tbinom{n-j-1}{s-a-1}} {\tbinom{n}{s}}
=\frac{ \frac{j!}{a!(j-a)!} \cdot \frac{(n-j-1)!}{(s-a-1)!(n-j-s+a)!} } {\frac{n!}{s!(n-s)!}}
= \frac{s!}{a! (s-a-1)!} \cdot \frac{j!}{(j-a)!} \cdot \frac{(n-j-1)!}{(n-j-s+a)!} \cdot \frac{(n-s)!}{n!}
.
\end{align}

Applying the Stirling's approximation when $n \rightarrow \infty$, we have that:
\begin{align}
\label{eq:proof_proposiionn_2}
&
\frac{j!}{(j-a)!} \cdot \frac{(n-j-1)!}{(n-j-s+a)!} \cdot \frac{(n-s)!}{n!}
\nonumber \\
&
= \frac{\sqrt{2 \pi j} (\frac{j}{e})^j}{\sqrt{2 \pi (j-a)} (\frac{j-a}{e})^{j-a} } \cdot \frac{\sqrt{2 \pi (n-j-1)} (\frac{n-j-1}{e})^{n-j-1} }{\sqrt{2 \pi (n-j-s+a)} (\frac{n-j-s+a}{e})^{n-j-s+a} } \cdot \frac{\sqrt{2 \pi (n-s)} (\frac{n-s}{e})^{n-s}}{\sqrt{2 \pi n} (\frac{n}{e})^{n}}
\nonumber \\
&
=
\frac{j^j}{ (j-a)^{j-a} } e^{-a} \cdot \frac{ (n-j-1)^{n-j-1} }{ (n-j-s+a)^{n-j-s+a} } e^{-s+a+1} \cdot \frac{ (n-s)^{n-s} }{ (n)^{n} } e^{s}
\nonumber \\
&
=
e \frac{j^j}{ (j-a)^{j-a} } \cdot \frac{ (n-j)^{n-j-1} }{ (n-j)^{n-j-s+a} } \cdot \frac{ (n)^{n-s} }{ (n)^{n} }
\nonumber \\
&
=
e \frac{(n)^j (\frac{j}{n})^j}{ (n)^{j-a} (\frac{j-a}{n})^{j-a} } \cdot (n-j)^{s-a-1} \cdot (n)^{-s} 
\nonumber \\
&
=
e \frac{(\frac{j}{n})^j}{ (\frac{j-a}{n})^{j-a} } (n)^a \cdot (\frac{n-j}{n})^{s-a-1} (n)^{s-a-1} \cdot (n)^{-s} 
\nonumber \\
&
=
\frac{e}{n} \cdot (\frac{j}{n})^a \cdot (1 - \frac{j}{n})^{s-a-1} 
\end{align}

Combining the Equation (\ref{eq:proof_proposiionn_1}) and  (\ref{eq:proof_proposiionn_2}), we have that:
\begin{align*}
&
\lim_{n \to \infty} \P(Z(j))
=
\frac{e}{n} \cdot \frac{s!}{a! (s-a-1)!} \cdot (\frac{j}{n})^a \cdot (1 - \frac{j}{n})^{s-a-1} 
=
\frac{e}{n} \frac{ \Gamma(a + 1 + s - a ) }{ \Gamma(a+1) \Gamma(s-a) }  \cdot (\frac{j}{n})^a \cdot (1 - \frac{j}{n})^{s-a-1} 
\\
= & 
\frac{e}{n} \P_\Beta\Big(\frac{j}{n}; \lceil(1-\alpha)(s+1)\rceil+1; 
s- \lceil(1-\alpha)(s+1)\rceil \Big)
,
\end{align*}
where $\Gamma(n) = (n-1)!$ is gamma function and $\P_\Beta(x;a,b) = \frac{\Gamma(a+b)}{\Gamma(a) \Gamma(b)} x^{a-1} (1-x)^{b-1}$ is the probability density function of Beta distribution with shape parameters $a$ and $b$ at $x$.
\end{proof}

\subsection{ Proof for Theorem \ref{theorem:learning_bounds_SA} }
\label{section:proof_theorem_learning_bounds_SA}

\begin{theorem}
\label{theorem:learning_bounds_SA_appendix}
(Theorem \ref{theorem:learning_bounds_SA} restated, learning bounds of SA method)
Suppose that Assumption \ref{assumption:bi_lipschitz} and \ref{assumption:straongly_concave} hold.
Assume $(1-\alpha))(s+1)$ is not an integer.
If $\lceil (1-\alpha))(s+1) \rceil - (1-\alpha))(s+1) \geq \Omega(1/s)$ and $s \leq \sqrt{n}$,
then the following inequality holds with probability at least $1 - \delta$:
\begin{align*}
\Omega (1/s) 
\leq 
| \widehat \calL_c^\SA(f) - \calL_c(f)  | \leq 
\tilde O (1/\sqrt{s})
.
\end{align*}
\end{theorem}

\begin{proof}
(Proof of Theorem \ref{theorem:learning_bounds_SA})

We need the following helpful technical lemmas:

\begin{lemma}
\label{lemma:distance_empirical_sa}
(Lower bound for $\widehat \ell(f, \E_{\widehat q_f} [\widehat q_f] ) - \E_{\widehat q_f} [\widehat \ell(f; \widehat q_f) ]$)
Suppose that Assumptions \ref{assumption:bi_lipschitz} and \ref{assumption:straongly_concave} hold.
Then, the following inequality holds:
\begin{align*}
| \widehat \ell(f, \E_{\widehat q_f} [\widehat q_f ] ) - \E_{\widehat q_f}[\widehat \ell(f, \widehat q_f)] | 
\geq
\Omega(1/s).
\end{align*}
\end{lemma}

\begin{lemma}
\label{lemma:distance_empirical_expectation}
(Lower bound for $\widehat \ell(f, \E[\widehat q_f]) - \widehat \ell(f, Q_f) $)
Suppose that Assumption \ref{assumption:bi_lipschitz} holds.
Assume $(1-\alpha))(s+1)$ is not an integer.
If $\lceil (1-\alpha))(s+1) \rceil - (1-\alpha))(s+1) \geq \Omega(1/s)$ and $s \leq \sqrt{n}$,
hen the following inequality holds:
\begin{align*}
\widehat \ell(f, \E[\widehat q_f]) - \widehat \ell(f, Q_f) 
\geq 
\Omega(1/s).
\end{align*}
\end{lemma}

\begin{lemma}
\label{lemma:upper_bound_variance}
(Distance between $\calL_c(f, \widehat q_f)$ and $\E_{\widehat q_f}[ \calL_c(f; \widehat q_f)]$)
Suppose that Assumption \ref{assumption:bi_lipschitz} holds,
With probability at least $1-\delta$, the following inequality holds:
\begin{align*}
| \calL_c(f, \widehat q_f) - \E_{\widehat q_f}[ \calL_c(f, \widehat q_f)] | 
\leq 
\tilde O(1/\sqrt{s}).
\end{align*}
\end{lemma}

The proofs of Lemma \ref{lemma:distance_empirical_sa}, \ref{lemma:distance_empirical_expectation}, and \ref{lemma:upper_bound_variance} are deferred to the end of this proof.

Recall that 
\begin{align*}
\widehat \calL_c^\SA(f) 
= &
\E_{\widehat q_f \sim \calQ_f} [ \widehat \ell(f, \widehat q_f)] 
= 
\E_{\widehat q_f \sim \calQ_f} \Bigg[ 
\underbrace{
\frac{1}{n}\sum_{i=1}^n \sum_{y\in \calY} \tilde \indicator[ S_f(X_i, y) \leq \widehat q_f] 
}_{ = \widehat \ell(f,\widehat q_f)}
\Bigg],
\\
\calL_c(f)
= &
\ell(f, Q_f)
=
\E_{X} \Bigg[ \sum_{y\in \calY} \tilde \indicator [S_f(X,y) \leq Q_f] \Bigg]
.
\end{align*}

We start with the lower bound in Theorem \ref{theorem:learning_bounds_SA}.
\begin{align*}
( \widehat \calL_c^\SA(f) - \calL_c(f) )^2
= &
( \E_{\widehat q \sim \calQ_f} [ \widehat \ell(f, \widehat q_f) ] - \ell(f, Q_f) )^2
= 
( \E_{\widehat q \sim \calQ_f} [ \widehat \ell(f, \widehat q_f) ]  - \widehat \ell(f, \E[\widehat q_f])
+ \widehat \ell(f, \E[\widehat q_f]) - \ell(f, Q_f) )^2
\\
= & 
\underbrace{ ( \E_{\widehat q \sim \calQ_f} [ \widehat \ell(f, \widehat q_f) ]  - \widehat \ell(f, \E[\widehat q_f]) )^2
}_{ \geq \Omega(1/s^2), \text{ Lemma \ref{lemma:distance_empirical_sa}} }
+ 
\underbrace{ 
( \widehat \ell(f, \E[\widehat q_f]) - \ell(f, Q_f) )^2
}_{ \geq 0 }
\\
& 
+
\underbrace{ 
( ( \E_{\widehat q \sim \calQ_f} [ \widehat \ell(f, \widehat q_f) ]  - \widehat \ell(f, \E[\widehat q_f]) ) 
}_{ \geq \Omega(1/s), \text{ Lemma \ref{lemma:distance_empirical_sa}} } 
( 
\underbrace{
\widehat \ell(f, \E[\widehat q_f]) - \widehat \ell(f, Q_f) )
}_{ \geq \Omega(1/s), \text{ Lemma \ref{lemma:distance_empirical_expectation}} }
+ 
\underbrace{ 
\widehat \ell(f, Q_f) - \ell(f, Q_f) 
}_{ \geq - \tilde O(1/\sqrt{n}), {\text{ Lemma \ref{lemma:generalization_error_conformal_loss}}} }
)
\\
\geq &
\Omega(1/s^2)
,
\end{align*}
where the last inequality is due to $s \ll \sqrt{n}$.

The above inequality thus indicates the result:
\begin{align*}
| \widehat \calL_c^\SA(f) - \calL_c(f) |
\geq
\Omega(1/s)
.
\end{align*}

Next, we begin to prove the upper bound in Theorem
\ref{theorem:learning_bounds_SA}:
\begin{align*}
&
|\calL_c(f, Q_f) - \E_{\widehat q_f}[ \calL_c(f; \widehat q_f)]| 
\\
\leq &
|\calL_c(f, Q_f) - \calL_c(f; \widehat q_f)| 
+ 
\underbrace{|\calL_c(f; \widehat q_f) - \E_{\widehat q_f}[ \calL_c(f; \widehat q_f)]|}_{\leq O(1/\sqrt{s}), \text{ according to Lemma \ref{lemma:upper_bound_variance} }}
\\
\leq &
L |Q_f - \widehat q_f| + O(1/\sqrt{s})
\\
= &
L \big |S_{f, (\lceil(1-\alpha)(s+1)\rceil)} - S_{f, (\lceil(1-\alpha - O(1/\sqrt{s}))(s+1)\rceil)} \big| + O(1/\sqrt{s})
\\
\leq &
L \cdot L_2 \bigg | \frac{\lceil(1-\alpha)(s+1)\rceil}{s} - \frac{\lceil(1-\alpha - O(1/\sqrt{s}))(s+1)\rceil}{s} \bigg | + O(1/\sqrt{s})
\\
\leq &
L \cdot L_2 \cdot O(1/\sqrt{s}) + O(1/\sqrt{s})
= 
O(1/\sqrt{s}) + O(1/\sqrt{s})
= 
O(1/\sqrt{s}),  
\end{align*}
where the first inequality is due to the triangle inequality, the second inequality is due to the $L$-Lipschitz continuous of $\calL_c(f, Q_f)$ in Lemma \ref{lemma:lipschitz_conformal_loss} and Lemma \ref{lemma:upper_bound_variance},
and the third inequality is due to (\ref{eq:concentration_empirical_population_quantiles}) and Assumption \ref{assumption:bi_lipschitz}.

Combining the lower and upper learning bound, the proof of Theorem \ref{theorem:learning_bounds_SA} is finished.
\end{proof}

\subsection{ Proof for Lemma \ref{lemma:distance_empirical_sa} }
\label{subsection:proof_distance_empirical_sa}

\begin{proof}
(of Lemma \ref{lemma:distance_empirical_sa})

We begin to prove the lower bound of 
$\widehat \ell(f, \E_{\widehat q_f} [\widehat q_f] ) - \E_{\widehat q_f} [\widehat \ell(f; \widehat q_f) ]$.

By the $\mu$-strong-concavity of $\widehat \ell(f, q)$ in $q$ (Assumption \ref{assumption:straongly_concave}), 
we have
\begin{align*}
&
\widehat \ell(f, \E_{\widehat q_f} [\widehat q_f] ) - \E_{\widehat q_f} [\widehat \ell(f; \widehat q_f) ] 
=
\E_{\widehat q_f} [ \widehat \ell(f, \E_{\widehat q_f} [\widehat q_f] ) - \widehat \ell(f; \widehat q_f) ] 
\geq 
\E_{\widehat q_f} [ 
\partial_q \widehat \ell(f, \E_{\widehat q_f} [\widehat q_f]) ( \widehat q_f - \E[\widehat q_f] )
+ \frac{\mu}{2} ( \widehat q_f - \E[ \widehat q_f ])^2
]
\\
= &
\frac{\mu}{2} \E_{\widehat q_f} [ ( \widehat q_f - \E[\widehat q_f] )^2 ]
\geq
L_1^2 \cdot \frac{\mu}{2} \frac{\E_j \Big[ \big( \lceil(1-\alpha)(n+1)\rceil - j \big )^2 \Big]}{n^2}
\geq
L_1^2 \cdot \frac{\mu}{2} \frac{\E \Big[ \big( \E[j] - j \big )^2 \Big]}{n^2}
\\
= &
L_1^2 \cdot \frac{\mu}{2} \cdot \frac{(\lceil(1-\alpha)(s+1)\rceil+1)(s - \lceil(1-\alpha)(s+1)\rceil) }{(s+1)^2(s+2)} 
\geq 
L_1^2 \cdot \frac{\mu}{2} \cdot \Omega(1/s)
=
\Omega(1/s),
\end{align*}
where  
the second inequality is due to Proposition \ref{proposition:quantile_distribution} and Bi-Lipschitz continuity of score $S_{(j)}$ in $j$ (Assumption \ref{assumption:bi_lipschitz}),
the third inequality is due to the property of variance, i.e., $\Var(X) = \E[(\E[X] - X)^2] \leq \E[(a - X)^2]$ for any $a$, 
and
the third equality is due to the variance of Beta distribution ($j$ satisfies Beta distribution, as Proposition \ref{proposition:quantile_distribution}) is $\frac{ab}{(a+b)^2(a+b+a)}$ with the shape parameter $a, b$.
\end{proof}

\subsection{ Proof for Lemma \ref{lemma:distance_empirical_expectation} }
\label{subsection:proof_distance_empirical_expectation}

\begin{proof}
(of Lemma \ref{lemma:distance_empirical_expectation})

To prove Lemma \ref{lemma:distance_empirical_expectation},
we need the following technical lemma (proof is deferred to Section \ref{section:proof_probabilities_lower_bound}).

\begin{lemma}
\label{lemma:probabilities_ub_lb}
Assume $(1-\alpha))(s+1)$ is not an integer.
If $\lceil (1-\alpha))(s+1) \rceil - (1-\alpha))(s+1) \geq \Omega(1/s)$ and $s \leq \sqrt{n}$,
then
$\frac{ \lceil (1-\alpha)(s+1) \rceil }{s+1} - \frac{ \lceil (1-\alpha)(n+1) \rceil }{ n }
\geq 
\Omega(1/s)
.
$
\end{lemma}

By using the $\mu$-strong-concavity of $\ell(f,q)$ in $q$, 
we have
\begin{align*}
&
\ell(f, \E[\widehat q_f]) - \widehat \ell(f, Q_f) 
\geq 
- \partial_q \ell(f, \E[\widehat q_f]) ( Q_f - \E[\widehat q_f] ) + \frac{\mu}{2} ( Q_f - \E[\widehat q_f] )^2
\\
= &
\underbrace{ 
\partial_q \ell(f, \E[\widehat q_f]) 
}_{ > 0 {\text{ due to no tie}}}
( \E[\widehat q_f] - Q_f )
+ \frac{\mu}{2} \underbrace{ ( Q_f - \E[\widehat q_f] )^2 }_{ \geq 0 }
\geq 
B \cdot  
( \E[\widehat q_f] - Q_f ) 
,
\end{align*}
where the second inequality is due to no tie in the distribution of non-conformity scores and 
there exists a value $B > 0$ as in (\ref{eq:gradient_of_conformal_loss_in_q}) such that $\partial_q \ell(f, q) \geq B = 0$.
It suffixes to show the lower bound of $\E[\widehat q_f] - Q_f$.

Due to $j$ satisfies the Beta distribution with the two shape parameters $a = \lceil (1-\alpha)(s+1) \rceil + 1$ and $b = s - \lceil (1-\alpha)(s+1) \rceil$ (Proposition \ref{proposition:quantile_distribution}) 
and 
the mean of Beta random variable is $a/(a+1)$, 
we know that the corresponding probability for $\E[\widehat q_f]$ is $\frac{ \lceil (1-\alpha)(s+1) \rceil }{s+1}$ and it is larger than $1-\alpha$ if $(1-\alpha)(s+1)$ is not an integer.

Recall that $\widehat Q_f$ is the empirical quantile on a set of data.
We regard $\widehat Q_f$ as the ($\frac{ \lceil (1-\alpha)(n+1) \rceil }{ n }$)-quantile on $\calD_\tr$.
Since $Q_f(\alpha)$ is increasing as $\alpha$ decreases, we use the Bi-Lipschitz continuity of $S_{(j)}$ (Assumption \ref{assumption:bi_lipschitz}):
\begin{align*}
\E[\widehat q_f] - Q_f 
= &
\E[\widehat q_f] - \widehat Q_f
+ 
\widehat Q_f - Q_f 
\geq 
L_1 \Bigg( \frac{ \lceil (1-\alpha)(s+1) \rceil }{s+1} - \frac{ \lceil (1-\alpha)(n+1) \rceil }{ n } \Bigg) 
+ 
\Bigg( \widehat Q_f - \widehat Q_f(\alpha - \tilde O(1/\sqrt{n}) \Bigg)
\\
\geq &
L_1 \Omega(1/s)
- L_1 \tilde O(1/\sqrt{n})
\geq 
\Omega(1/s)
,
\end{align*}
where the two inequalities are due to
Lemma \ref{lemma:probabilities_ub_lb},
Bi-Lipschitz continuity of $S_{(j)}$ (Assumption \ref{assumption:bi_lipschitz}) 
and 
the concentration inequality for the empirical and population quantiles in (\ref{eq:concentration_empirical_population_quantiles}).
\end{proof}

\subsection{ Proof for Lemma \ref{lemma:upper_bound_variance} }
\label{subsection:proof_lemma_upper_bound_variance}

\begin{proof}
(of Lemma \ref{lemma:upper_bound_variance})

According to Chebyshev's inequality, we have that: 
\begin{align*}
&
\P \{ | \calL_c(f, \widehat q_f) - \E_{\widehat q_f}[ \calL_c(f, \widehat q_f)] | \leq \epsilon \}
\geq 
1 - \frac{\E_{\widehat q_f}[(\calL_c(f, \widehat q_f) - \E_{\widehat q_f}[ \calL_c(f, \widehat q_f)])^2]}{\epsilon^2}
\\
\geq &
1 - \frac{L^2 \E_{\widehat q_f}[(\widehat q_f - \E_{\widehat q_f}[ \widehat q_f])^2]}{\epsilon^2}
\geq 
1 - \frac{L^2 L_2^2}{\epsilon^2 } \cdot \frac{(\lceil(1-\alpha)(s+1)\rceil+1)(s - \lceil(1-\alpha)(s+1)\rceil)}{(s+1)^2(s+2)}
\geq
1 - \delta
,
\end{align*}
where the first and last inequality hold due to the Chebyshev's inequality, the second inequality is due to the $L$-Lipschitz continuous of $\calL_c(f, Q_f)$ in Lemma \ref{lemma:lipschitz_conformal_loss}.

Rearranging the above inequality, the following inequality holds with probability at least $1 - \delta$:
\begin{align*}
| \calL_c(f, \widehat q_f) - \E_{\widehat q_f}[ \calL_c(f, \widehat q_f)] | 
\leq 
L \cdot L_2 \cdot \sqrt{\frac{(\lceil(1-\alpha)(s+1)\rceil+1)(s - \lceil(1-\alpha)(s+1)\rceil)}{\delta \cdot (s+1)^2(s+2)}} 
\leq
L \cdot L_2 \cdot O(1/\sqrt{s})
=
O(1/\sqrt{s}),
\end{align*}
where the first inequality is due to Assumption \ref{assumption:bi_lipschitz} and Lemma \ref{lemma:lipschitz_conformal_loss}.
\end{proof}

\subsection{ Proof for Lemma \ref{lemma:probabilities_ub_lb} }
\label{section:proof_probabilities_lower_bound}

\begin{proof}
(of Lemma \ref{lemma:probabilities_ub_lb})

Recall that $(1-\alpha)(s+1)$ is not an integer.
Then we begin to compare the upper and lower bound of $\frac{\lceil(1-\alpha)(n+1)\rceil}{n}$ and $\frac{\lceil(1-\alpha))(s+1)\rceil }{s+1 }$, respectively:
\begin{align}
\label{eq:probabilities_ub_lb1}
(1-\alpha)+\frac{1-\alpha}{n} 
\leq &
\frac{\lceil(1-\alpha)(n+1)\rceil}{n}
\leq 
(1-\alpha)+\frac{2-\alpha}{n},
\\
\label{eq:probabilities_ub_lb2}
(1-\alpha) 
< &
\frac{\lceil(1-\alpha))(s+1)\rceil }{s+1 }
\leq
(1-\alpha)+\frac{1}{s+1}, 
\end{align}
where all inequalities are due to $(1-\alpha)(n+1) \leq \lceil(1-\alpha)(n+1)\rceil \leq (1-\alpha)(n+1)+1$.

Due to $s<n$ and $\alpha \in [0,1] $, we have that:
\begin{align*}
\frac{\lceil(1-\alpha))(s+1)\rceil }{s+1 } - \frac{\lceil(1-\alpha)(n+1)\rceil}{n}
&
\leq 
\max \bigg \{ (1-\alpha)+\frac{2-\alpha}{n} - (1-\alpha), (1-\alpha)+\frac{1}{s+1} - (1-\alpha)-\frac{1-\alpha}{n} \bigg \}
\\
&
= 
\max \bigg \{ \frac{2-\alpha}{n}, \frac{1}{s+1} - \frac{1-\alpha}{n} \bigg \}
\\
&
\leq O(1/s).
\end{align*}

On the other hand, 
letting $B = \lceil (1-\alpha))(s+1) \rceil - (1-\alpha))(s+1),$
the lower bound for 
\begin{align*}
\frac{\lceil (1-\alpha))(s+1) \rceil }{s+1 } 
- \frac{\lceil(1-\alpha)(n+1)\rceil}{n}
&
\geq 
\Big( \frac{ (1-\alpha))(s+1) }{s+1} + B \Big)
- \Big( (1-\alpha) + \frac{2-\alpha}{n} \Big)
=
B - \frac{2-\alpha}{n}
\geq 
\Omega(1/s)
,
\end{align*}
where the last inequality is due to the assumption $\lceil (1-\alpha))(s+1) \rceil - (1-\alpha))(s+1) \geq \Omega(1/s)$ and $s \leq \sqrt{n}$.

\end{proof}

\section{ Proof for Theorem \ref{subsection:method} }
\label{section:appendix:proofs_learning_bound}

\begin{theorem}
\label{theorem:learning_bound_DPSM_appendix}
(Theorem \ref{theorem:learning_bound_DPSM} restated, learning bound of DPSM)
Suppose that Assumption \ref{assumption:bi_lipschitz} holds.
For any model $f \in \calF$, the following inequality holds with probability at least $1 - \delta$:
\begin{align*}
| \bar \calL_c^\DPSM ( f, q^*(f) ) - \calL_c(f) | 
\leq 
\tilde O(1/\sqrt{n}).
\end{align*}
\end{theorem}

\begin{proof}
(Proof of Theorem \ref{theorem:learning_bound_DPSM})


To prove Theorem \ref{theorem:learning_bound_DPSM}, we need the following three technical lemma and defer their proof after proving Theorem \ref{theorem:learning_bound_DPSM}.

\begin{lemma}
\label{lemma:distance_population_empirical_quantiles} 
(Distance between $Q_f$ and $q^* \in \calU(f)$)
Let $q^*_f \in \calU(f) = \arg\min_{q \in \R} ~ \widehat \calL^\QR(f, q)$ is a ($1- \alpha$)-quantile computed by minimizing QR loss. 
Suppose that Assumption \ref{assumption:bi_lipschitz} holds.
Then the following inequality holds with probability at least $1-\delta$:
\begin{align*}
\big | Q_f - q^*_f \big | 
\leq 
\tilde O(1 / \sqrt{n}).
\end{align*}
\end{lemma}

\begin{lemma}
\label{lemma:lipschitz_conformal_loss}
(Lipschitz continuity of $\widehat \calL_c^\DPSM(f,q)$)
$\widehat \calL_c^\DPSM(f,q)$ is $L$-Lipschitz continuous in input $q$ with $L = \frac{K}{ 4 \tau_\Sigmoid}$.
\end{lemma}

\begin{lemma}
\label{lemma:generalization_error_conformal_loss}
(Generalization error)
Let $f: \calX \rightarrow \R$ and $\{X_1, ..., X_n\}$ be i.i.d. samples dram from an underlying distribution.
If there exists $M > 0$ such that $| f(X) | \leq M$ for all $X \in \calX$,
then with probability at least $1-\delta$, we have
\begin{align*}
\Big| \E_{X} [f(X)] - \frac{1}{n} \sum_{i=1}^n f(X_i) \Big|
\leq 
M \sqrt{ \frac{ \log(2/\delta) }{2n} }
.
\end{align*}
\end{lemma}

Then we begin to prove Theorem \ref{theorem:learning_bound_DPSM}.
Recall $\bar \calL_c^\DPSM(f) = \widehat \calL_c^\DPSM(f, q^*_f)$,
where $q^*_f \in \calU(f)$,
and define the empirical version of the conformal alignment loss $\widehat \calL_c(f)$
$$
\widehat \calL_c(f) = \sum_{i \in \calD} \Bigg[ \sum_{y\in \calY} \tilde \indicator[S_f(X_i, y) \leq Q_f] \Bigg]
=
\widehat \calL_c^\DPSM(f, Q_f)
,
~~~
\text{ see (\ref{eq:conformal_loss_population}) and (\ref{eq:conformal_loss_DPSM})}
.
$$

We start with the following triangle inequality:
\begin{align*}
| \bar \calL_c^\DPSM ( f ) - \calL_c(f) | 
= &
| \widehat \calL_c^\DPSM ( f, q^*_f ) - \widehat \calL_c ( f) 
+ \widehat \calL_c ( f) - \calL_c(f) |  
\leq 
\underbrace{
| \widehat \calL_c^\DPSM ( f, q^*_f ) - \widehat \calL_c^\DPSM ( f, Q_f ) | 
}_{ L\text{-Lipschitz, Lemma \ref{lemma:lipschitz_conformal_loss} } }
+
\underbrace{
| \widehat \calL_c ( f ) - \calL_c(f) | 
}_{ \leq \tilde O(1/\sqrt{n}), \text{ Lemma \ref{lemma:generalization_error_conformal_loss}}}
\\
\leq &
L \underbrace{
| q^*(f) - Q_f |
}_{ \leq \tilde O(1/\sqrt{n}), \text{ Lemma \ref{lemma:distance_population_empirical_quantiles}}}
+ \tilde O(1/\sqrt{n})
\leq
\tilde O(1/\sqrt{n})
.
\end{align*}
\end{proof}

\subsection{ Proof for Lemma \ref{lemma:distance_population_empirical_quantiles} }
\label{subsection:proof_lemma_distance_population_empirical_quantiles}

\begin{proof}
(of Lemma \ref{lemma:distance_population_empirical_quantiles})

In this proof, define $q^{*}_f(\alpha) \in \calU(f) = \arg\min_{q \in \R} ~ \widehat \calL^\QR(f, q)$.
We explicitly write the probability $1-\alpha$ for our analysis.

First, we prove that $q^{*}_f(\alpha - O(1/\sqrt{n}))
\leq
Q_f(\alpha)
\leq
q^{*}_f(\alpha + O(q/\sqrt{n}))$

Define $Z_i = \indicator{ [ S_{f,i} \leq Q(\alpha) ] }$ where $1 \leq i \leq n$
Thus, $Z_{i}$ is a Bernoulli random variable.
According to the definition of $Q_f(\alpha)$, we have that $\P\{ Z_i = 1 \} = 1 - \alpha$ and $\P\{ Z_i = 0 \} = \alpha$.
Let $\widehat Z = \frac{1}{n} \sum_{i=1}^n Z_i$ and $\E[\widehat Z] = 1-\alpha.$

According to Chernoff bound, we know
\begin{align*}
\P\Bigg\{ \Bigg| \frac{1}{n} \sum_{i=1}^n Z_i - \E[\widehat Z] \Bigg| \geq \varepsilon \E[\widehat Z] \Bigg\}
\leq 
2 \exp\Bigg( - \E[\widehat Z] \varepsilon^2 / 3 \Bigg) 
=
2 \exp\Bigg( - n (1-\alpha) \varepsilon^2 / 3 \Bigg) .
\end{align*}

By setting $\delta = 2 \exp( - n (1-\alpha) \varepsilon^2 / 3 )$, i.e., $\varepsilon = \sqrt{ ( 3 \log(2/\delta) ) / ( ( 1 - \alpha ) n  ) }$, we have with probability at least $1-\delta$:
\begin{align}\label{eq:abs_bound}
\Bigg| \frac{1}{n} \sum_{i=1}^n \indicator[ V_i \leq Q_f(\alpha) ] - ( 1 - \alpha ) \Bigg| 
\leq 
\varepsilon ( 1 - \alpha )
=
\sqrt{ ( 3 ( 1 - \alpha )  \log(2/\delta) ) / n }
=
\tilde O(1 / \sqrt{n}) .
\end{align}

Recall the definition of $q^{*}_f(\alpha) \in \arg\min_{q \in \R} ~ \widehat \calL^\QR(f, q)$.
Then we know the following upper bound and lower bound for $1-\alpha$:
\begin{align*}
( 1 - \alpha )
\leq 
\frac{1}{n} \sum_{i=1}^n \indicator[ S_{f,i} \leq q^{*}_f(\alpha) ] , ~~~
( 1 - \alpha )
\geq 
\frac{1}{n} \sum_{i=1}^n \indicator[ S_{f,i} \leq q^{*}_f(\alpha + 1/n ) ] .
\end{align*}

Re-arranging (\ref{eq:abs_bound}) and using the above upper/lower bounds, with probability at least $1-\delta$, we have
\begin{align}\label{eq:upper_and_lower_bounds_quantile}
&
( 1 - \alpha ) ( 1 - \varepsilon )
\leq 
\frac{1}{n} \sum_{i=1}^n \indicator[ S_{f,i} \leq Q_f(\alpha) ]
\leq 
( 1 - \alpha ) ( 1 + \varepsilon)
\nonumber\\
\Leftrightarrow ~~~
& 
1 - ( \underbrace{ 1 - ( 1 - \alpha ) ( 1 - \varepsilon ) }_{ = \alpha' } )
\leq 
\frac{1}{n} \sum_{i=1}^n \indicator[ S_{f,i} \leq Q_f(\alpha) ]
\leq 
1 - ( \underbrace{ 1 - ( 1 - \alpha ) ( 1 + \varepsilon) }_{ = \alpha'' } )
\nonumber\\
\Rightarrow ~~~ 
& 
\frac{1}{n} \sum_{i=1}^n \indicator[ S_{f,i} \leq q^{*}_f(\alpha' + 1/n ) ]
\leq
\frac{1}{n} \sum_{i=1}^n \indicator[ S_{f,i} \leq Q_f(\alpha)  ]
\leq 
\frac{1}{n} \sum_{i=1}^n \indicator[ S_{f,i} \leq q^{*}_f(\alpha'' ) ] 
\nonumber\\
\Leftrightarrow ~~~
&
q^{*}_f(\alpha' + 1/n)
\leq 
Q_f(\alpha)
\leq 
q^{*}_f(\alpha'') .
\end{align}

Finally, by using the definition of $\varepsilon$ above, 
we analyze $\alpha'$ and $\alpha''$ as follows
\begin{align}\label{eq:alpha_prime}
\alpha' 
=
1 - (1-\alpha) (1-\varepsilon)
=
\alpha + \varepsilon (1-\alpha)
=
\alpha + \sqrt{ 3 ( 1 - \alpha ) \log(2/\delta) / n }
=
\alpha + \tilde O(1/\sqrt{n}),
\\
\label{eq:alpha_prime_prime}
\alpha''
=
1 - (1-\alpha) (1+\varepsilon)
=
\alpha - \varepsilon (1-\alpha)
=
\alpha - \sqrt{ 3 ( 1 - \alpha ) \log(2/\delta) / n }
=
\alpha - \tilde O(1/\sqrt{n}).
\end{align}

Thus, plugging (\ref{eq:alpha_prime}) and (\ref{eq:alpha_prime_prime}) into (\ref{eq:upper_and_lower_bounds_quantile}), 
we have
\begin{align}\label{eq:concentration_empirical_population_quantiles}
q^{*}_f(\alpha + \tilde O(1/\sqrt{n}))
\leq
Q_f(\alpha)
\leq
q^{*}_f(\alpha - \tilde O(1/\sqrt{n})) .
\end{align}

Next, we prove that $\big |Q_f(\alpha) - q^{*}_f(\alpha)\big | \leq \tilde O(1/\sqrt{n})$ from (\ref{eq:concentration_empirical_population_quantiles}):
\begin{align*}
|Q_f(\alpha) - q^{*}_f(\alpha)\big |
\leq
L_2 \bigg | \frac{\lceil(1-\alpha)(n+1)\rceil}{n} - \frac{\lceil(1-\alpha - \tilde O(1/\sqrt{n}))(n+1)\rceil}{n} \bigg |
\leq \tilde O(1/\sqrt{n}),
\end{align*}
where the first inequality is due to Assumption \ref{assumption:bi_lipschitz}.
\end{proof}

\subsection{Proof for Lemma \ref{lemma:lipschitz_conformal_loss}}
\label{section:proof_lemma_lipschitz_conformal_loss}


\begin{proof}
(of Lemma \ref{lemma:lipschitz_conformal_loss})

Recall that $\tilde \indicator[ S \leq q ] = \frac{1}{1+\exp (- \frac{q-S}{\tau_\Sigmoid})}$.
Define $u = \frac{q - S }{\tau_\Sigmoid}$.

We start with the gradient of $\widehat \calL_c^\DPSM(f, q)$ in $q$ as follows
\begin{align}\label{eq:gradient_of_conformal_loss_in_q}
\frac{\partial \widehat \calL_c^\DPSM(f, q)}{\partial q} 
= &
\frac{\partial \sum_{y \in \calY } \tilde \indicator \big [ S \leq q \big ]}{\partial q}
= 
\sum_{y \in \calY} \Big ( \frac{\partial \big ( \frac{1}{1+ \exp(-u)} \big ) }{\partial u} \cdot \frac{ \partial u}{\partial q}\Big )
=
\sum_{y \in \calY} \Big ( \frac{\exp(-u)}{(1+ \exp(-u))^2} \cdot \frac{1}{\tau_\Sigmoid}\Big )
\nonumber\\
= &
\sum_{y \in \calY} \frac{\exp (- \frac{q-S}{\tau_\Sigmoid} )}{ \tau_\Sigmoid \big ( 1+\exp (- \frac{q-S}{\tau_\Sigmoid}) \big )^2}
> 0,
\end{align}
where the first inequality is due to $\exp(x) >0 $. 
Thus, $\widehat \calL_c^\DPSM(f, q)$ is a increasing function in input $q$.

For each component of the summation in above (\ref{eq:gradient_of_conformal_loss_in_q}), define $h(u) = \frac{\exp(-u)}{(1+\exp(-u))^2}$, then the gradient of $h(u)$ is:
\begin{align}\label{eq:each_component_in_gradient_of_conformal_loss_in_q}
h'(u) & = \frac{-\exp(-u)(1+\exp(-u))^2 - \exp(-u) \cdot  2 (1+\exp(-u)) \cdot  (-\exp(-u)) }{(1+\exp(-u))^4}
\nonumber\\
& = \frac{\big (1+\exp(-u) \big )\cdot \exp(-u) \cdot \Big ( - \big (1+\exp(-u) \big ) + 2 (\exp(-u) \Big ) }{(1+\exp(-u))^4}
\nonumber\\
& = \frac{\big (1+\exp(-u) \big )\cdot \exp(-u) \cdot \Big ( \exp(-u) -1 \Big ) }{(1+\exp(-u))^4}
.
\end{align}

We can take a closer look at the value of $h'(u)$ in (\ref{eq:each_component_in_gradient_of_conformal_loss_in_q}):
when $\exp(-u) \in (0,1)$, or equivalently $u > 0$, we have $h'(u) > 0$; 
when $ \exp(-u) \in (1, +\infty)$, or equivalently $u < 0$, we have $h'(u) < 0$.
It means that $h(u)$ increases when $u >0$ and decreases when $u < 0$.
Further, when $u=0$, we achieve the maximal value of $h(u)$, 
i.e., $\max_{u} h(u) = h(0) = 1/4$,
or equivalently $\frac{\exp(-u)}{(1+\exp(-u))^2} \leq 1/4$.

Therefore, considering $\frac{ \partial \widehat \calL_c^\DPSM(f,q)}{\partial q} > 0$ as in (\ref{eq:gradient_of_conformal_loss_in_q}) and plugging the above inequality into (\ref{eq:gradient_of_conformal_loss_in_q}), $\Big| \frac{\partial \widehat \calL_c^\DPSM(f, q)}{\partial q} \Big|$ is bounded as follows
\begin{align*}
\Bigg| \frac{\partial \widehat \calL_c^\DPSM(f, q)}{\partial q} \Bigg|
\leq 
\Bigg| \sum_{y \in \calY} \frac{ \exp (- \frac{q-S}{\tau_\Sigmoid} ) }{ \tau_\Sigmoid \big ( 1+\exp (- \frac{q-S}{\tau_\Sigmoid}) \big )^2} \Bigg| 
\leq
\sum_{y \in \calY} \frac{1}{ 4 \tau_\Sigmoid } 
=
\frac{K}{ 4 \tau_\Sigmoid}
,
\end{align*}
if $\tau_\Sigmoid \neq 0$ and $|\calY| = K$.
Therefore, $\calL_c(f, q)$ is $L$-Lipschitz continuous for input $q$
with $L = \frac{K}{ 4 \tau_\Sigmoid}$.
\end{proof}

\subsection{ Proof for Lemma \ref{lemma:generalization_error_conformal_loss} }
\label{subsection:proof_generalization_error_conformal_loss}

\begin{proof}
(of Lemma \ref{lemma:generalization_error_conformal_loss})


Let $A_i = f(X_i)$ with $|A_i| \leq M$.
According to Hoeffding's inequality, we have:
\begin{align*}
\P\Bigg\{ \Bigg| \frac{1}{n} \sum_{i=1}^n A_i - \E[A] \Bigg| \geq \epsilon \Bigg\}
\leq 
2 \exp\Bigg( - \frac{2 n \epsilon^2}{ M^2 } \Bigg).
\end{align*}

By setting $\delta = 2 \exp ( - 2n\epsilon^2/M^2)$, i.e., $\epsilon = M \sqrt{\log(2/\delta)/(2n)}$, we have the following inequality:
\begin{align*}
\P\Bigg\{ 
\Bigg| 
\frac{1}{n} \sum_{i=1}^n A_i - \E[A] 
\Bigg| 
\leq 
M \sqrt{\frac{\log(2/\delta)}{2n}}
\Bigg\}
\geq 
1-\delta
.
\end{align*}

Therefore, with probability at least $1-\delta$, we have
\begin{align*}
\Big| \E_X[f(X)] - \frac{1}{n} \sum_{i=1}^n f(X_i) \Big| 
\leq 
M \sqrt{\frac{\log(2/\delta)}{2n}}
\leq 
\tilde O(1/\sqrt{n})
.
\end{align*}

\end{proof}

\section{ Proofs for Results in Section \ref{subsection:optimization} }
\label{appendix:section:proof_convergence}

\begin{lemma}
\label{lemma:LEB_QR_loss}
(Lemma \ref{lemma:HEB_QR_loss} restated, H\"olderian error bound condition for QR loss)
Suppose there is no tie in $\{S_i\}_{i=1}^n$.
Then, fixing $f$, the QR loss $\widehat \calL^\QR(f,q)$ satisfies H\"olderian error bound w.r.t. $q$ for $\nu=1$.
\end{lemma}

\begin{proof}
(of Lemma \ref{lemma:LEB_QR_loss})

Below we assume $f$ is fixed and let $g(q) \triangleq \widehat \calL^\QR(f,q)$,
so we omit the dependence of $g$ on $f$ for simplicity of notations.
Similarly, we denote the optimal solution set by $\calU = \arg\min_{q} g(q)$ in this proof.

In the following, we consider two cases to prove Lemma \ref{lemma:LEB_QR_loss}.

{\bf Case 1.}
Suppose we have two quantile variables $q_1, q_2$ such that $q_1 \notin \calU$, $q_1 < q_2 \in \calU$ and $q_2 = \arg\min_{q' \in \calU} \| q' - q_1 \|$ is the quantile in the optimal set closest to $q_1$.
Then we have
\begin{align}\label{eq:leb_proof_intermediate1}
n \big( g(q_1) - g(q_2) \big)
= &
\sum_{i=1}^n ( \rho_\alpha( q_1, S_i ) - \rho_\alpha(q_2, S_i) )
\nonumber\\
= &
\sum_{i=1}^n \Bigg(
\indicator[ S_i \leq q_1 ] \bigg( \alpha ( q_1 - S_i ) - \alpha ( q_2 - S_i ) \bigg)
+ \indicator[ q_1 < S_i < q_2 ] \bigg( ( 1 - \alpha ) ( S_i - q_1 ) - \alpha ( q_2 - S_i ) \bigg)
\nonumber\\
& \qquad \qquad
+ \indicator[ q_2 \leq S_i ] \bigg( (1-\alpha) (S_i - q_1) - (1-\alpha) (S_i - q_2) \bigg)
\Bigg)
\nonumber\\
= &
\sum_{i=1}^n \Bigg(
\indicator[ S_i \leq q_1 ] \bigg( \alpha ( q_1 - q_2 ) ) \bigg)
+ \indicator[ q_1 < S_i < q_2 ] \bigg( S_i - q_1 + \alpha (q_1 - q_2) \bigg)
\nonumber\\
& \qquad \qquad
+ \indicator[ q_2 \leq S_i ] \bigg( (1-\alpha) (q_2 - q_1) \bigg)
\Bigg)
\\
= &
\sum_{i=1}^n \Bigg( 
\indicator[ S_i \leq q_2 ] \alpha ( q_1 - q_2 ) 
+ \indicator [ q_2 \leq S_i ] (1-\alpha) (q_2 - q_1)
+ \indicator [ q_1 < S_i < q_2 ] ( S_i - q_1 )
\Bigg)
\nonumber\\
= &
\sum_{i=1}^n \indicator [ S_i \leq q_2 ] \alpha (q_1 - q_2)
+ \sum_{i=1}^n \indicator [ q_2 \leq S_i ] (1-\alpha) (q_2-q_1)
+ \sum_{i=1}^n \indicator [ q_1 < S_i < q_2 ] (S_i - q_1)
\nonumber\\
= &
(1-\alpha) \alpha (q_1 - q_2)
- \alpha (1-\alpha) (q_1 - q_2)
+ \sum_{i=1}^n \indicator [ q_1 < S_i < q_2 ] ( S_i - q_1 )
\nonumber\\
\label{eq:leb_proof_intermediate2}
= &
\sum_{i=1}^n \indicator [ q_1 < S_i < q_2 ] ( S_i - q_1 )
.
\end{align}

On the other hand, let $\tilde n = \sum_{i=1}^n \indicator[q_1 < S_i < q_2]$.
Denote $(1-c') q_1 + c' q_2$ as the weighted average between $q_1$ and $q_2$ with a weight parameter $c' \in (0,1)$.
Due to that the weighted average between $q_1$ and $q_2$ can be always smaller than the average score in the set $\{S_i : q_1 < S_i < q_2\}$ for a sufficiently small $c' \in (0,1)$,
we have
\begin{align*}
&
q_1 + c' ( q_2 - q_1 )
=
(1-c') q_1 + c' q_2 
\leq
\frac{1}{\tilde n} \sum_{i=1}^n \indicator [ q_1 < S_i < q_2 ] S_i
\\
\Leftrightarrow 
\quad &
\tilde n ( q_1 + c' ( q_2 - q_1 ) )
\leq 
\sum_{i=1}^n \indicator [ q_1 < S_i < q_2 ] ( S_i - q_1 + q_1 )
\\
\Leftrightarrow 
\quad &
c' \tilde n ( q_2 - q_1 )
\leq 
\sum_{i=1}^n \indicator [ q_1 < S_i < q_2 ] ( S_i - q_1 )
\stackrel{ (\ref{eq:leb_proof_intermediate2}) }{ = }
n \big( g(q_1) - g(q_2) \big)
\\
\Leftrightarrow 
\quad &
\frac{c'\tilde n}{n} (q_2 - q_1)
\leq 
g(q_1) - g(q_2)
.
\end{align*}

Recall that $q_2 = \arg\min_{q' \in \calU} \| q' - q_1 \|$ and $(q_2 - q_1) = \dist(q_1, \calU)$.
Plug it into the above equality, we have
\begin{align*}
\frac{c'\tilde n}{n} \dist(q_1, \calU)
\leq 
g(q_1) - g(q_2)
.
\end{align*}

Denoting $c = n / ( \tilde n c' ) > 0$, then we have
\begin{align}\label{eq:leb_holds_case1}
\dist(q_1, \calU)
\leq 
c( g(q_1) - \min_{q'} g(q) )
,
\end{align}
which satisfies the HEB condition shown as Definition \ref{definition:HEB} with the exponent $\nu = 1$.

{\bf Case 2.}
Suppose that $q_1 \in \calU$, $q_1 < q_2 \notin \calU$ and $q_1 = \arg\min_{q' \in \calU} \| q' - q_2\|$.
We can start from (\ref{eq:leb_proof_intermediate1}) as follows
\begin{align*}
n \big( g(q_1) - g(q_2) \big)
= &
\sum_{i=1}^n \Bigg(
\indicator[ S_i \leq q_1 ] \bigg( \alpha ( q_1 - q_2 ) ) \bigg)
+ \indicator[ q_1 < S_i < q_2 ] \bigg( S_i - q_1 + \alpha (q_1 - q_2) \bigg)
\nonumber\\
& \qquad \qquad
+ \indicator[ q_2 \leq S_i ] \bigg( (1-\alpha) (q_2 - q_1) \bigg)
\Bigg)
\\
= &
\sum_{i=1}^n \Bigg( 
\indicator[S_i \leq q_1] \alpha (q_1 - q_2)
+ \indicator[ q_1 \leq S_i ] (1-\alpha) (q_2-q_1)
+ \indicator[q_1<S_i<q_2] (S_i-q_2)
\Bigg)
\\
= &
\sum_{i=1}^n \indicator[q_1 < S_i < q_2] (S_i - q_2)
.
\end{align*}

With the same notation of $\tilde n = \sum_{i=1}^n \indicator[q_1 < S_i < q_2]$, we know that there must exist a constant $c' \in (0,1)$ such that 
\begin{align*}
q_2 + c' (q_1-q_2)
=
c' q_1 + (1-c') q_2
\geq 
\frac{1}{\tilde n} \sum_{i=1}^n \indicator[q_1 < S_i < q_2] S_i
\end{align*}

Therefore, we have 
\begin{align*}
g(q_1) - g(q_2)
\leq
\frac{\tilde n}{n} c' (q_1-q_2)
=
- \frac{\tilde n}{n} c' \dist(q_2, \calU)
.
\end{align*}

Denoting $c=n/(\tilde n c) > 0$, 
then we have
\begin{align}\label{eq:leb_holds_case2}
\dist(q_2, \calU)
\leq 
c ( g(q_2) - \min_{q' \in \calU} g(q') )
.
\end{align}

By combining (\ref{eq:leb_holds_case1}) and (\ref{eq:leb_holds_case2}), 
we show that the QR loss satisfies the H\"olderian error bound condition with a constant $c > 0$ and an exponent $\nu=1$.
\end{proof}

\section{Additional Experimental Setup Details}
\label{subsec:train_test_strategies}

\noindent \textbf{Dataset and Split.}
We consider the benchmark datasets CIFAR-100 \cite{krizhevsky2009learning}, Caltech-101 \cite{FeiFei2004LearningGV},
and iNaturalist \cite{van2018inaturalist}. 
We split the original testing datasets into used calibration and testing datasets. 
Table \ref{tab:Data_stat} summarizes key statistics of the used datasets which we elaborate on in the following.

\noindent \textbf{Hyperparameters for training.}
We set datasets, base models, batch size, training epochs, training parameters (learning rate, learning schedule, momentum, gamma, and weight decay), and $\lambda$ as hyperparameter choices. 
We search for hyperparameters on batch size $\in \{64, 128\}$, epochs $\in \{30,40, 60\}$, learning rate ($\eta$) $\in \{0.001, 0.005, 0.01, 0.05, 0.1\}$, learning rate schedule $\in \{[3], [25], [25, 40]\}$, Momentum $=0.9$, weight decay $\in \{ 0.1, 0.97\}$, and $\lambda = \{0.01, 0.05, 0.1, 0.5, 1.0, 5.0\}$ to select the best combination of hyperparameters of each methods.
We also search the learning rate for lower function ($\gamma$) $\in \{0.0001, 0.001, 0.005, 0.01, 0.05, 0.1, 0.2, 0.5\}$ for DPSM.
The hyerparameters employed to get the results presented in the main paper are summarized in Table \ref{tab:finetune_hyper_params}.

\begin{table*}[!ht]
\centering
\caption{Description of the data sets are given in the table. $^{*}$The number of classes in the iNaturalist data set depends on the taxonomy level (e.g., species, genus, family). We employ"Fungi" species which has 341 different categories.}
\label{tab:Data_stat}
\begin{tabular}{|p{2.8cm}|p{2.3cm}|p{2.3cm}|p{2.3cm}|p{2.3cm}|p{2.3cm}|}
\hline
Data            & Number of Classes    & Number of Training Data      & Number of Validation Data     & Number of Calibration Data      & Number of Test Data   \\ \hline
CIFAR-100        & 100                  & 45000                        & 5000                          & 3000                            & 7000                  \\ \hline
Caltech-101     & 101                  & 4310                         & 1256                          & 1111                            & 2000                  \\ \hline
iNaturalist     & 341*                 & 15345                        & 1705                          & 1410                            & 2000                  \\ \hline
\end{tabular}
\end{table*}

\begin{table}[!ht]
\centering
\caption{The below table shows the details we used to train our models. We reported the hyperparameters which gives the best predictive efficiency. We employed SGD optimizer for all training unless specified.}
\label{tab:finetune_hyper_params}
\resizebox{\textwidth}{!}{
\begin{tabular}{|c|c|c|c|c|c|c|c|c|c|}
\hline
Data                           & Architecture & Batch size & Epochs & $\eta$    & lr schedule & Momentum & weight decay & $\gamma$ & $\lambda$ \\ \hline
\multirow{2}{*}{CIFAR-100}      & DenseNet     & 64         & 40     & 0.1   & 25          & 0.9      & 0.1   & 0.01       & 0.05    \\ \cline{2-10} 
                               & ResNet       & 128        & 40     & 0.1   & 25          & 0.9      & 0.1   & 0.01       & 0.01       \\ \hline
\multirow{2}{*}{Caltech-101}   & DenseNet     & 128        & 60     & 0.05  & 25, 40      & 0.9      & 0.1   & 0.1       & 1.0       \\ \cline{2-10} 
                               & ResNet       & 128        & 60     & 0.05  & 25, 40      & 0.9      & 0.1   & 0.05          & 0.1       \\ \hline 
\multirow{2}{*}{iNaturalist}   & DenseNet     & 128        & 60     & 0.001 & 3           & 0.9      & 0.97  & 0.001          & 1.0      \\ \cline{2-10} 
                               & ResNet       & 128        & 60     & 0.001 & 3           & 0.9      & 0.97  & 0.001          & 0.5      \\ \hline
\end{tabular}
}
\end{table}

\noindent \textbf{HPS$\shortrightarrow$Training/Calibration/Testing.} In our study, we primarily apply cross-entropy loss for the classification model. For the ConfTr and DPSM methods, we minimize both the classification loss and the HPS non-conformity score-based predictive inefficiency loss during training. Post-training, we estimate the HPS non-conformal scores using equation \ref{eq:HPS} for all three methods during calibration and testing. We report the marginal coverage and the prediction set size with the correpsonding mean and standard deviation, with the results presented in Table \ref{tab:cvg_set_hps_train_hps_test}.

\textbf{HPS $\shortrightarrow$ Training, APS $\shortrightarrow$ Calibration/Testing.} 
In our study, we primarily apply cross-entropy loss for the classification model. For the ConfTr and DPSM methods, we minimize both the classification loss and the differentiable HPS non-conformity score-based inefficiency loss during training. After training, we estimate APS non-conformal scores using equation \ref{eq:APS}. We calculate the marginal coverage and the prediction set size, with the results presented in Table \ref{tab:cvg_set_hps_train_aps_test}.

\textbf{HPS $\shortrightarrow$ Training, RAPS $\shortrightarrow$ Calibration/Testing.} 
In our study, we primarily apply cross-entropy loss for the classification model. For the ConfTr and DPSM methods, we minimize both the classification loss and the differentiable HPS non-conformity score-based inefficiency loss during training. After training, we estimate RAPS non-conformal scores using equation \ref{eq:RAPS}. 
Then, we calculate the marginal coverage and the prediction set size, with the results presented in Table \ref{tab:cvg_set_hps_train_raps_test}.


\section{Additional Experiments}
\label{appendix:subsec:additional_exps}

\subsection{Additional Experiments for Marginal Coverage }
\label{appendix:subsec:additional_exps_marginal}

\begin{table*}[!ht]
\centering
\caption{
\textbf{HPS $\shortrightarrow$ Training/Calibration/Testing}: 
The APSS on three different datasets with two different deep models trained and calibrated with HPS when $\alpha = 0.1$. 
$\downarrow$ indicates the percentage improvement in predictive efficiency compared to the best existing method, whereas $\uparrow$ denotes a percentage decrease in predictive efficiency. 
All results are the average over 10 different runs, with the mean and standard deviation reported.
DPSM significantly outperforms almost the best baselines with $20.32\%$ prediction set size reduction across all datasets. 
}
\label{tab:cvg_set_hps_train_hps_test}
\resizebox{\textwidth}{!}{
\begin{NiceTabular}{@{}c|cccc|cccc@{}}
\toprule
\multirow{2}{*}{Model} & \multicolumn{4}{c|}{Marginal Coverage} & \multicolumn{4}{c }{Prediction Set Size} \\ 
\cmidrule(lr){2-5} \cmidrule(lr){6-9}
& CE & CUT & ConfTr & DPSM & CE & CUT & ConfTr & DPSM  \\ 
\midrule
\Block{1-*}{CalTech-101}
\\
\midrule
DenseNet 
& 0.90 $\pm$ 0.006 & 0.90 $\pm$ 0.005 & 0.90 $\pm$ 0.008 & 0.90 $\pm$ 0.003
& 3.50 $\pm$ 0.10  & 1.62 $\pm$ 0.030 & 4.10 $\pm$ 0.19  & \textbf{0.90 $\pm$ 0.003} ($\downarrow$ 44.44\%)  
\\ 
ResNet   
& 0.90 $\pm$ 0.004 & 0.90 $\pm$ 0.007 & 0.90 $\pm$ 0.006 & 0.90 $\pm$ 0.005
& 1.57 $\pm$ 0.018 & 1.64 $\pm$ 0.049 & 1.52 $\pm$ 0.040 & \textbf{0.91 $\pm$ 0.005} ($\downarrow$ 44.51\%) 
\\ 
\midrule
\Block{1-*}{CIFAR-100}
\\
\midrule
DenseNet 
& 0.90 $\pm$ 0.007 & 0.90 $\pm$ 0.009 & 0.90 $\pm$ 0.007 & 0.90 $\pm$ 0.006
& 2.59 $\pm$ 0.053 & 2.27 $\pm$ 0.09  & 2.28 $\pm$ 0.07  & \textbf{2.17 $\pm$ 0.086 }  ($\downarrow$ 4.82\%) 
\\ 
ResNet   
& 0.90 $\pm$ 0.006 & 0.90 $\pm$ 0.005 & 0.90 $\pm$ 0.007 & 0.90 $\pm$ 0.007
& 3.39 $\pm$ 0.10  & 3.01 $\pm$ 0.11  & 3.77 $\pm$ 0.14  & \textbf{2.94 $\pm$ 0.08}  ($\downarrow$ 2.32\%)
\\ 
\midrule
\Block{1-*}{iNaturalist}
\\
\midrule
DenseNet 
& 0.90 $\pm$ 0.009 & 0.90 $\pm$ 0.011 & 0.90 $\pm$ 0.011 & 0.90 $\pm$ 0.008
& 94.58 $\pm$ 3.45 & 77.13 $\pm$ 3.72 & 79.93 $\pm$ 3.70 & \textbf{61.22 $\pm$ 2.49} ($\downarrow$ 20.63\%) 
\\ 
ResNet   
& 0.90 $\pm$ 0.019 & 0.90 $\pm$ 0.007 & 0.90 $\pm$ 0.012 & 0.90 $\pm$ 0.008
& 99.48 $\pm$ 8.95 & 73.09 $\pm$ 2.00 & 76.73 $\pm$ 3.87 & \textbf{70.04 $\pm$ 1.99} ($\downarrow$ 4.17\%) 
\\ 
\bottomrule
\end{NiceTabular}
}
\end{table*}

\begin{table*}[!t]
\centering
\caption{\textbf{HPS $\shortrightarrow$ Training, APS $\shortrightarrow$ Calibration/Testing}: 
The APSS on three different datasets with two different deep models trained with HPS and calibrated with APS when $\alpha = 0.1$. 
$\downarrow$ indicates the percentage improvement in predictive efficiency compared to the best existing method, whereas $\uparrow$ denotes a percentage decrease in predictive efficiency. 
All results are the average over 10 different runs, with the mean and standard deviation reported.
DPSM significantly outperforms almost the best baselines with $20.61\%$ prediction set size reduction across all datasets. 
}
\label{tab:cvg_set_hps_train_aps_test}
\resizebox{\textwidth}{!}{
\begin{NiceTabular}{@{}c|cccc|cccc@{}}
\toprule
\multirow{2}{*}{Model} & \multicolumn{4}{c|}{Marginal Coverage} & \multicolumn{4}{c }{Prediction Set Size} \\ 
\cmidrule(lr){2-5} \cmidrule(lr){6-9}
& CE & CUT & ConfTr & DPSM & CE & CUT & ConfTr & DPSM  \\ 
\midrule
\Block{1-*}{CalTech-101}
\\
\midrule
DenseNet 
& 0.90 $\pm$ 0.006 & 0.90 $\pm$ 0.007 & 0.90 $\pm$ 0.008 & 0.90 $\pm$ 0.005
& 8.44  $\pm$ 0.15 & 3.87 $\pm$ 0.11  & 8.64 $\pm$ 0.21  & \textbf{1.58 $\pm$ 0.022} ($\downarrow$ 59.17\%) 
\\ 
ResNet   
& 0.90 $\pm$ 0.006 & 0.90 $\pm$ 0.004 & 0.90 $\pm$ 0.007 & 0.90 $\pm$ 0.005
& 4.50 $\pm$ 0.059 & 4.59 $\pm$ 0.072 & 3.61 $\pm$ 0.08 & \textbf{1.74 $\pm$ 0.031} ($\downarrow$ 51.80\%) 
\\ 
\midrule
\Block{1-*}{CIFAR-100}
\\
\midrule
DenseNet 
& 0.90 $\pm$ 0.007 & 0.90 $\pm$ 0.009 & 0.90 $\pm$ 0.008 & 0.90 $\pm$ 0.006
& 3.38 $\pm$ 0.12 & \textbf{2.41 $\pm$ 0.11} & 3.08 $\pm$ 0.11   & 2.64 $\pm$ 0.86 ($\uparrow$ 8.71\%)
\\ 
ResNet   
& 0.90 $\pm$ 0.006 & 0.90 $\pm$ 0.11 & 0.90 $\pm$ 0.007 & 0.90 $\pm$ 0.007 
& 3.98 $\pm$0.13 & 3.81 $\pm$ 0.08 & 4.90 $\pm$ 0.18  & \textbf{3.53 $\pm$ 0.11} ($\downarrow$ 7.35\%)
\\ 
\midrule
\Block{1-*}{iNaturalist}
\\
\midrule
DenseNet 
& 0.90 $\pm$ 0.009 & 0.90 $\pm$ 0.010 & 0.90 $\pm$ 0.011 & 0.90 $\pm$ 0.010
& 101.97 $\pm$ 3.16 & 88.93 $\pm$ 3.06 & 90.79 $\pm$ 3.17 & \textbf{75.98 $\pm$ 2.99} ($\downarrow$ 14.56\%)
\\ 
ResNet   
& 0.90 $\pm$ 0.013 & 0.90 $\pm$ 0.006 & 0.90 $\pm$ 0.012 & 0.90 $\pm$ 0.009
& 95.81 $\pm$ 3.80 & \textbf{79.00 $\pm$ 2.21} & 88.70 $\pm$ 3.88   & 79.43 $\pm$ 2.39 ($\uparrow$ -0.54\%) 
\\ 
\bottomrule
\end{NiceTabular}
}
\end{table*}

\begin{table*}[!t]
\centering
\caption{\textbf{HPS $\shortrightarrow$ Training, RAPS $\shortrightarrow$ Calibration/Testing}: 
The APSS on three different datasets with two different deep models trained with HPS and calibrated with RAPS when $\alpha = 0.1$, where $\lambda_{\RAPS} = 0.01$ and $k_\reg = 5$. 
$\downarrow$ indicates the percentage improvement in predictive efficiency compared to the best existing method, whereas $\uparrow$ denotes a percentage decrease in predictive efficiency. 
All results are the average over 10 different runs, with the mean and standard deviation reported.
DPSM significantly outperforms almost the best baselines with $19.04\%$ prediction set size reduction across all datasets. 
}
\label{tab:cvg_set_hps_train_raps_test}
\resizebox{\textwidth}{!}{
\begin{NiceTabular}{@{}c|cccc|cccc@{}}
\toprule
\multirow{2}{*}{Model} & \multicolumn{4}{c|}{Marginal Coverage} & \multicolumn{4}{c }{Prediction Set Size} \\ 
\cmidrule(lr){2-5} \cmidrule(lr){6-9}
& CE & CUT & ConfTr & DPSM & CE & CUT & ConfTr & DPSM  \\ 
\midrule
\Block{1-*}{CalTech-101}
\\
\midrule
DenseNet 
& 0.90 $\pm$ 0.007 & 0.90 $\pm$ 0.008 & 0.90 $\pm$ 0.009 & 0.90 $\pm$ 0.006
& 6.58  $\pm$ 0.12 & 3.53 $\pm$ 0.11  & 6.88 $\pm$ 0.17  & \textbf{1.50 $\pm$ 0.020} ($\downarrow$ 57.51\%) 
\\ 
ResNet   
& 0.90 $\pm$ 0.005 & 0.90 $\pm$ 0.004 & 0.90 $\pm$ 0.007 & 0.90 $\pm$ 0.004
& 3.89 $\pm$ 0.047 & 3.98 $\pm$ 0.59 & 3.19 $\pm$ 0.69 & \textbf{1.67 $\pm$ 0.025} ($\downarrow$ 47.65\%) 
\\ 
\midrule
\Block{1-*}{CIFAR-100}
\\
\midrule
DenseNet 
& 0.90 $\pm$ 0.007 & 0.90 $\pm$ 0.007 & 0.90 $\pm$ 0.006 & 0.90 $\pm$ 0.006
& 2.73 $\pm$ 0.043 & \textbf{2.14 $\pm$ 0.55} & 2.69 $\pm$ 0.053   & 2.34 $\pm$ 0.035 ($\uparrow$ 8.55\%)
\\ 
ResNet   
& 0.90 $\pm$ 0.006 & 0.90 $\pm$ 0.11 & 0.90 $\pm$ 0.007 & 0.90 $\pm$ 0.007 
& 3.25 $\pm$ 0.14 & \textbf{2.93 $\pm$ 0.60} & 4.02 $\pm$ 0.11  & \textbf{2.93 $\pm$ 0.05} 
\\ 
\midrule
\Block{1-*}{iNaturalist}
\\
\midrule
DenseNet 
& 0.90 $\pm$ 0.013 & 0.90 $\pm$ 0.008 & 0.90 $\pm$ 0.015 & 0.90 $\pm$ 0.008
& 97.27 $\pm$ 4.23 & 82.88 $\pm$ 2.47 & 86.77 $\pm$ 6.17 & \textbf{68.17 $\pm$ 2.18} ($\downarrow$ 17.74\%)
\\ 
ResNet   
& 0.90 $\pm$ 0.012 & 0.90 $\pm$ 0.011 & 0.90 $\pm$ 0.015 & 0.90 $\pm$ 0.009
& 97.43 $\pm$ 4.17 & \textbf{76.75 $\pm$ 4.47} & 81.18 $\pm$ 4.93   & 76.81 $\pm$ 2.68 ($\uparrow$ 0.08\%) 
\\ 
\bottomrule
\end{NiceTabular}
}
\end{table*}

\begin{figure*}[!t]
    \centering
    \begin{minipage}[t]{0.32\linewidth}
    \centering
    \textbf{(a)} Caltech-101
    \includegraphics[width = \linewidth]{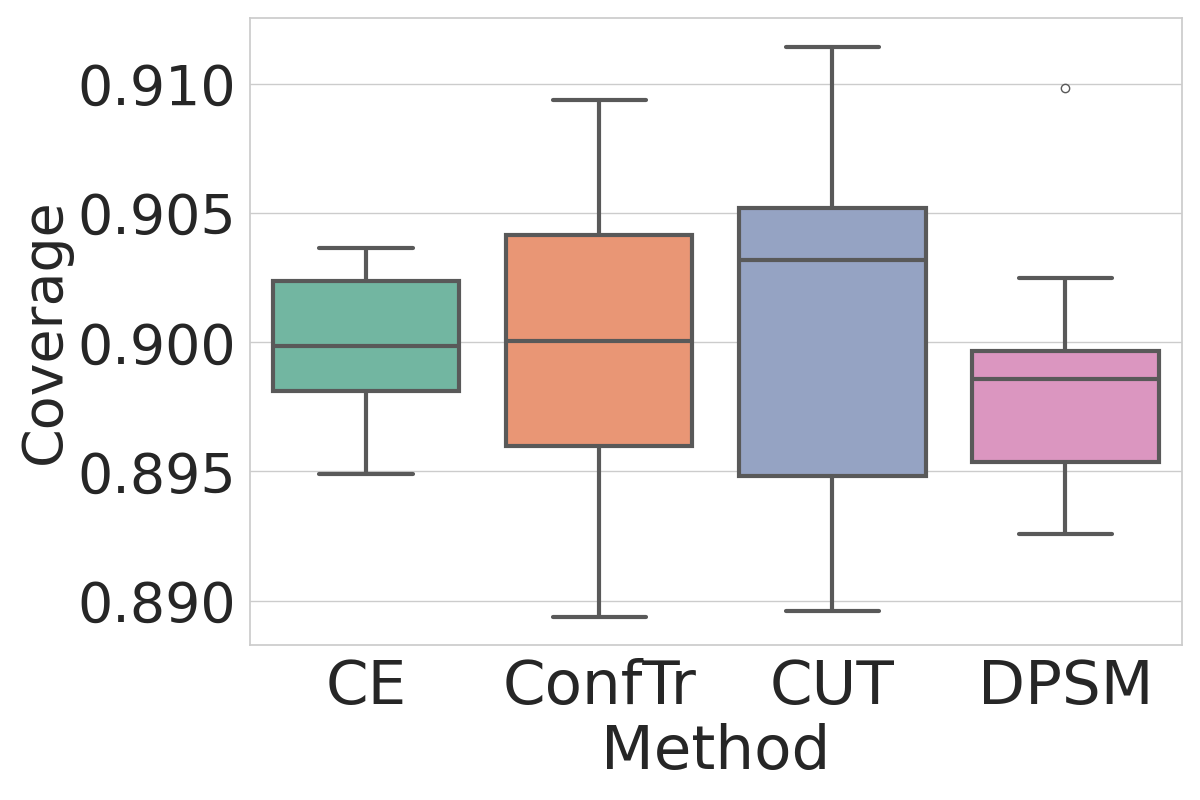}
    \\
    \includegraphics[width = \linewidth]{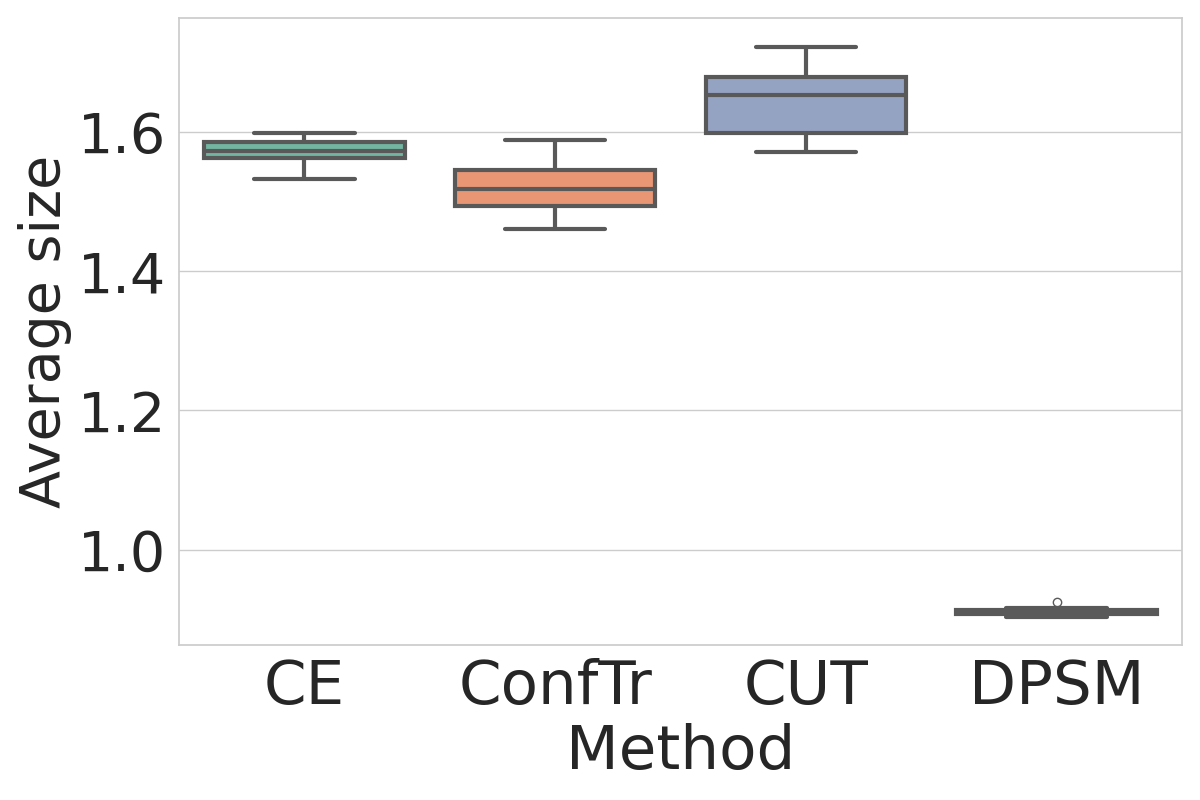}
    \end{minipage} 
    \begin{minipage}[t]{0.32\linewidth}
    \centering
    \textbf{(b)} CIFAR-100
    \includegraphics[width=\linewidth]{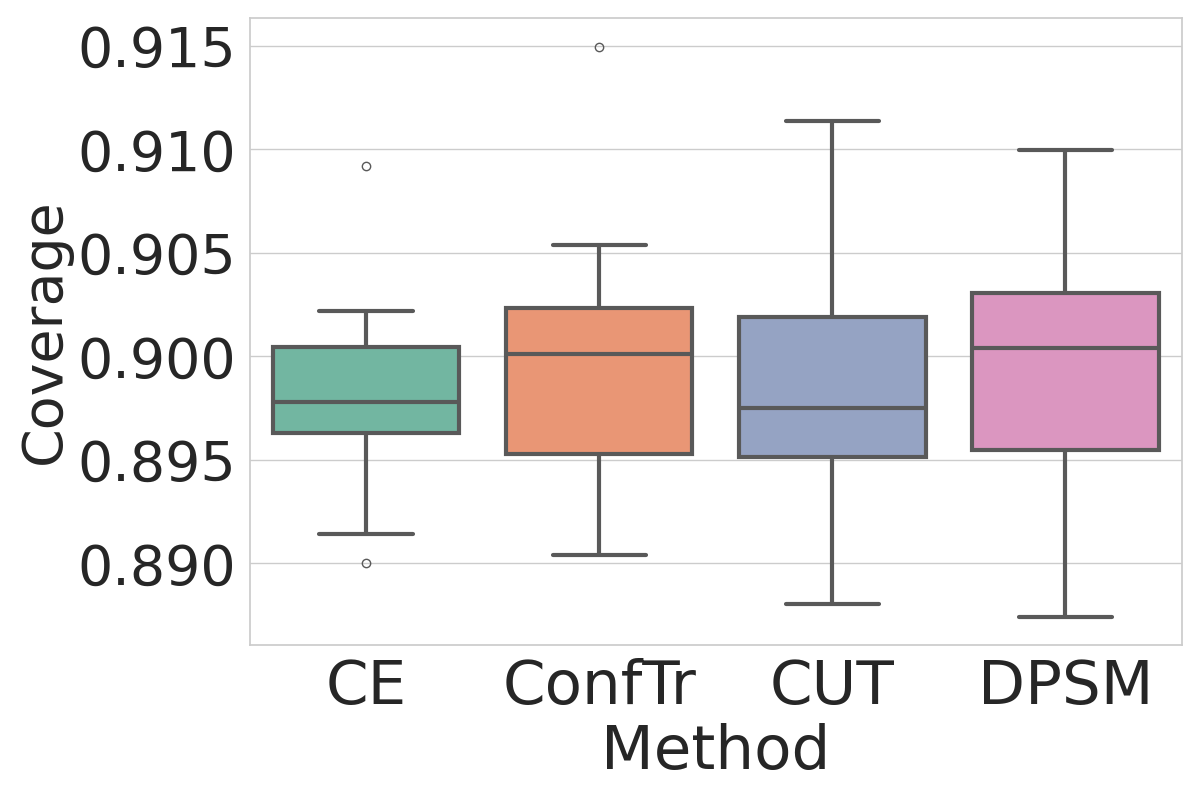}
    \\
    \includegraphics[width=\linewidth]{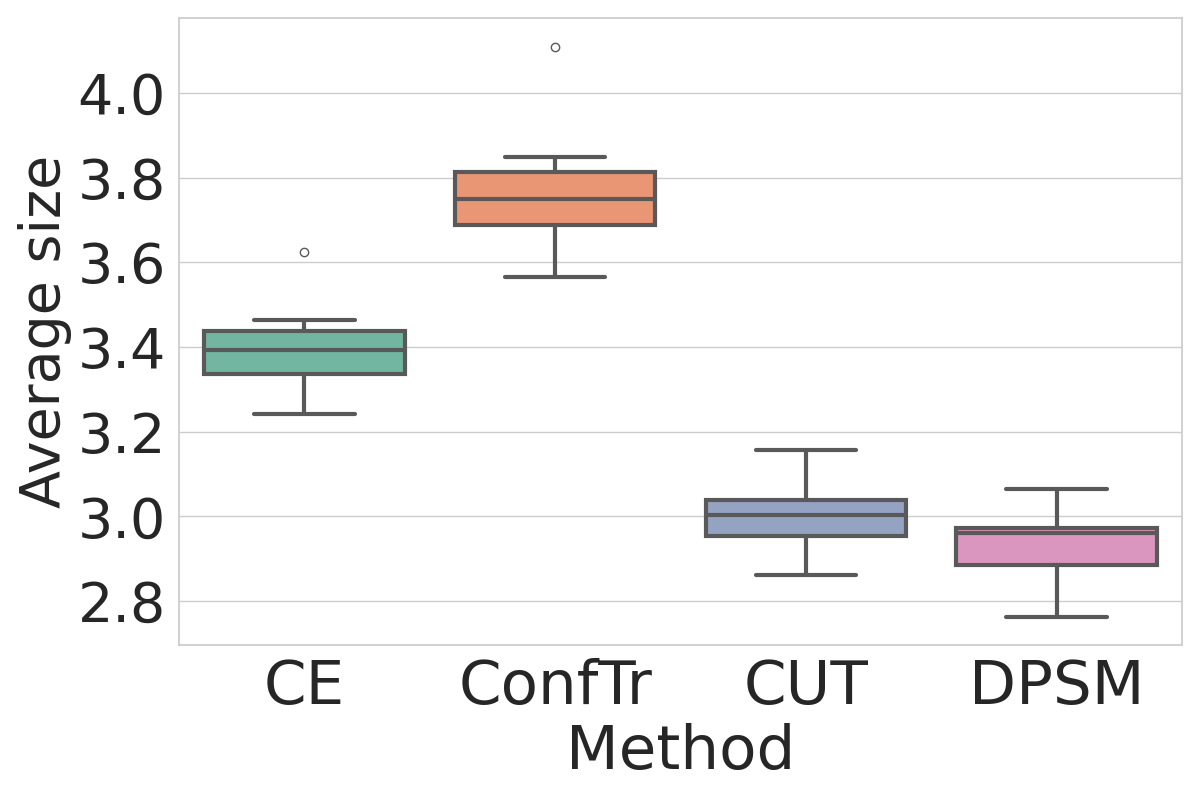}
    \end{minipage} 
    \begin{minipage}[t]{0.32\linewidth}
    \centering
    \textbf{(c)} iNaturalist
    \includegraphics[width=\linewidth]{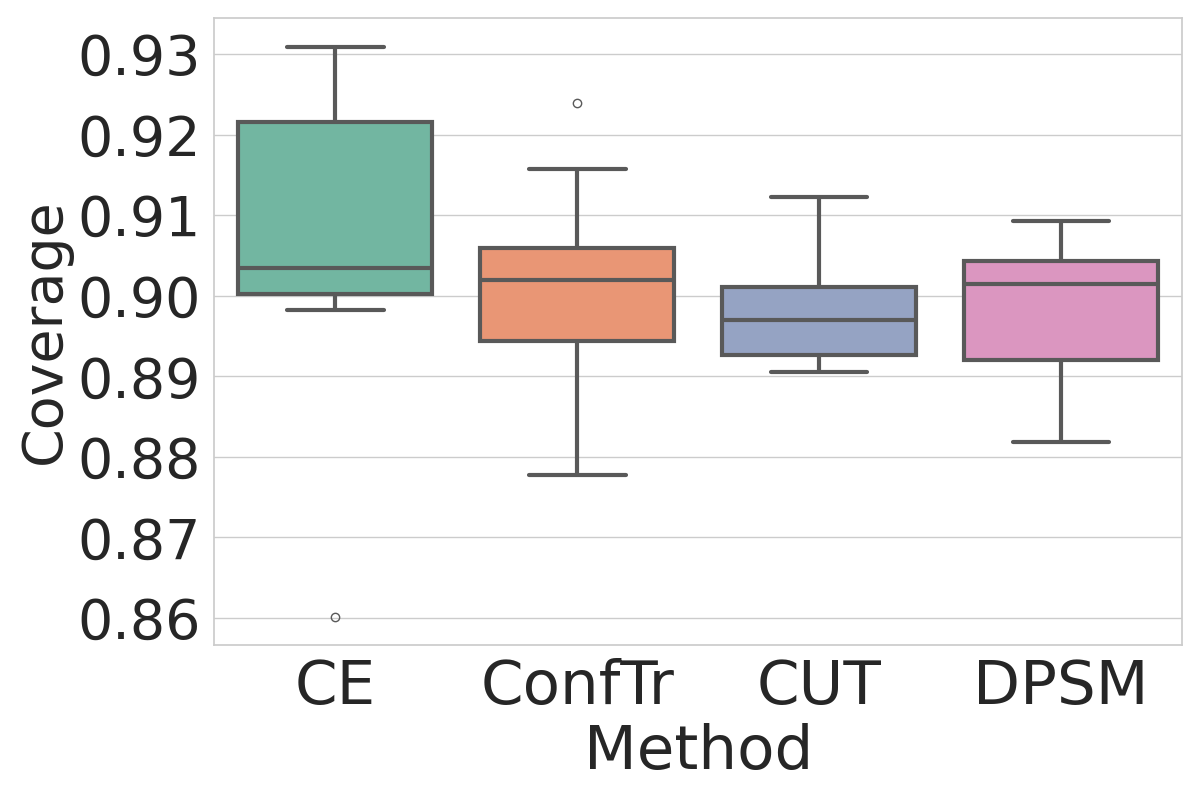}
    \\
    \includegraphics[width=\linewidth]{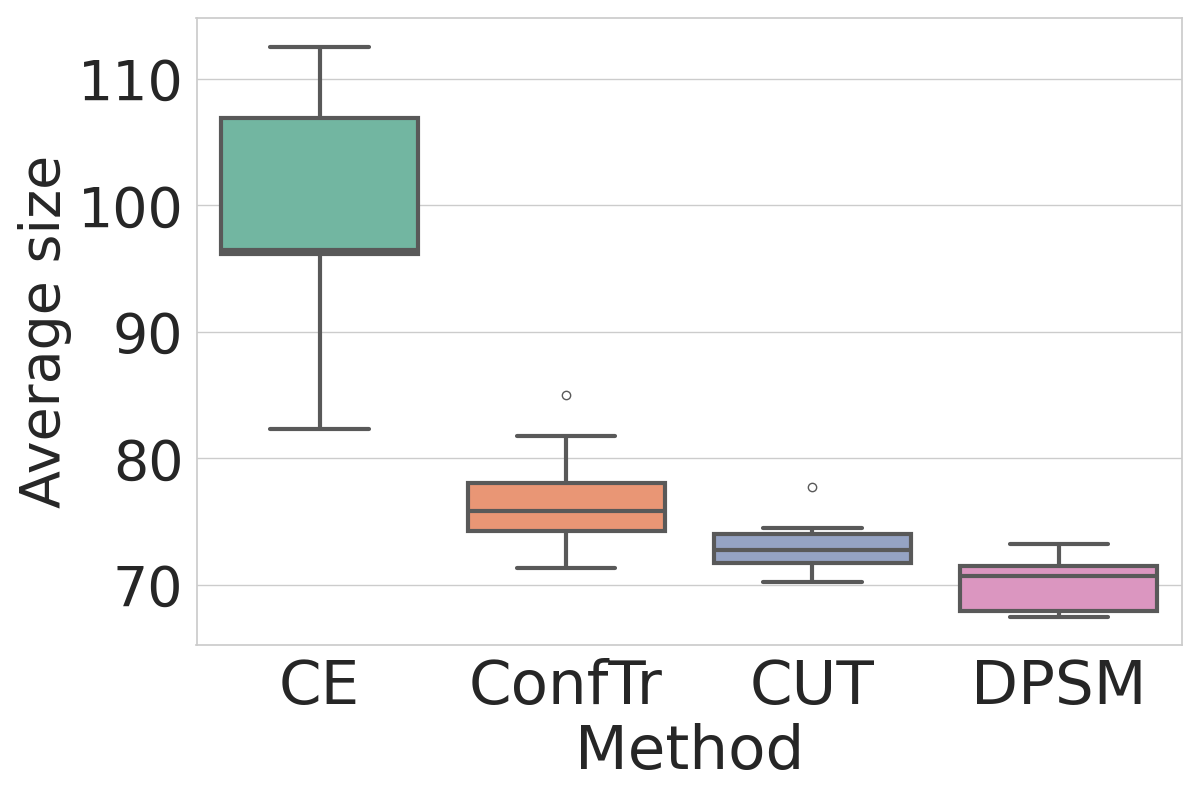}
    \end{minipage} 
    \vspace{-0.15in}
    \caption{
    \textbf{Box plots of coverage (Top row) and APSS (Bottom row)} of all methods using ResNet and HPS score.
    DPSM achieves significantly smaller prediction set size while maintaining the valid coverage.
    }
    \label{fig:results_overall_resnet}
\end{figure*}

\noindent \textbf{DPSM generates smaller prediction sets.}
Table \ref{tab:cvg_set_hps_train_hps_test} presents the set sizes and coverage rates for different methods using HPS score across training, calibration, and testing phases. 
DPSM outperforms all existing baselines with $20.32\%$ reduction in terms of prediction set size across all datasets.
Table \ref{tab:cvg_set_hps_train_aps_test} presents the set sizes and coverage rates for different methods using HPS for training and APS for calibrating. 
DPSM outperforms nearly all existing baselines with $20.61\%$ reduction in terms of prediction set size across all datasets, except $\uparrow$ 8.71\% increase on CIFAR-100 with DenseNet and $\uparrow$ 0.54\% increase on iNaturalist with ResNet in terms of prediction set size.
Table \ref{tab:cvg_set_hps_train_raps_test} presents the set sizes and coverage rates for different methods using HPS for training and RAPS for calibrating. 
DPSM outperforms existing baselines on several datasets, achieving a $19.04\%$ reduction in prediction set size across all datasets, except for CIFAR-100. On iNaturalist with ResNet, it shows a marginal $0.08\%$ increase in prediction set size.
Combined with above three tables, DPSM improve the predictive efficiency with $19.99\%$ reduction in term of prediction set size across all settings and all datasets.
We also visualize the coverage rate and APSS with confidence intervals of all methods using DenseNet and HPS score in Figure \ref{fig:results_overall_resnet}.
It clearly confirms that DPSM achieves significantly smaller prediction set size while maintaining the valid coverage.

\begin{figure}[!t]
    \centering
    \begin{minipage}[t]{0.24\linewidth}
    \centering
    \textbf{(a)} Upper loss
    \includegraphics[width=\linewidth]{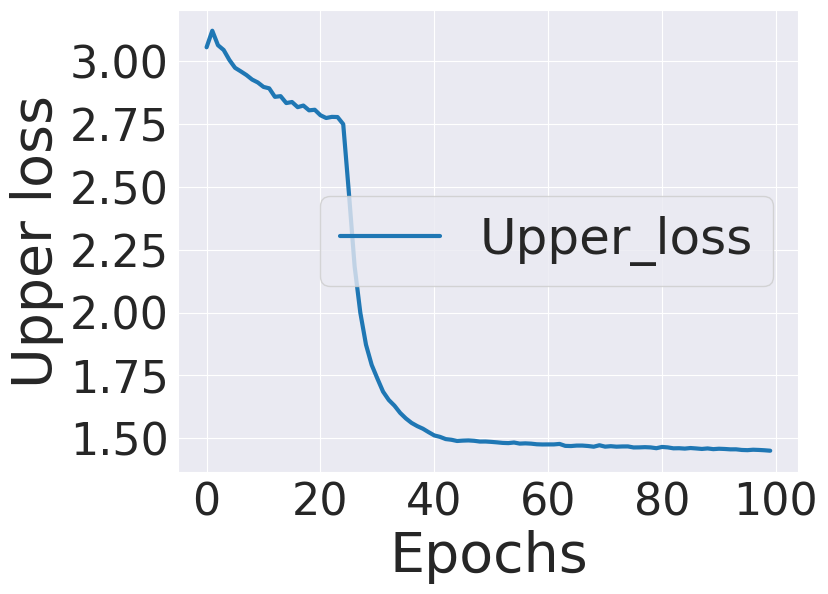}
    \end{minipage} 
    \begin{minipage}[t]{0.24\linewidth}
    \centering
    \textbf{(b)} Lower loss
    \includegraphics[width=\linewidth]{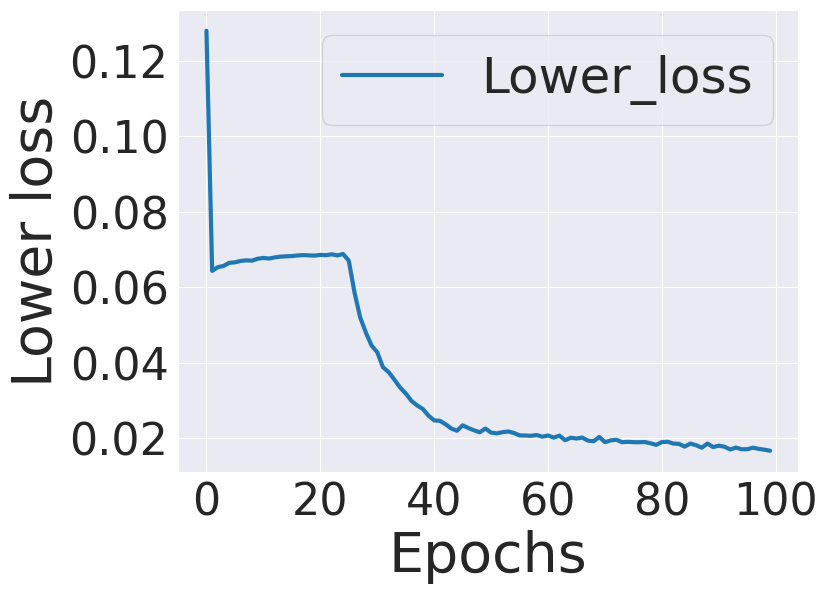}
    \end{minipage} 
    \begin{minipage}[t]{0.24\linewidth}
    \centering
    \textbf{(c)} Conformal loss optimization gap
    \includegraphics[width = \linewidth]{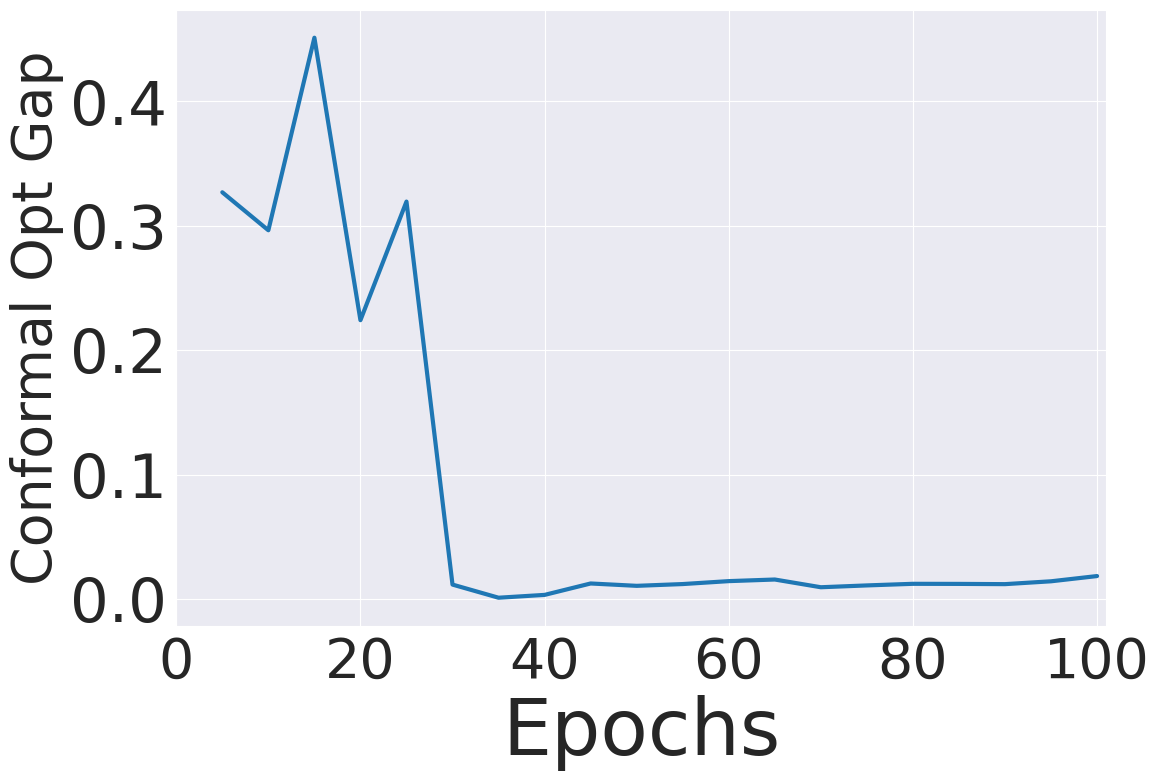}
    \end{minipage} 
    \begin{minipage}[t]{0.24\linewidth}
    \centering
    \textbf{(d)} QR loss optimization gap
    \includegraphics[width = \linewidth]{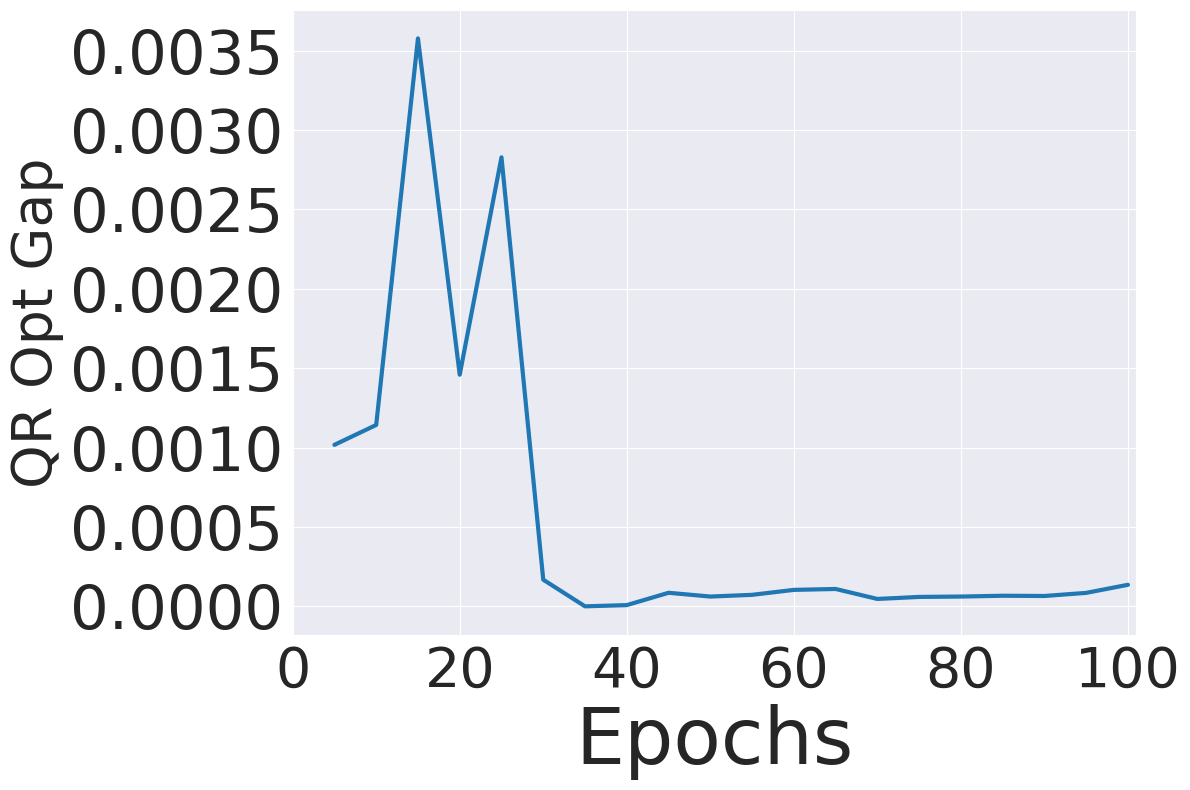}
    \end{minipage} 
    \vspace{-0.15in}
    \caption{
    \textbf{Justification experiments for the convergence of DPSM }on CIFAR-100 using ResNet and HPS score. 
   \textbf{(a)} Upper level loss (i.e., a combination of classification loss and conformal alignment loss);
    \textbf{(b)} Lower level loss (i.e., QR loss);
    \textbf{(c)} Optimization gap of conformal loss, defined as the difference between conformal losses using learned batch-level quantiles and dataset-level quantiles on the training set;
    \textbf{(d)} Optimization gap of the lower-level QR loss, defined similarly as the loss difference between learned batch-level quantiles and dataset-level quantiles.
    }
    \label{fig:results_convergence_appendix}
\end{figure}

\noindent \textbf{DPSM converges to stable error for bilevel optimization.} 
To further explore how DPSM effectively generates smaller prediction sets, we analyze the convergence of DPSM by plotting the loss of the upper level function (i.e., a combination of classification loss and conformal alignment loss) and the lower level function (i.e., QR loss) over training $100$ epochs with ResNet model.
Figure \ref{fig:results_convergence_appendix} (a) and (b) show the upper-level loss and lower-level loss over epochs, respectively. 
We also report the results of $40$-epoch training regime in Figure \ref{fig:results_main} and \ref{fig:results_appendix} for reference.
As shown, the upper-level loss of DPSM exhibits an initial increase during the first $2$ epochs, reaching the peak, then steadily decreases before stabilizing around epoch $35$.
In contrast, the lower-level loss decreases sharply within the first $2$ epochs, followed by a gradual reduction until convergence near the end of training. 
These results empirically demonstrate that DPSM effectively converges in terms of both upper and lower level training errors, validating its bilevel optimization approach.
To investigate how the learned quantiles influence the optimization error, we compute both conformal and QR losses using the learned quantiles and the optimal (dataset-level) quantiles. 
The corresponding optimization errors—defined as the loss differences between learned and optimal quantiles—are visualized in Figure \ref{fig:results_convergence_appendix} (c) and (d). 
Both errors converge to nearly $0$, indicating that the learned quantiles effectively approximate the optimal quantiles over training.

\noindent \textbf{DPSM estimates empirical quantiles with small error.} 
To compare the precision of empirical quantiles estimation in ConfTr and DPSM,
we plot the estimation error between $\widehat Q^n_f$ (quantiles evaluated on the whole training dataset) and $\widehat q_f$ (quantiles evaluated in ConfTr or learned in DPSM on mini-batches) with ResNet. 
Figure \ref{fig:results_bound_appendix} (a) plots this estimation error over training epochs. 
For the first $25$ epochs, the estimation errors for DPSM are significantly larger compared to ConfTr. 
However, as the training progresses, the estimation errors for DPSM decrease rapidly, converging close to $0$ after epoch $32$. 
This result verifies the theoretical result for smaller estimation error in learning bound analysis from Theorem \ref{theorem:learning_bound_DPSM}.
Furthermore, the rapid reduction in estimation error also reflects the convergence of the lower loss (i.e., QR loss), highlighting the effectiveness of DPSM in accurately estimating quantiles.

\begin{figure}[!ht]
    \centering
    \begin{minipage}[t]{0.24\linewidth}
    \centering
    \textbf{(a)} Upper loss
    \end{minipage} 
    \begin{minipage}[t]{0.24\linewidth}
    \centering
    \textbf{(b)} Lower loss
    \end{minipage} 
    \hfill
    \begin{minipage}[t]{0.24\linewidth}
    \centering
    \textbf{(c)} Est. error of quantiles
    \end{minipage} 
    \begin{minipage}[t]{0.24\linewidth}
    \centering
    \textbf{(d)} Average soft set size
    \end{minipage} 
    \begin{minipage}[t]{0.24\linewidth}  
    \centering 
        \includegraphics[width=\linewidth]{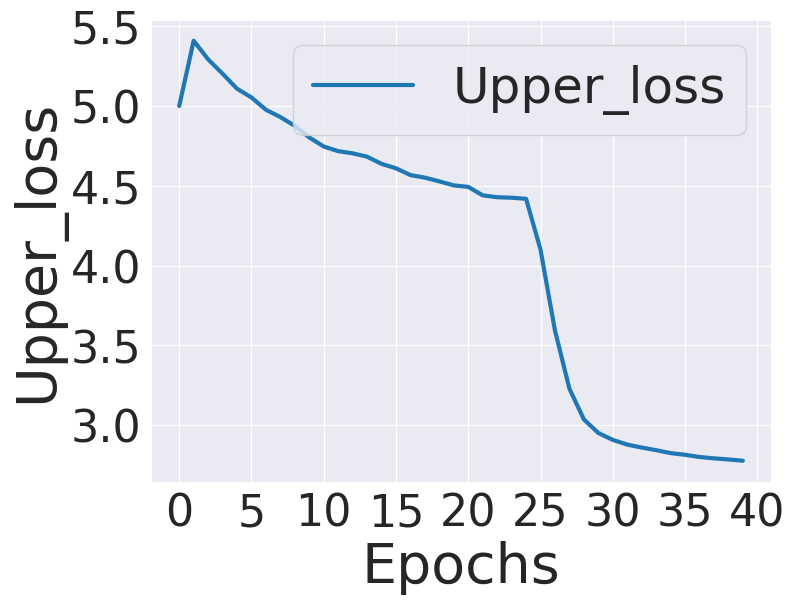}
    \end{minipage}
    \begin{minipage}[t]{0.24\linewidth}
        \centering 
        \includegraphics[width=\linewidth]{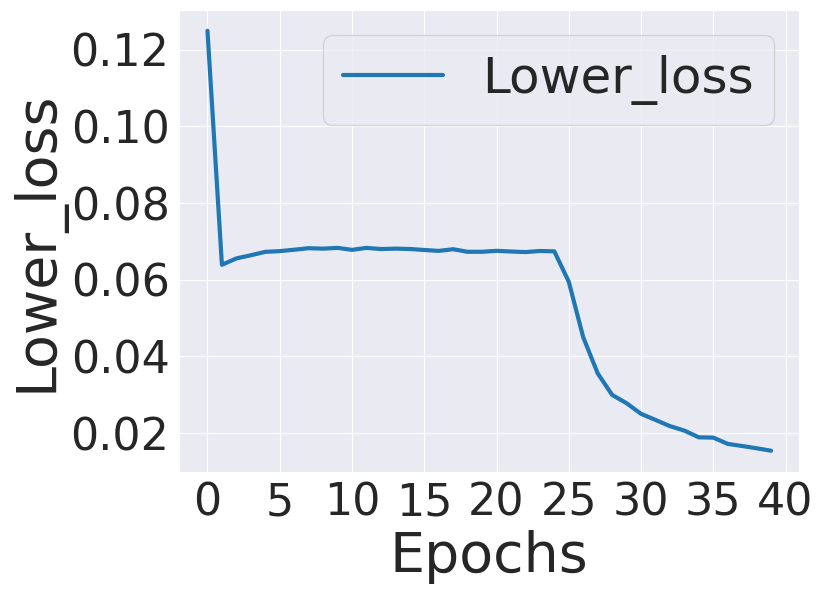}
    \end{minipage}
    \begin{minipage}[t]{0.24\linewidth}
     \centering   
     \includegraphics[width = \linewidth]{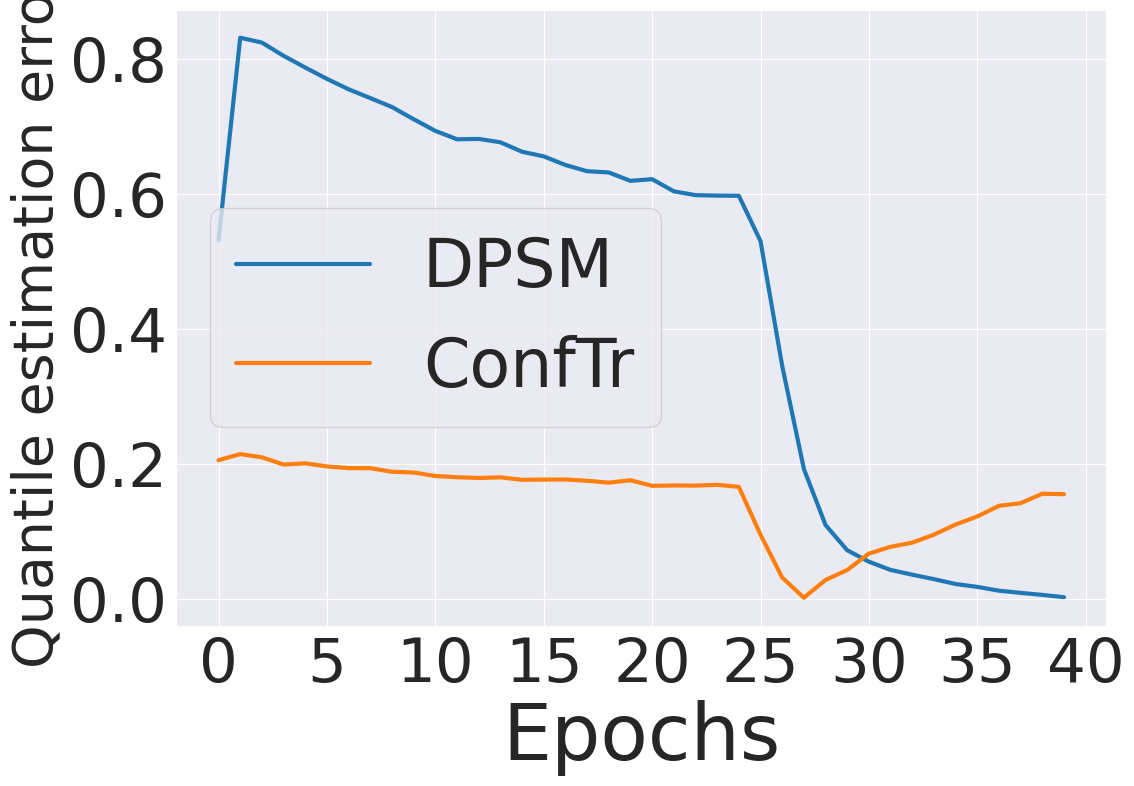}
     \end{minipage}
    \begin{minipage}[t]{0.24\linewidth}
    \centering
    \includegraphics[width=\linewidth]{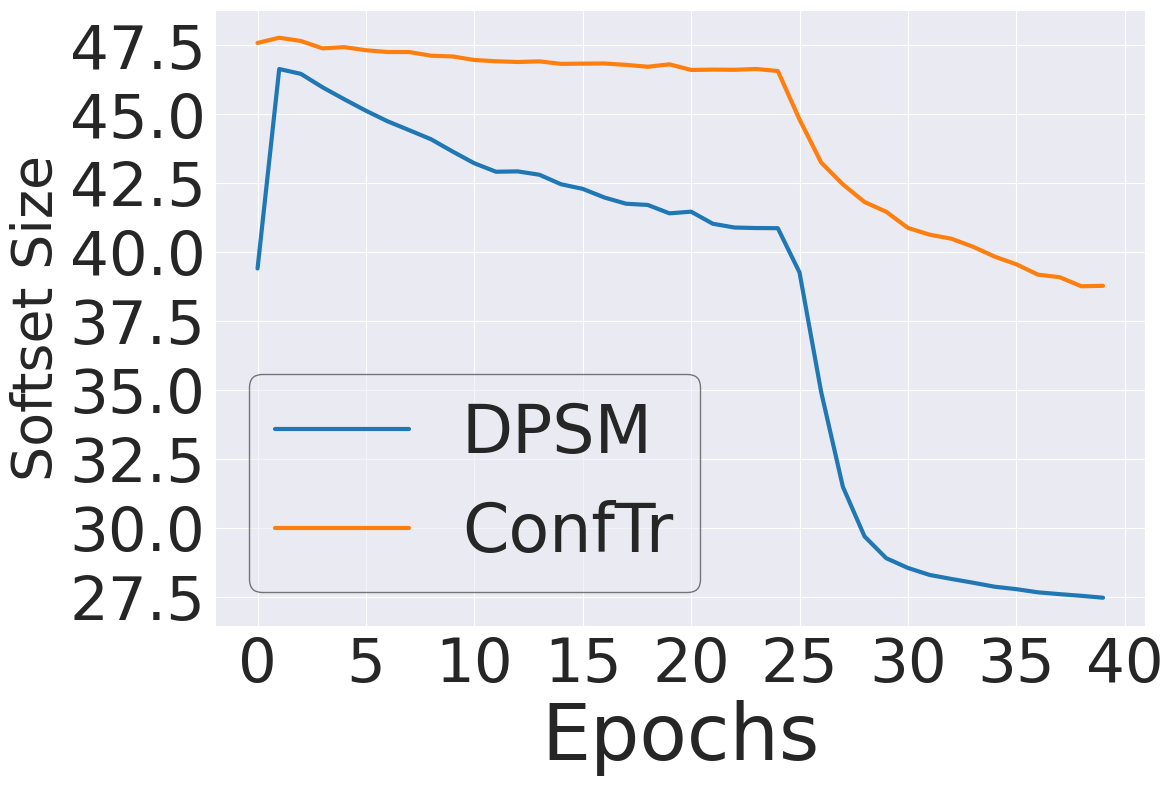}
    \end{minipage}
    \caption{\textbf{Justification experiments for effectiveness of DPSM} trained with 40 epochs on CIFAR-100 using DenseNet and HPS score. 
   \textbf{(a)} Upper level loss (i.e., a combination of
    classification loss and conformal alignment loss) in DPSM;
    \textbf{(b)} Lower level loss (i.e., QR loss) in DPSM;
    \textbf{(c)} Estimation error between the $\widehat Q^n_f$ (quantiles evaluated on the whole training data) and  $\widehat q_f$ (quantiles evaluated in ConfTr or learned in DPSM on mini-batches);
    and \textbf{(d)} Average soft set size of DPSM and ConfTr (using Sigmoid function).}
    \label{fig:results_main}
\end{figure}

\begin{figure}[!ht]
    \centering
    \begin{minipage}[t]{0.24\linewidth}
    \centering
    \textbf{(a)} Upper loss
    \end{minipage} 
    \begin{minipage}[t]{0.24\linewidth}
    \centering
    \textbf{(b)} Lower loss
    \end{minipage} 
    \begin{minipage}[t]{0.24\linewidth}
    \centering
    \textbf{(c)} Estimation error of quantiles
    \end{minipage} 
    \begin{minipage}[t]{0.24\linewidth}
    \centering
    \textbf{(d)} Average soft set size
    \end{minipage} 
     \hfill
    \begin{minipage}[t]{0.24\linewidth}  
    \centering 
        \includegraphics[width=\linewidth]{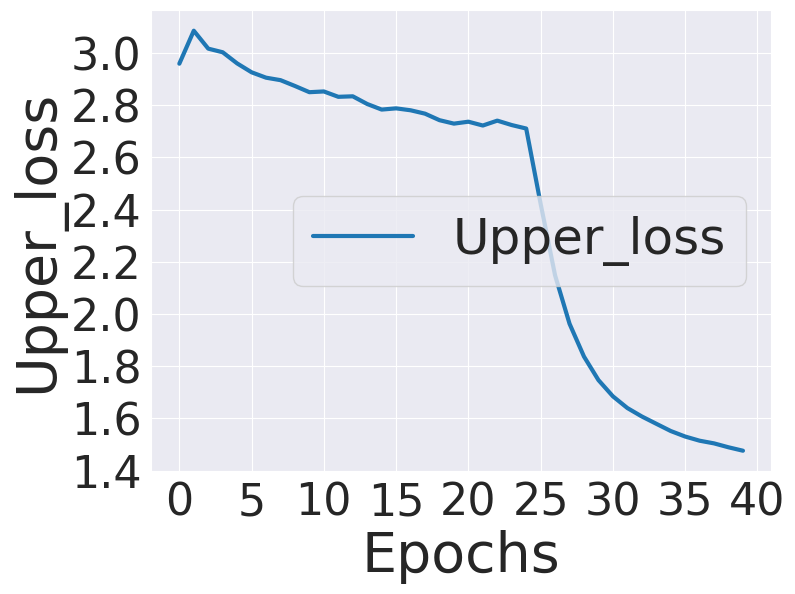}
    \end{minipage}
    \begin{minipage}[t]{0.24\linewidth}
        \centering 
        \includegraphics[width=\linewidth]{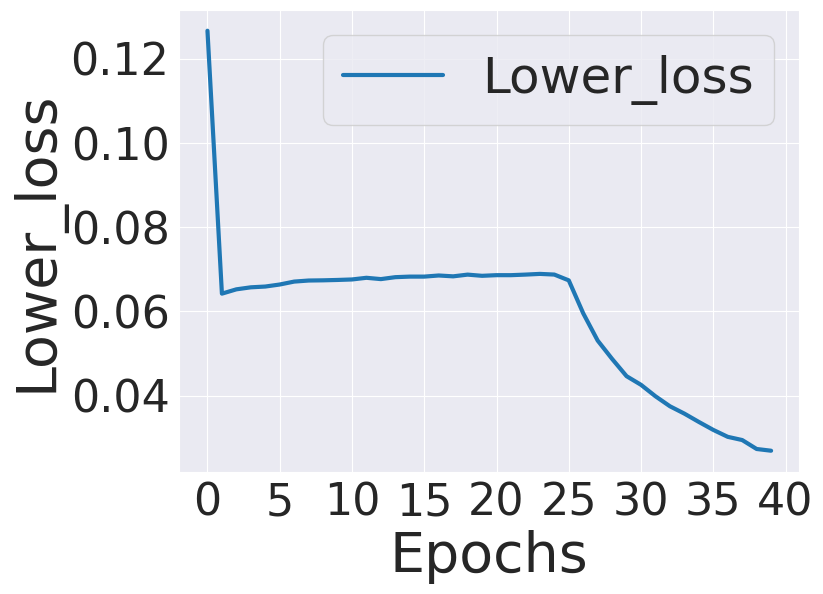}
    \end{minipage}
    \begin{minipage}[t]{0.24\linewidth}
     \centering   
     \includegraphics[width = \linewidth]{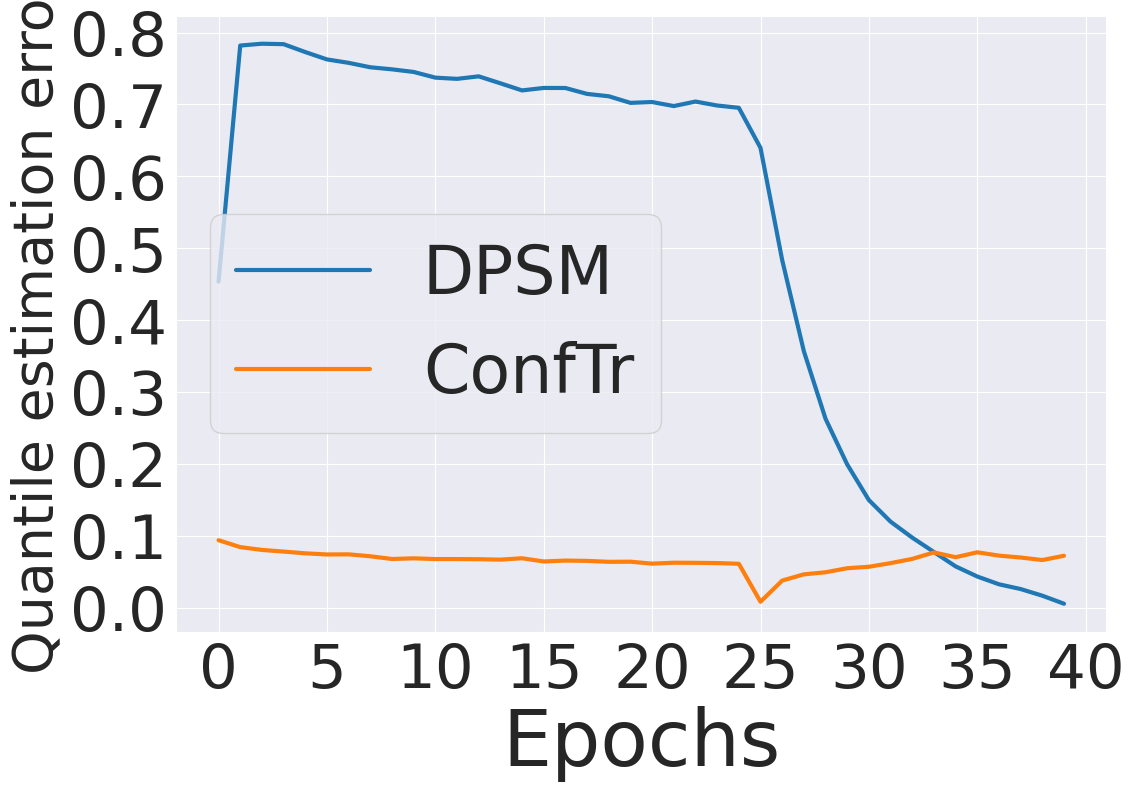}
     \end{minipage}
    \begin{minipage}[t]{0.24\linewidth}
    \centering
    \includegraphics[width=\linewidth]{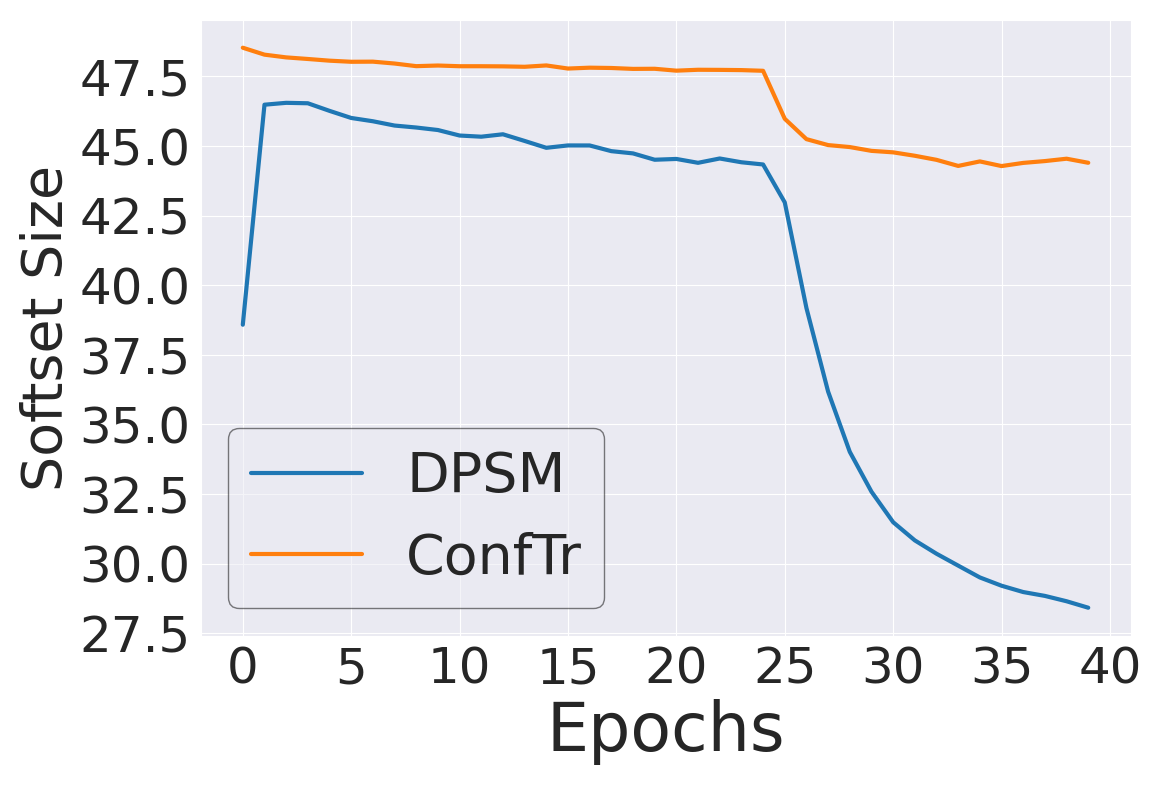}
    \end{minipage}
    \caption{\textbf{Justification experiments for effectiveness of DPSM} trained with 40 epochs on CIFAR-100 using ResNet and HPS score. 
   \textbf{(a)} Upper level loss (i.e., a combination of
    classification loss and conformal alignment loss) in DPSM;
    \textbf{(b)} Lower level loss (i.e., QR loss) in DPSM;
    \textbf{(c)}: Estimation error between the $\widehat Q^n_f$ (quantiles evaluated on the whole training data) and  $\widehat q_f$ (quantiles evaluated in ConfTr or learned in DPSM on mini-batches);
    and \textbf{(d)}: Average soft set size of DPSM and ConfTr.}
    \label{fig:results_appendix}
\end{figure}

\begin{figure*}[!t]
    \centering
    \begin{minipage}[t]{0.33\linewidth}
    \centering
    \textbf{(a)} Estimation error of quantiles
    \includegraphics[width = \linewidth]{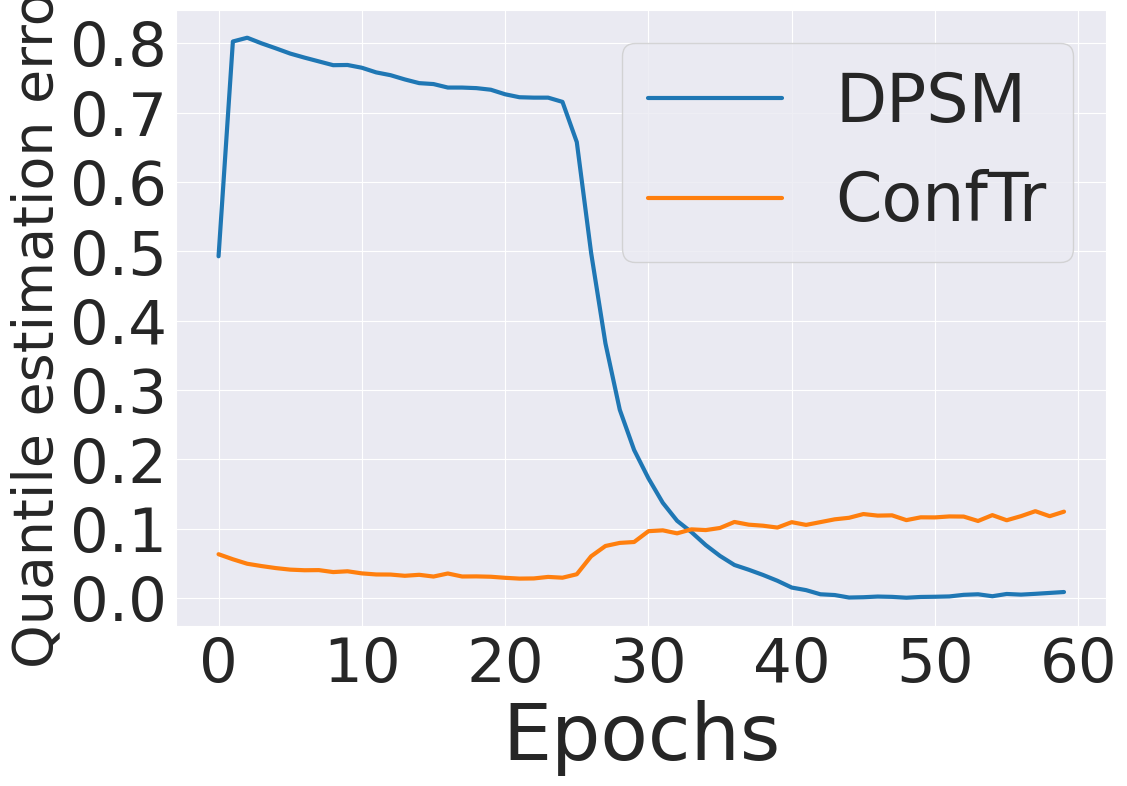}
    \end{minipage} 
    \begin{minipage}[t]{0.33\linewidth}
    \centering
    \textbf{(b)} Average soft set size
    \includegraphics[width=\linewidth]{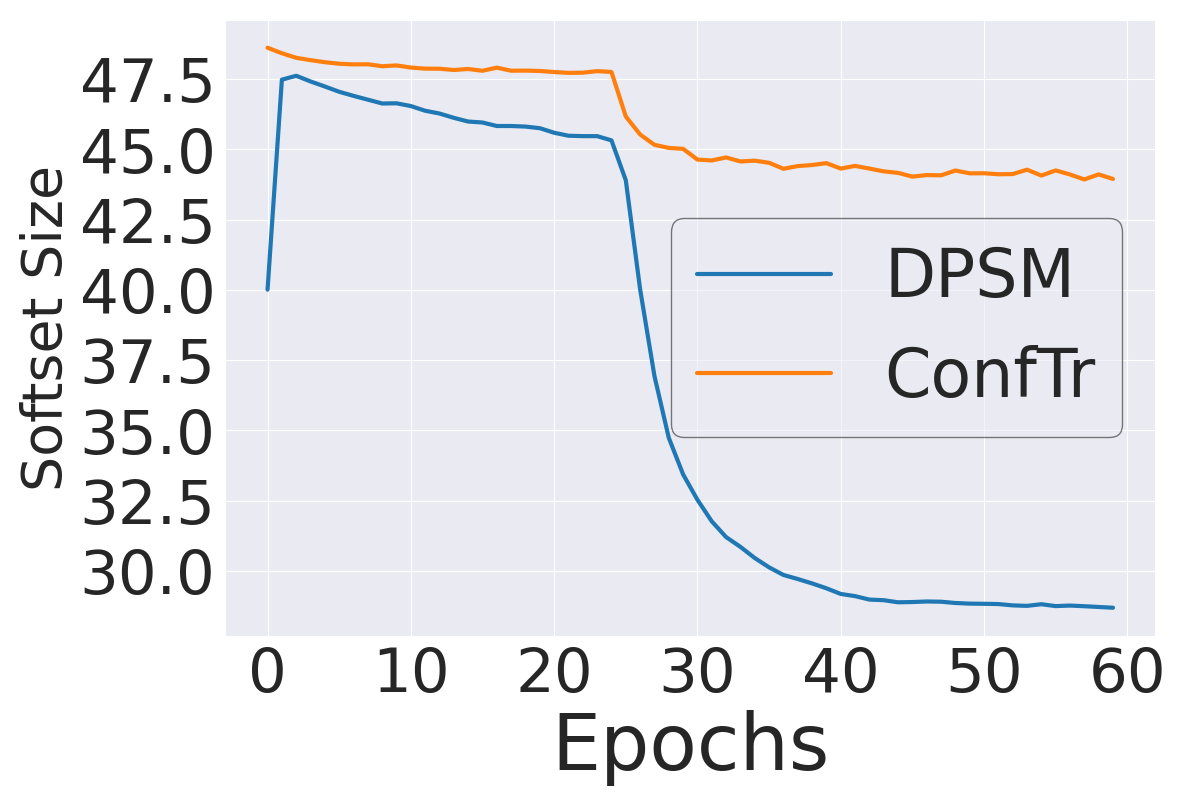}
    \end{minipage} 
    \begin{minipage}[t]{0.33\linewidth}
    \centering
    \textbf{(c)} Learning bound approximation
    \includegraphics[width=\linewidth]{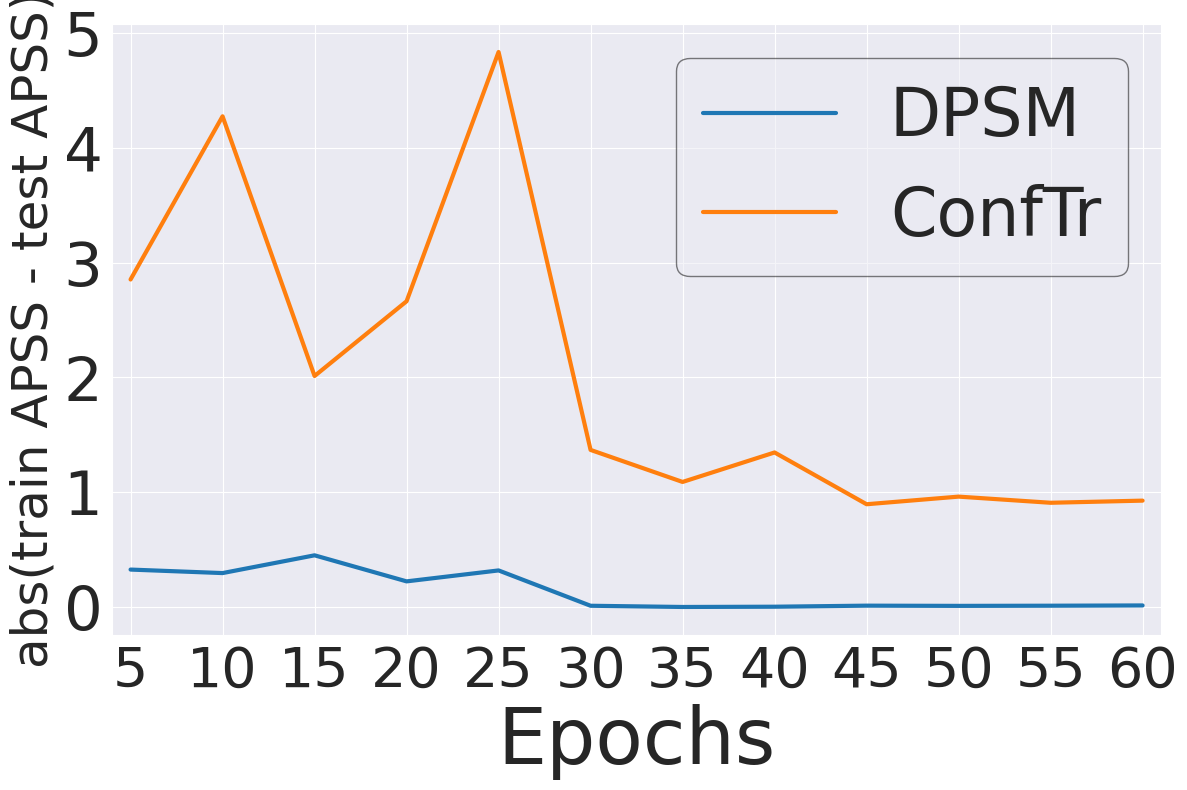}
    \end{minipage} 
    \vspace{-0.2in}
    \caption{
    \textbf{Justification experiments for the learning bound of DPSM }on CIFAR-100 using ResNet and HPS score. 
    \textbf{(a)} Estimation error between the $\widehat Q^n_f$ (dataset-level quantiles on training data) and $\widehat q_f$ (batch-level quantiles evaluated in ConfTr or learned in DPSM);
    \textbf{(b)} Average soft set size of DPSM and ConfTr (using Sigmoid function);
    \textbf{(c)} Approximated learning error comparison between DPSM and ConfTr, measured by their gaps between the training and testing APSS.
    }
    \label{fig:results_bound_appendix}
    \vspace{-3.0ex}
\end{figure*}


\noindent \textbf{Learning bound of DPSM is much tighter than ConfTr.}
To approximately compare the learning bounds of DPSM and ConfTr,
we compare the conformal alignment losses of ConfTr and DPSM during training in terms of the average soft set size, as shown in Figure \ref{fig:results_bound_appendix} (b).
The soft set size of DPSM is consistently smaller than that of ConfTr during training. Combining the empirical results of the smaller estimation error of quantiles from Figure \ref{fig:results_bound_appendix} (a), we can conclude that the learning bound of DPSM is much tighter than the learning bound of ConfTr, providing the empirical verification of Theorem \ref{theorem:learning_bound_DPSM} and \ref{theorem:learning_bounds_SA}.
Although learning bound cannot be empirically computed, we approximate it using a common strategy in ML literature \cite{yuan2019stagewise,yang2021exact}, which estimates generalization error by the absolute gap between training and test errors. 
For CP, we use APSS evaluated on train and test sets to approximate the learning errors. 
Specifically, for DPSM, at each iteration, we:
(i) compute APSS on the training set using the learned quantiles as thresholds. 
It includes optimization error since the learned quantiles are not optimal (true dataset-level quantiles);
(ii) compute the APSS on the testing set using the dataset-level quantiles as thresholds.
The gap between these two APSS values is employed as an approximation of the learning bound.
We apply the same strategy to the SA-based ConfTr, 
where the training APSS is computed using the quantiles evaluated on mini-batches from the training data, 
and the test APSS is computed using the dataset-level quantiles from the test data.
This comparison is shown in Figure~\ref{fig:results_bound_appendix} (c), which demonstrates that the approximated learning error is improved by DPSM.

\begin{figure}[ht]
    \centering
    \begin{minipage}[t]{0.48\linewidth}
    \centering
    \textbf{(a)} Bi-Lipschitz continuity of conformity score 
    \end{minipage} 
    \begin{minipage}[t]{0.48\linewidth}
    \centering
    \textbf{(b)} Strongly concavity of conformal loss
    \end{minipage} 
    \hfill
    \begin{minipage}[t]{0.48\linewidth}  
    \centering
    \includegraphics[width=.7\linewidth]{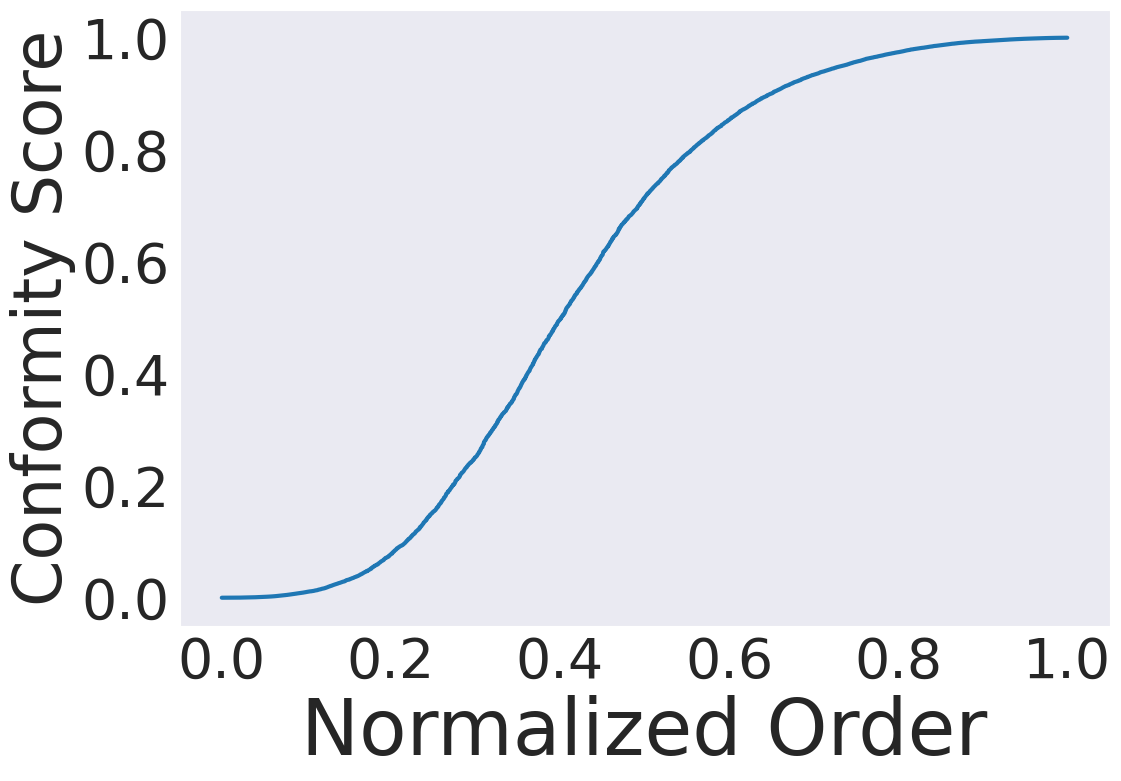}
    \end{minipage}
    \begin{minipage}[t]{0.49\linewidth}  
    \centering
    \includegraphics[width=.7\linewidth]{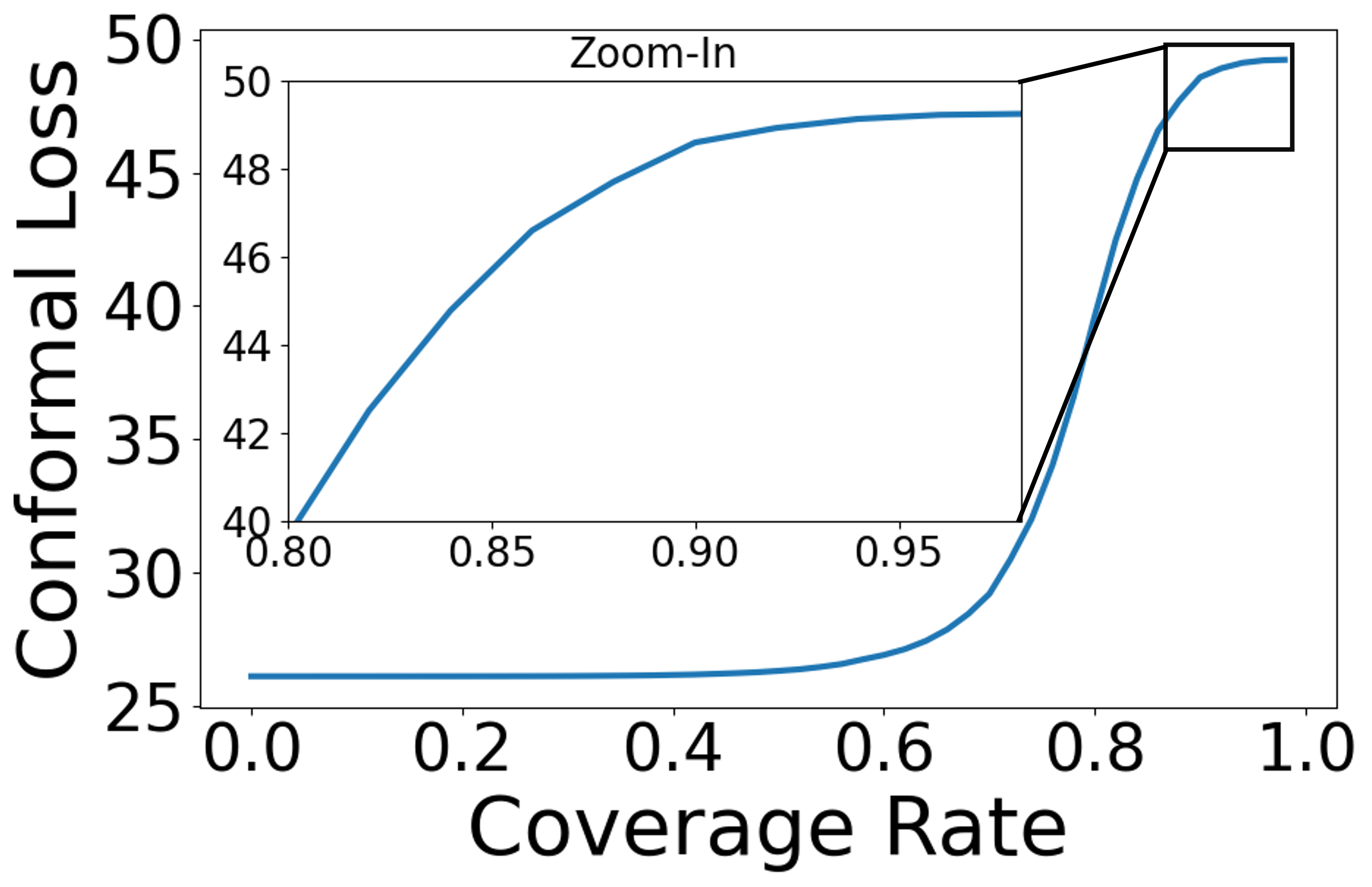}
    \end{minipage}
    \caption{\textbf{Verification studies} on CIFAR-100 with ResNet model using HPS scoring function on calibration dataset.
    \textbf{(a)}: HPS scores over corresponding normalized order produced by ConfTr. 
    The x-axis is the normalized order, the y-axis is the corresponding conformity score;
    \textbf{(b)}: The soft set size measure of ConfTr. The input coverage rate is from $[0.02, 0.98]$ with $0.02$ range with its zoom-in version where $1-\alpha$ is close to target coverage $0.9$;
    When coverage rate is close to $0.9$, the curve for the soft set size exhibits a concave shape.
    }
    \label{fig:sodt_set_size_comparisons_appendix}
\end{figure}

\noindent \textbf{Assumption \ref{assumption:bi_lipschitz} is empirically valid.} 
Figure \ref{fig:sodt_set_size_comparisons_appendix} (a) illustrates the conformity scores plotted against their corresponding normalized order. 
The x-axis represents the normalized order, while the y-axis represents conformity scores. 
From this figure, it is clear that the curve does not remain near the x-axis or y-axis, indicating that the gradient of conformity scores with respect to normalized index is both upper and lower bounded. 
This observation empirically supports the validity of Assumption \ref{assumption:bi_lipschitz}.

\noindent \textbf{Assumption \ref{assumption:straongly_concave} is empirically valid.}
Figure \ref{fig:sodt_set_size_comparisons_appendix} (b) visualizes the soft set size of ConfTr, with input as coverage rate $\in [0.02, 0.98]$ with range $0.02$.
When coverage rate approaches $0.9$, the curves of all methods exhibit a concave shape (zoom-in version also shown), providing empirical verification for Assumption \ref{assumption:straongly_concave}.

\subsection{Additional Experiments for Conditional Coverage }
\label{appendix:subsec:additional_exps_conditional}

\begin{table*}[!t]
\centering
\caption{
\textbf{The WSC ($\uparrow$ better), SSCV ($\downarrow$ better) and CovGap ($\downarrow$ better)} of all methods on CIFAR-100 with HPS score: 
The best results are in \textbf{bold}.
These results show that DPSM achieves the best performance for class-conditional coverage (the smallest CovGap). 
For size-stratified coverage (SSCV), DPSM has a worse measure compared with CE and CUT, but is better than ConfTr.
For WSC, CUT and DPSM have comparable performance.
}
\label{tab:cvg_set_hps_conditional}
\resizebox{\textwidth}{!}{
\begin{NiceTabular}{@{}c|cccc|cccc@{}}
\toprule
\multirow{2}{*}{Measures} & \multicolumn{4}{c|}{DenseNet} & \multicolumn{4}{c }{ResNet} \\ 
\cmidrule(lr){2-5} \cmidrule(lr){6-9}
& CE & CUT & ConfTr & DPSM & CE & CUT & ConfTr & DPSM  \\ 
\midrule
WSC 
& 0.88 $\pm$ 0.016 & \textbf{0.90 $\pm$ 0.022} & 0.88 $\pm$ 0.018 & 0.89 $\pm$ 0.011
& 0.88 $\pm$ 0.012  & 0.88 $\pm$ 0.020 & 0.88 $\pm$ 0.020  & \textbf{0.89 $\pm$ 0.019}
\\ 
SSCV  
& 0.12 $\pm$ 0.024 & \textbf{0.09 $\pm$ 0.019} & 0.21 $\pm$ 0.061 & 0.17 $\pm$ 0.034
& \textbf{0.09 $\pm$ 0.018} & 0.11 $\pm$ 0.017 & 0.14 $\pm$ 0.022 & 0.12 $\pm$ 0.019
\\
CovGap  
& 4.54 $\pm$ 0.49 & 5.20 $\pm$ 0.29 & 4.56 $\pm$ 0.28 & \textbf{4.43 $\pm$ 0.41}
& 4.71 $\pm$ 0.38 & 4.71 $\pm$ 0.32 & 4.70 $\pm$ 0.34 & \textbf{4.69 $\pm$ 0.26}
\\ 
\bottomrule
\end{NiceTabular}
}
\end{table*}

\begin{figure*}[!t]
    \centering
    \begin{minipage}[t]{0.32\linewidth}
    \centering
    \textbf{(a)} Class-wise coverage and size (DPSM vs CE)
    \includegraphics[width = \linewidth]{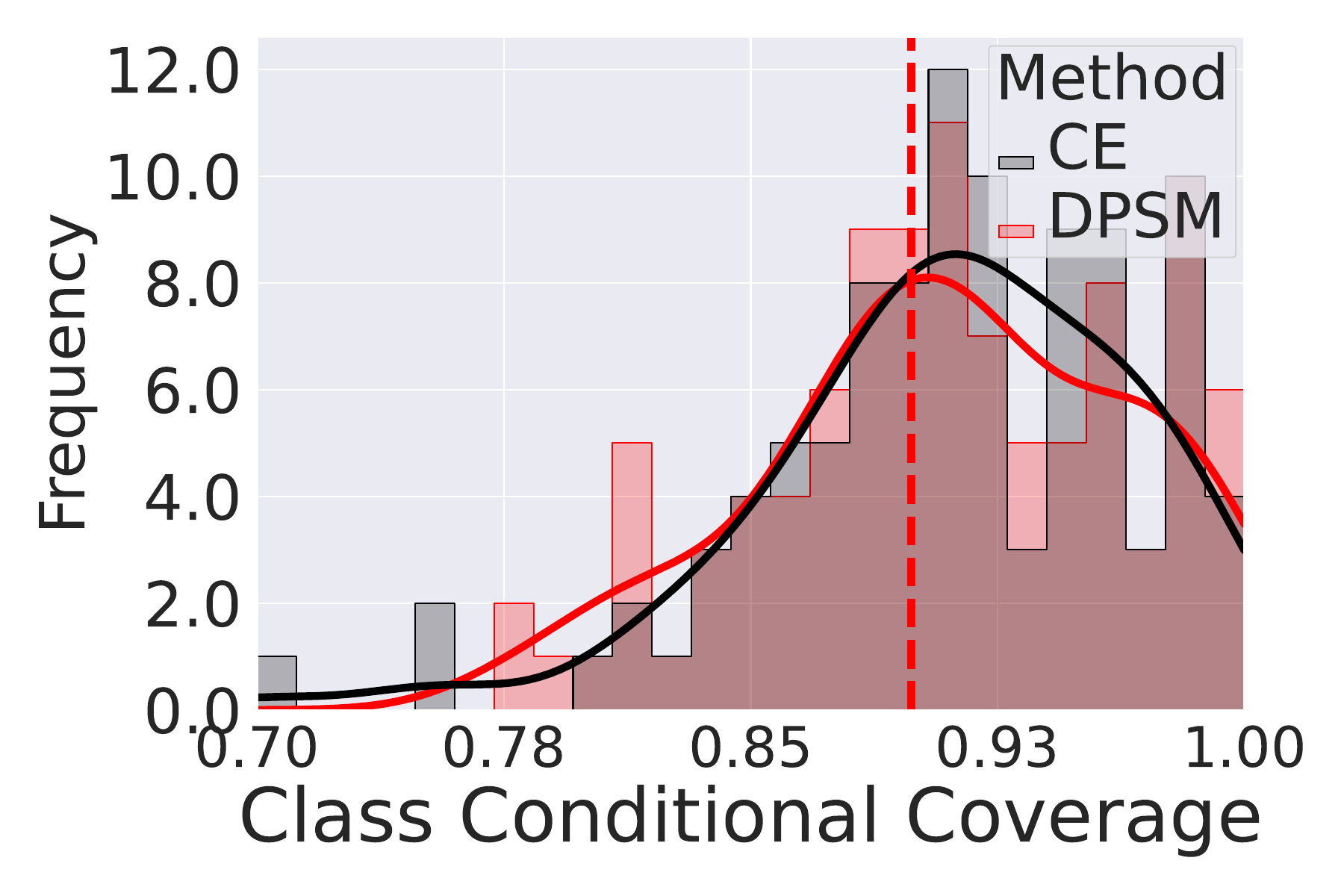}
    \\
    \includegraphics[width = \linewidth]{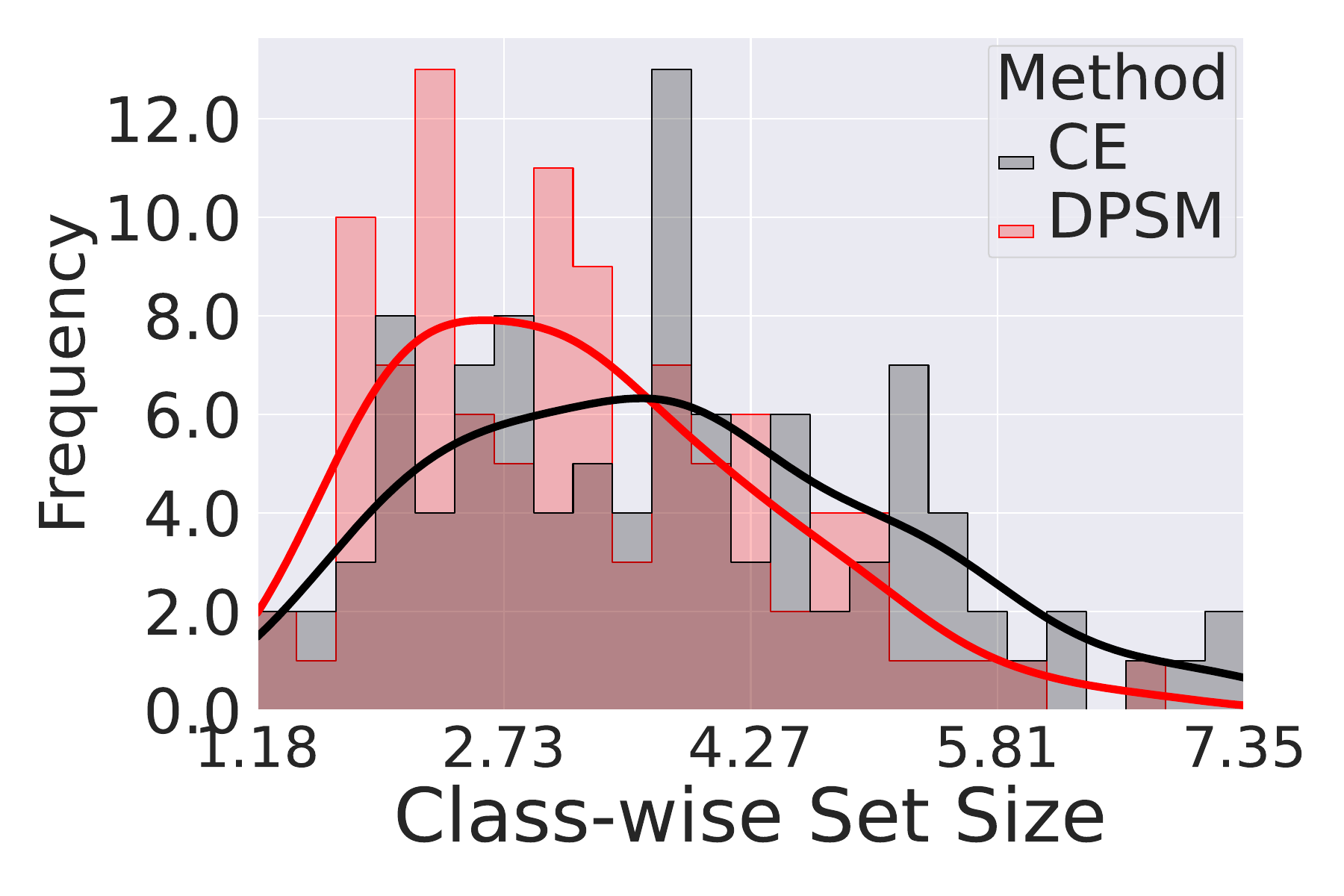}
    \end{minipage} 
    \begin{minipage}[t]{0.32\linewidth}
    \centering
    \textbf{(b)} Class-wise coverage and size (DPSM vs CUT)
    \includegraphics[width=\linewidth]{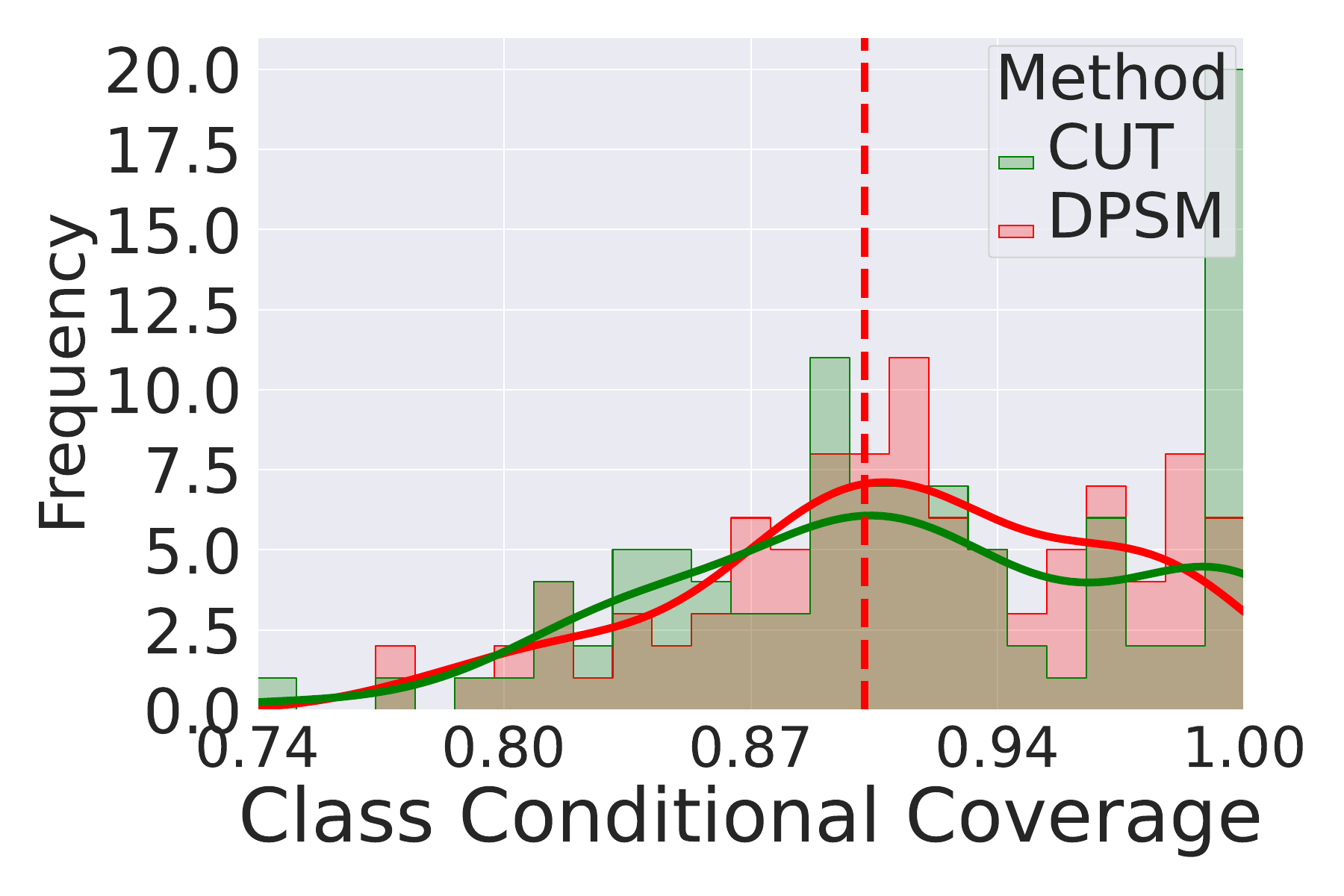}
    \\
    \includegraphics[width=\linewidth]{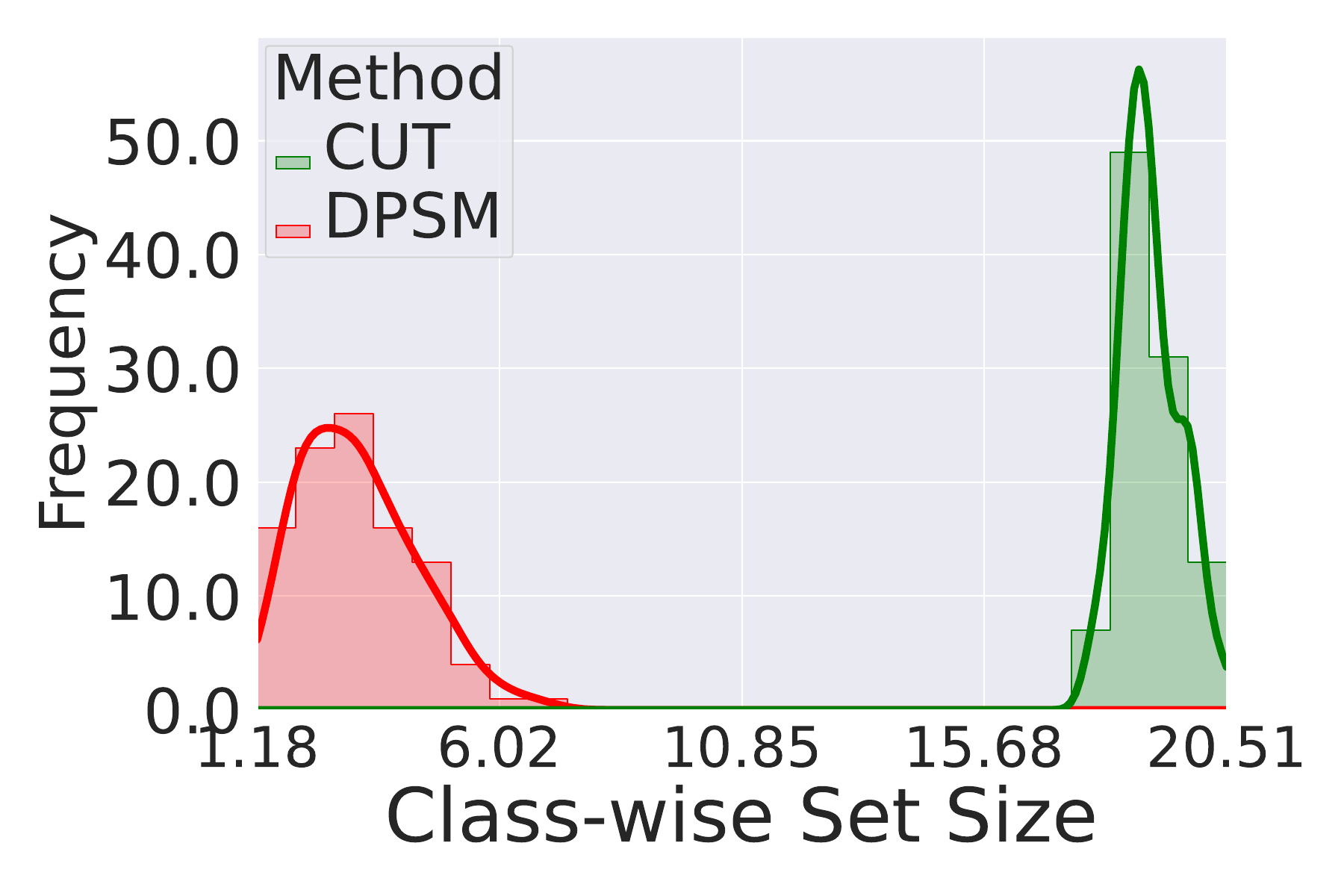}
    \end{minipage} 
    \begin{minipage}[t]{0.32\linewidth}
    \centering
    \textbf{(c)} Class-wise coverage and size (DPSM vs Conftr)
    \includegraphics[width=\linewidth]{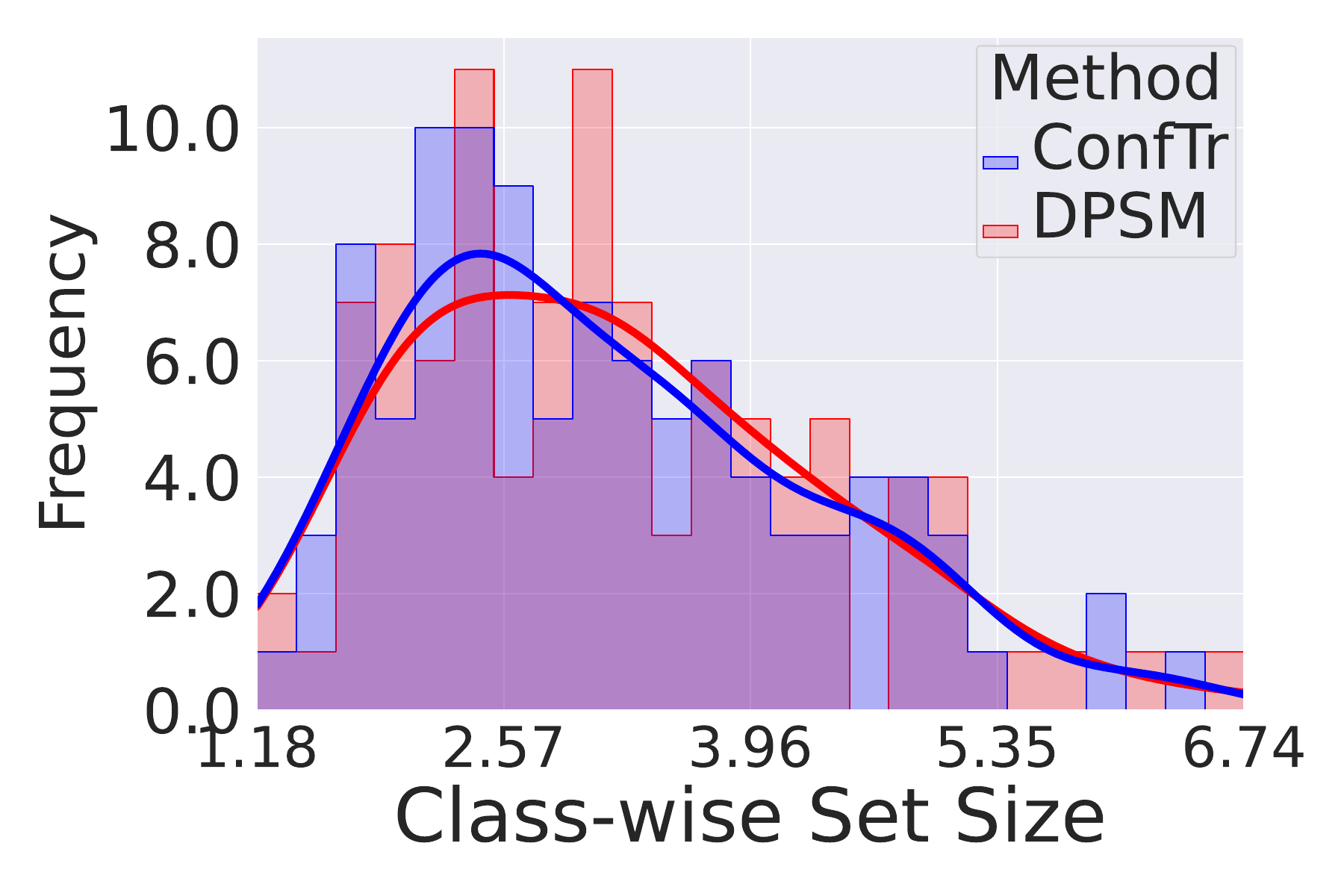}
    \\
    \includegraphics[width=\linewidth]{rebuttal_figure/Class_histogram_SizeHist_conftr.pdf}
    \end{minipage} 
    \vspace{-0.15in}
    \caption{
    \textbf{Class conditional coverage and class-wise prediction set size} of all methods on CIFAR-100 with DenseNet and HPS. 
     To better compare the class conditional coverage, we compare the class-conditional coverage between DPSM and 3 baselines in (a), (b) and (c) separately. 
     DPSM shows a bit more concentration in terms of class-wise coverage to the nominal coverage ($90\%$) and smaller prediction set size.
    }
    \label{fig:results_class}
    \vspace{-2.0ex}
\end{figure*}

\begin{figure*}[!t]
    \centering
    \begin{minipage}[t]{0.32\linewidth}
    \centering
    \textbf{(a)} Class-wise coverage and size (DPSM vs CE)
    \includegraphics[width = \linewidth]{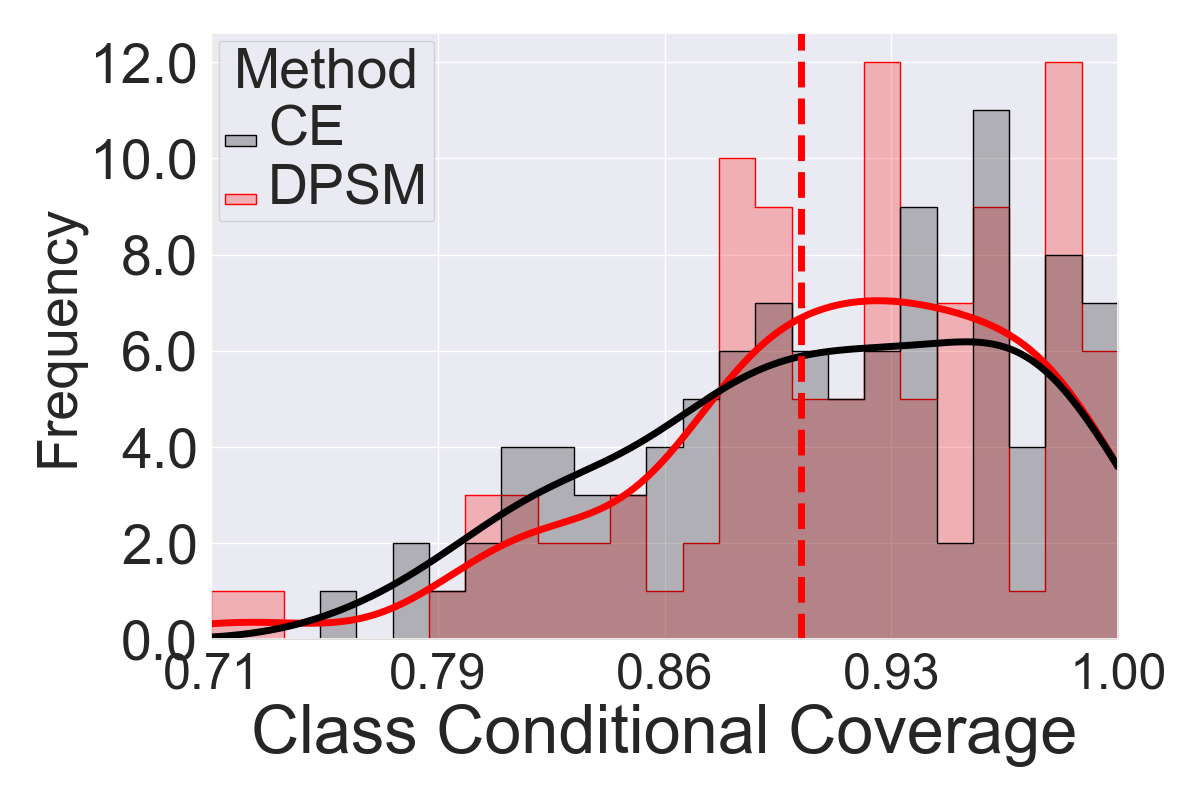}
    \\
    \includegraphics[width = \linewidth]{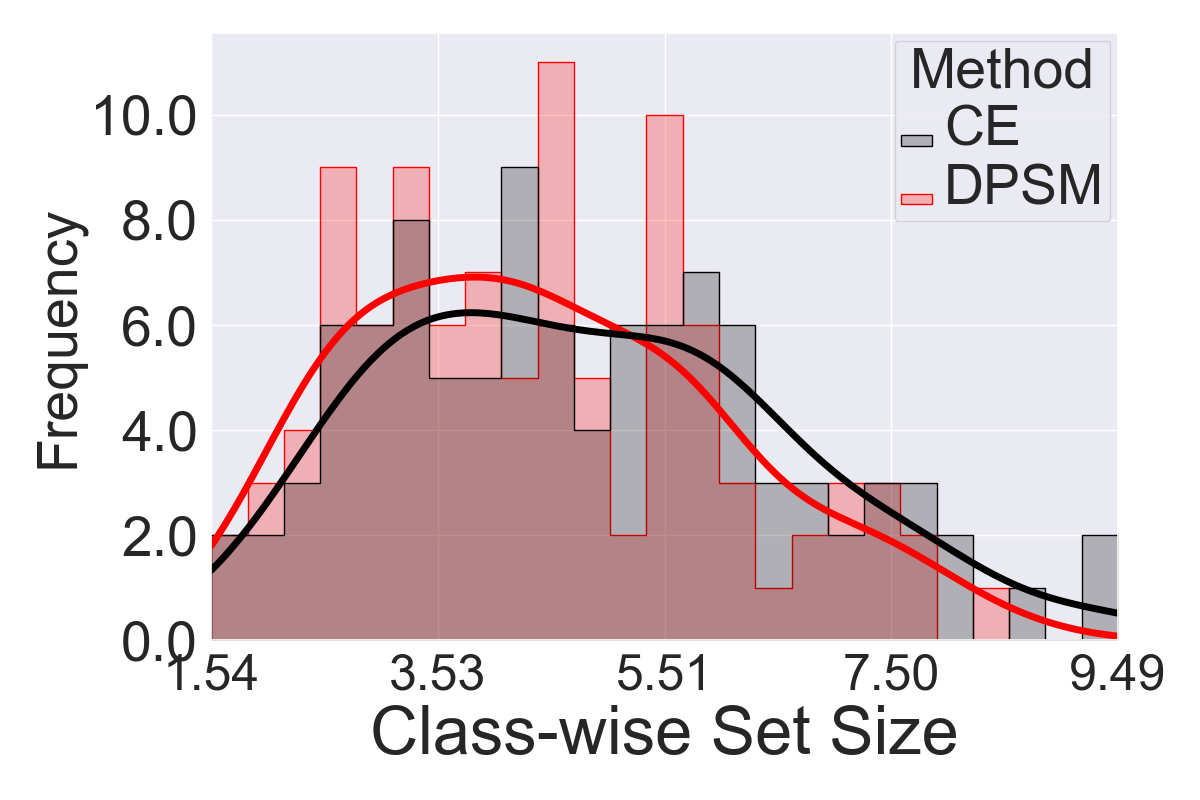}
    \end{minipage} 
    \begin{minipage}[t]{0.32\linewidth}
    \centering
    \textbf{(b)} Class-wise coverage and size (DPSM vs CUT)
    \includegraphics[width=\linewidth]{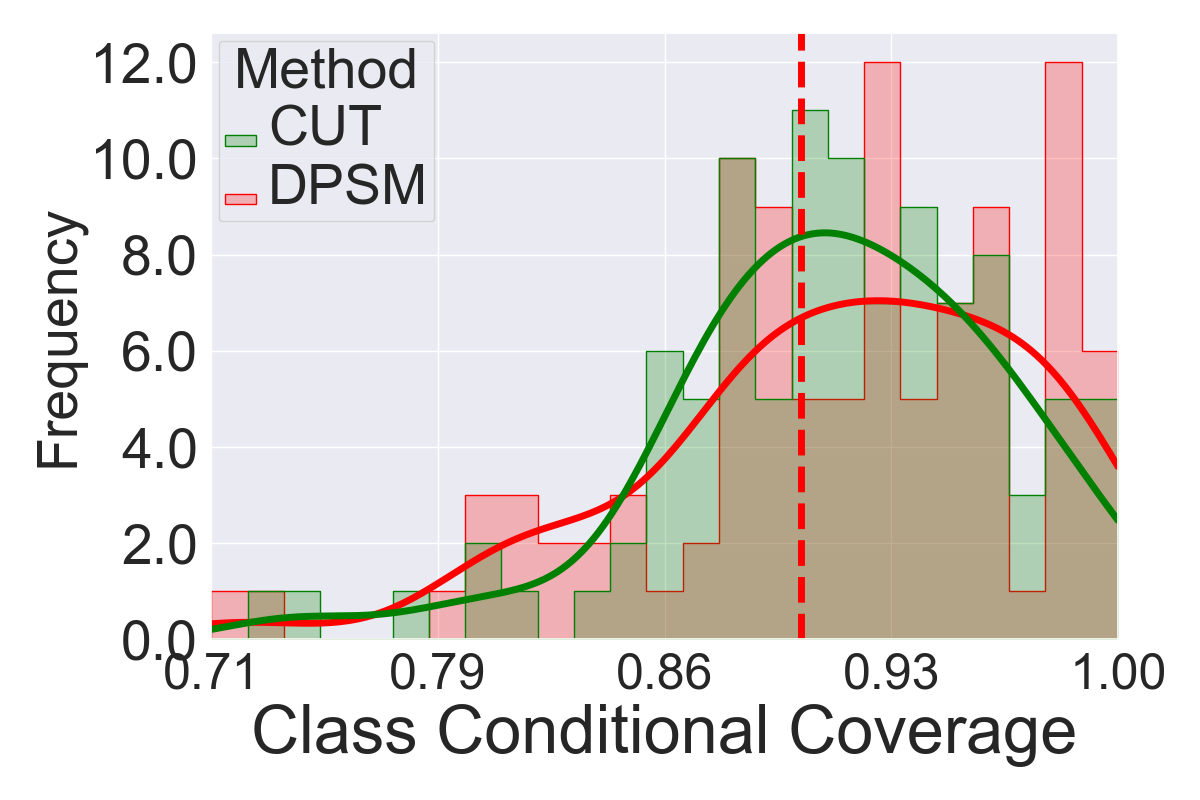}
    \\
    \includegraphics[width=\linewidth]{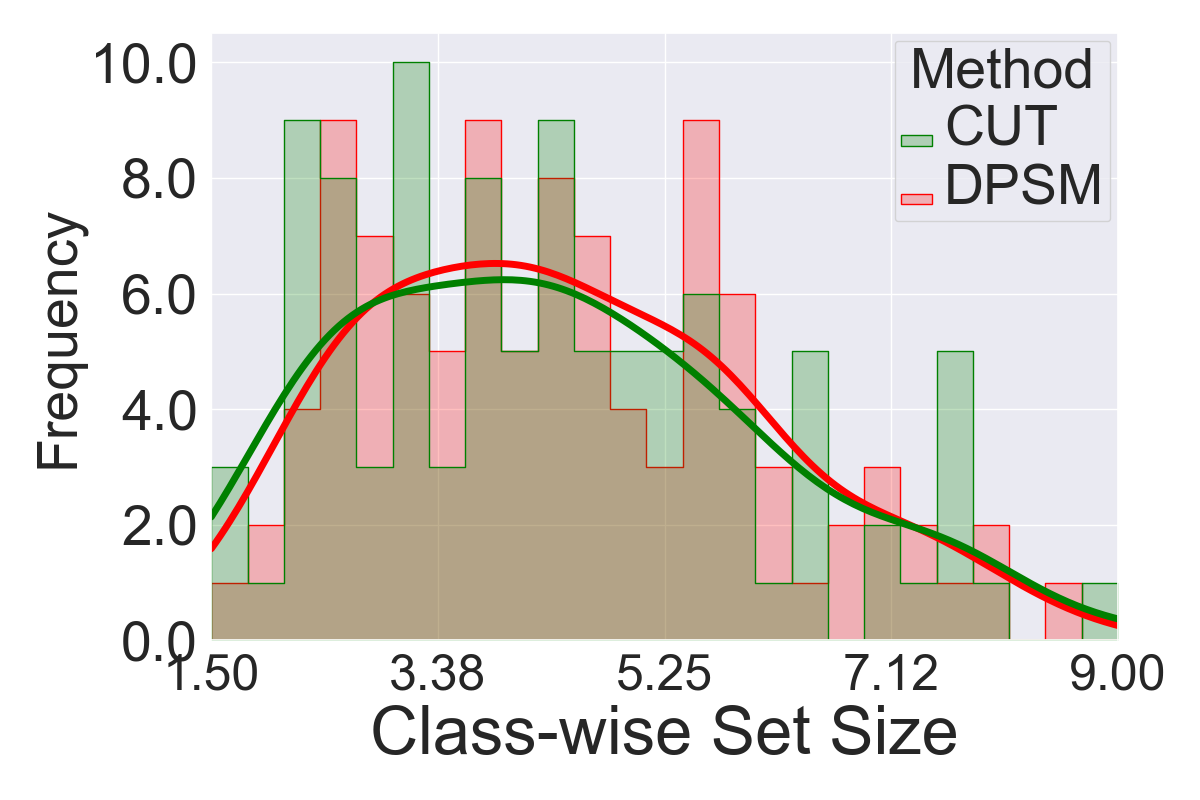}
    \end{minipage} 
    \begin{minipage}[t]{0.32\linewidth}
    \centering
    \textbf{(c)} Class-wise coverage and size (DPSM vs Conftr)
    \includegraphics[width=\linewidth]{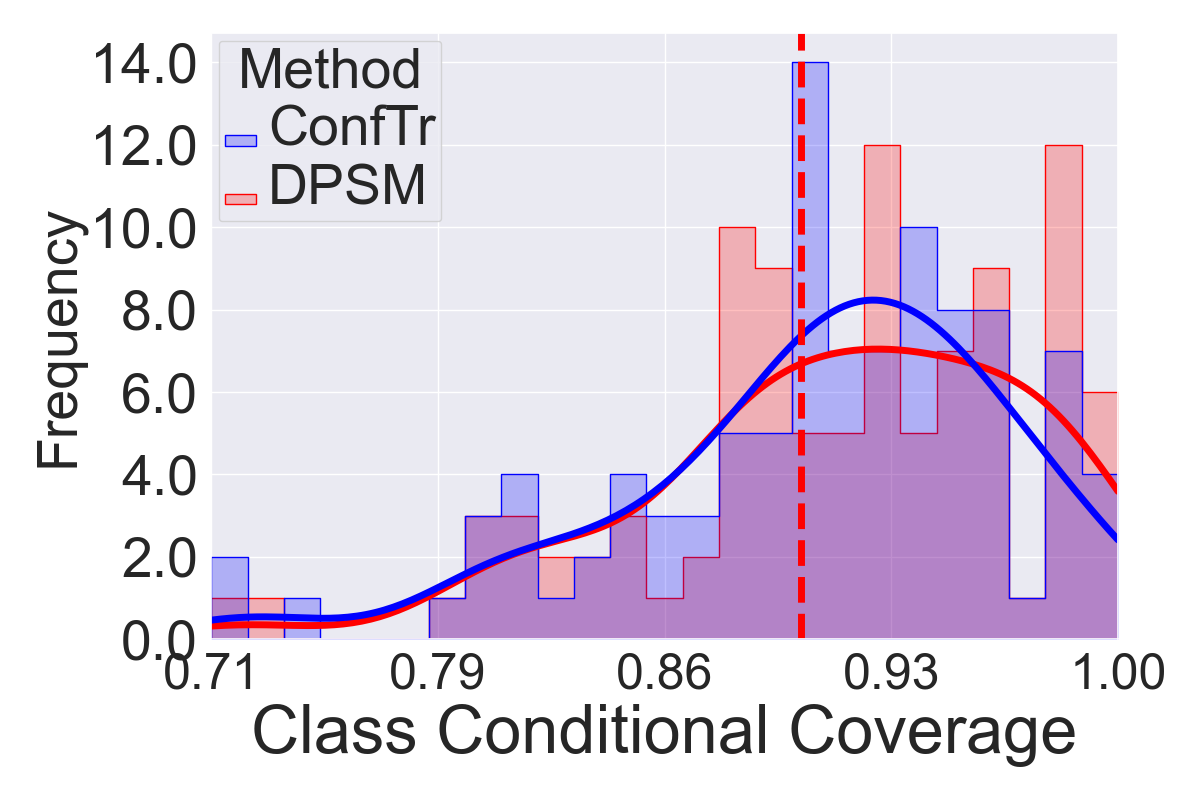}
    \\
    \includegraphics[width=\linewidth]{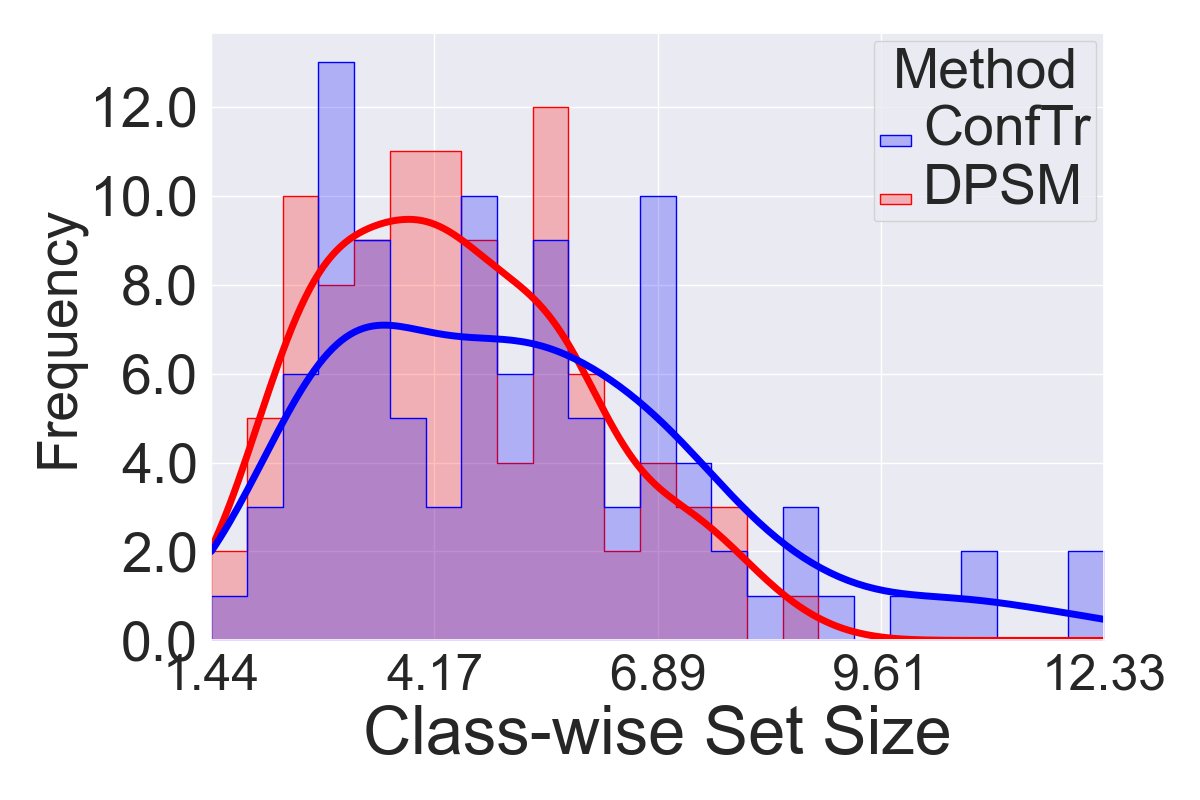}
    \end{minipage} 
    \vspace{-0.15in}
    \caption{
    \textbf{Class conditional coverage and class-wise prediction set size} of all methods on CIFAR-100 with ResNet and HPS. 
     To better compare the class conditional coverage, we compare the class-conditional coverage between DPSM and 3 baselines in (a), (b) and (c) separately. 
     DPSM shows a bit more concentration in terms of class-wise coverage to the nominal coverage ($90\%$) and smaller prediction set size.
    }
    \label{fig:results_class_resnet}
    \vspace{-2.0ex}
\end{figure*}

\textbf{DPSM achieves comparable conditional coverage performance and strong class-conditional coverage performance compared to baselines}.
To evaluate the impact of prediction set size reduction on the conditional coverage performance of DPSM, we report three metrics on CIFAR-100 using the HPS score in Table \ref{tab:cvg_set_hps_conditional}:
(i) WSC (Worst-Slab Coverage, $\uparrow$ better), introduced in \cite{romano2020classification};
(ii) SSCV (Size-Stratified Coverage, $\downarrow$ better), from \cite{angelopoulos2021uncertainty}; 
and (iii) CovGap (Average Class Coverage Gap, $\downarrow$ better), proposed in \cite{ding2024class} to measure class-conditional coverage.
The results show that DPSM achieves the best class-conditional coverage, reflected by the lowest CovGap.
For SSCV, DPSM performs slightly worse than CE and CUT, but better than ConfTr.
In terms of WSC, DPSM and CUT achieve comparable performance.
We further visualize the distribution of class-conditional coverage and class-wise average prediction set size in Figure \ref{fig:results_class} and \ref{fig:results_class_resnet}, providing fine-grained insights into class-conditional performance.
DPSM demonstrates slightly more concentrated class-wise coverage and generally smaller class-wise prediction set sizes.

\end{document}